\documentclass[a4paper,10pt]{amsart}

\usepackage{cite}
\usepackage{amsthm,amssymb}
\usepackage{mathtools,thmtools,thm-restate}
\usepackage{bm,esint}
\usepackage{color}
\usepackage{float,graphicx,subcaption}
\usepackage{cancel}
\usepackage[shortlabels]{enumitem}
\usepackage{mathrsfs}  
\usepackage{bbm}
\usepackage{euscript}
\usepackage{hyperref}
\usepackage{cleveref}
\usepackage{upgreek}
\usepackage{tikz-cd}
\usetikzlibrary{decorations.pathmorphing}

\usepackage{todonotes}

\usepackage{algorithm}
\usepackage{algpseudocode}

\declaretheoremstyle[
  headfont=\normalfont\bfseries\itshape,
  numbered=yes,
  bodyfont=\normalfont,
  spaceabove=1em plus 0.75em minus 0.25em,
  spacebelow=.5em,
  qed={\small$\Diamond$},
]{deflt}

\declaretheorem[style=deflt,numberwithin=section]{theorem}
\declaretheorem[style=deflt,sibling=theorem]{lemma}
\declaretheorem[style=deflt,sibling=theorem]{proposition}
\declaretheorem[style=deflt,sibling=theorem]{corollary}
\declaretheorem[style=deflt,sibling=theorem]{assumption}

\declaretheorem[style=deflt,sibling=theorem]{example}
\declaretheorem[style=deflt,sibling=theorem]{remark}
\declaretheorem[style=deflt,sibling=theorem]{definition}

\numberwithin{equation}{section}

\newcommand{\T}{\mathbb{T}}
\newcommand{\R}{\mathbb{R}}
\newcommand{\C}{\mathbb{C}}
\newcommand{\E}{\mathbb{E}}
\newcommand{\N}{\mathbb{N}}

\newcommand{\Z}{\mathbb{Z}}

\newcommand{\cE}{\mathcal{E}}
\newcommand{\cF}{\mathcal{F}}

\newcommand{\cJ}{\mathcal{J}}

\newcommand{\cL}{\mathcal{L}}

\newcommand{\cN}{\mathcal{N}}

\newcommand{\cP}{\mathcal{P}}
\newcommand{\cQ}{\mathcal{Q}}
\newcommand{\cR}{\mathcal{R}}
\newcommand{\cS}{\mathcal{S}}

\newcommand{\cV}{\mathcal{V}}

\newcommand{\cX}{\mathcal{X}}
\newcommand{\cY}{\mathcal{Y}}

\newcommand{\tL}{\tilde{\cL}}

\newcommand{\tQ}{\tilde{\cQ}}

\newcommand{\tS}{\tilde{\cS}}

\newcommand{\DX}{{D_{\cX}}}
\newcommand{\DY}{{D_{\cY}}}
\newcommand{\ev}{{\mathrm{ev}}}

\newcommand{\define}{\textbf}

\newcommand{\size}{\mathrm{size}}

\newcommand{\depth}{\mathrm{depth}}

\renewcommand{\Re}{\mathrm{Re}}
\renewcommand{\Im}{\mathrm{Im}}

\renewcommand{\tilde}{\widetilde}
\renewcommand{\hat}{\widehat}
\renewcommand{\bar}{\overline}

\newcommand{\set}[2]{{\left\{ #1 \,\middle|\, #2 \right\}}}
\newcommand{\slot}{{\,\cdot\,}}
\newcommand{\per}{\mathrm{per}}

\newcommand{\explain}[2]{\overset{\mathclap{\underset{\downarrow}{#2}}}{#1}}

\tikzcdset{every label/.append style = {font = \normalsize},
            cells={nodes={font=\large}}}

\title[Parametric complexity of operator learning]{The parametric complexity of operator learning}
\author{Samuel Lanthaler \and Andrew M. Stuart}
\date{ \today }

\makeatletter
\renewcommand{\paragraph}{%
  \@startsection{paragraph}{4}%
  {\z@}{.8ex \@plus 1ex \@minus .2ex}{-1em}%
  {\normalfont\normalsize\bfseries}%
}
\makeatother

\usepackage{color}
\definecolor{darkred}{rgb}{.6,0,0}
\definecolor{darkblue}{rgb}{0,0,.7}
\definecolor{darkgreen}{rgb}{0,.7,0}
\definecolor{darkbrown}{rgb}{0.8,0.4,0.4}

\newcommand{\rev}[1]{#1}
\newcommand{\revv}[1]{#1}

\newtheoremstyle{named}{}{}{\itshape}{}{\bfseries}{.}{.5em}{\thmnote{#3}}
\theoremstyle{named}

\renewcommand{\j}{{j}}
\renewcommand{\k}{{k}}
\newcommand{\jk}{{j_k}}
\newcommand{\mfQ}{\mathfrak{Q}}
\newcommand{\mfq}{\mathfrak{q}}
\newcommand{\cmplx}{\mathrm{cmplx}}
\newcommand{\oF}{\mathring{F}}

\newcommand{\ssection}[1]{\subsubsection{#1}}
\begin{document}

\begin{abstract}
Neural operator architectures employ neural networks to approximate operators mapping between Banach spaces of functions; they may be used to accelerate model evaluations via emulation, or to discover models from data. Consequently, the methodology has received increasing attention over recent years, giving rise to the rapidly growing field of operator learning. 
The first contribution of this paper is to prove that for general classes of operators which are characterized only by their $C^r$- or Lipschitz-regularity, operator learning suffers from a ``curse of parametric complexity'', which is an infinite-dimensional analogue of the well-known curse of dimensionality encountered in high-dimensional approximation problems. The result is applicable to a wide variety of existing neural operators, including 
PCA-Net, DeepONet and the FNO.
The second contribution of the paper is to prove that this general curse can be overcome for solution operators defined by the Hamilton-Jacobi equation; this is achieved by leveraging additional structure in the underlying solution operator, going beyond regularity. To this end, a novel neural operator architecture is introduced, termed HJ-Net, which explicitly takes into account characteristic information of the underlying Hamiltonian system. Error and complexity estimates are derived for HJ-Net which show that this architecture can provably beat the curse of parametric complexity related to the infinite-dimensional input and output function spaces.
\end{abstract}

\maketitle

\tableofcontents

\section{Introduction}
\label{sec:I}

This paper is devoted to a study of the computational complexity of the approximation of maps
between Banach spaces by means of neural operators. The paper has two main focii: establishing
a complexity barrier for general classes of $C^r-$ or Lipschitz regular maps; and 
then showing that this barrier can  be beaten for Hamilton-Jacobi (HJ) equations.
In Subsection \ref{ssec:cxt} we give a detailed literature review; we set in context the definition of
\emph{``the curse of parametric complexity''} that we introduce and use in this paper; and we
highlight our main contributions. Then, in Subsection \ref{ssec:ovr}, we overview the organization 
of the remainder of the paper.

\subsection{Context and Literature Review}
\label{ssec:cxt}

The use of neural networks to learn operators, typically mapping between
Banach spaces of functions defined over subsets of finite dimensional
Euclidean space and referred to as \emph{neural operators}, 
is receiving growing interest in the computational science and engineering community \cite{chen1995universal,zhu2018bayesian,boulle2022data,khoo2021solving,lu2021learning,bhattacharya2021model,nelsen2021random,li2021fourier}.
The methodology has the potential for \emph{accelerating} numerical methods for solving
partial differential equations (PDEs) when a model relating inputs and outputs is known;
and it has the potential for \emph{discovering} input-output maps from data when
no model is available. 

The computational complexity of learning and
evaluating such neural operators is crucial to understanding when the methods
will be effective. Numerical experiments addressing this issue may be
found in \cite{de2022cost} and the analysis of linear problems from
this perspective may be
found in \cite{boulle2022learning,de2023convergence}. Universal approximation
theorems, applicable beyond the linear setting, may be found in
\cite{chen1995universal,lu2021learning,lanthaler2021error,kovachki2021universal,kovachki2021neural,lanthaler2023nonlocal,bhattacharya2021model,YOU2022115296} but such theorems do not address the cost of achieving a given
small error.

Early work on operator approximation \cite{mhaskar1997neural} presents first quantitative bounds; most notably, this work identifies the continuous nonlinear $n$-widths of a space of continuous functionals defined on $L^2$-spaces, showing that these $n$-widths decay at most (poly-)logarithmically in $n$. Both upper and lower bounds are derived in this specific setting. 
More recently, upper bounds on the computational
complexity of recent approaches to operator learning based on deep neural networks, including the DeepONet \cite{lu2021learning} and the Fourier Neural Operator (FNO) \cite{li2021fourier}, have been studied in more
detail. Specific operator learning tasks arising in PDEs have been considered in the papers
\cite{schwab2019deep,herrmann2022neural,marcati2023exponential,lanthaler2021error,kovachki2021universal,ryck2022generic,deng2021convergence}. Related complexity analysis for the PCA-Net architecture \cite{bhattacharya2021model} has recently been established in \cite{lanthaler2023operator}.
These papers studying computational complexity 
focus on the issue of beating a form of the \emph{``curse of dimensionality''} 
in these operator approximation tasks.

In these operator learning problems the input and output spaces are
infinite dimensional, and hence the meaning of the curse of dimensionality could be ambiguous. 
In this infinite-dimensional context, ``beating the curse'' is interpreted as identifying problems, 
and operator approximations applied to those problems, 
for which a measure of evaluation cost (referred to as their complexity)
grows only algebraically 
with the inverse of the desired error. As shown rigorously in \cite{lanthaler2023operator}, this is a non-trivial issue: for the PCA-Net architecture, it has been established that such algebraic complexity and error bounds \emph{cannot} be achieved for general Lipschitz (and even more regular) operators. 

As will be explained in detail in the present work, this fact is not specific to PCA-Net, but extends to many other neural operator architectures. In fact, it can be interpreted as a scaling limit of the conventional curse of dimensionality; this conventional curse affects finite-dimensional approximation problems when the underlying dimension $d$ is very large. It can be shown that (ReLU) neural networks cannot overcome this curse, in general. As a consequence, neural operators, which build on neural networks, suffer from the scaling limit of this curse in infinite-dimensions. To distinguish this infinite-dimensional phenomenon from the conventional curse of dimensionality encountered in high-but-finite-dimensional approximation problems, we will refer to the scaling limit identified in this work as \emph{``the curse of parametric complexity''.}

The first contribution of 
the present paper is to prove that for general classes of operators which 
are characterized only by their $C^r$- or Lipschitz-regularity, 
operator learning suffers from such a curse of parametric complexity:
Theorem \ref{thm:codexp} (and a variant thereof,
Theorem  \ref{thm:fno-cod}) shows that, in this setting,
there exist operators (and indeed even real-valued functionals)
which are $\epsilon-$approximable only with parametric model complexity
that grows exponentially in $\epsilon^{-1}$. 

To overcome the general curse of parametric complexity implied by Theorem \ref{thm:codexp} (and Theorem  \ref{thm:fno-cod}), efficient operator learning frameworks therefore have to leverage additional structure present in the operators of interest, going beyond $C^r$- or Lipschitz-regularity. Previous work on overcoming this curse for operator learning has mostly focused on operator holomorphy \cite{herrmann2020deep,schwab2019deep,lanthaler2021error} and the emulation of numerical methods \cite{kovachki2021universal,lanthaler2021error,lanthaler2023operator,marcati2023exponential} as two basic mechanisms for overcoming the curse of parametric complexity for specific operators of interest. 
A notable exception are the complexity estimates for DeepONets in \cite{deng2021convergence}. These estimates are based on explicit representation of the solution; most prominently, this is achieved via the Cole-Hopf transformation for the viscous Burgers equation. 
\rev{
 Another distinct mechanism to overcome the curse of parametric complexity is  highlighted by the notion of Barron spaces, e.g. \cite{barron1993universal,bach2017breaking,ma2022barron}, which have recently been extended to the operator learning setting in \cite{korolev2022two}. In the finite-dimensional context, functions belonging to such Barron spaces can be approximated at Monte-Carlo rates which are independent of the dimension of the underlying domain. As shown in \cite{korolev2022two}, similar rates can be established in an infinite-dimensional Barron class setting. The curse of parametric complexity derived in the present work indicates that spaces of $r$-times Fr\'echet differentiable operators cannot embed into operator Barron spaces, no matter the degree of differentiability $r$.
 }

An abstract characterization of the entire class of operators that allow for efficient approximation by neural operators would be very desirable. Unfortunately, this appears to be out of reach, at the current state of analysis.
Indeed, as far as the authors are aware, there does not even exist such a 
characterization for any class of standard numerical methods, 
such as finite difference, finite element or spectral, viewed as operator approximators. Therefore, in order to identify settings in which operator learning can be effective (without suffering from the general curse of parametric complexity), we restrict attention to specific classes of operators of interest.

The HJ equations present an application area that has the potential to be significantly \revv{impacted by the use of ideas from neural networks, especially
regarding the solution of problems for functions defined over subsets of high
dimensional ($d$) Euclidean space \cite{chow2017algorithm,chow2019algorithm,darbon2016algorithms,darbon2020overcoming};
in particular beating the curse of dimensionality with respect to this dimension $d$
has been of focus. We highlight work which proposes adapted neural network architectures, reflecting known structure of viscosity solutions of the HJ equation, to approximate and represent \emph{individual solutions} \cite{darbon2021some,meng2022sympocnet}. This cited research includes studies of analytical convergence rates for such architectures, as well as empirical work studying their practical performance \cite{nakamura2020qrnet,hure2021deep,darbon2023neural}. However, this entire} body of work has not studied operator learning, as it concerns settings
in which only fixed instances of the PDE are solved. The purpose of the second part of the present paper is to study the design and analysis of neural operators to approximate the solution operator for
HJ equations; this operator maps the initial condition (a function) to the
solution at a later time (another function). 

The second contribution of the paper is to prove in Theorem \ref{thm:HJNET}
that the general curse of parametric complexity can be overcome for maps defined 
by the solution operator of the Hamilton-Jacobi (HJ) equation; this is 
achieved by exposing additional structure in the underlying solution 
operator, different from holomorphy and emulation and going beyond regularity, that can be leveraged by neural operators; for the HJ equations, the identified structure relies on representation of solutions of the HJ equations in terms of characteristics. In this paper the dimension $d$ of the underlying spatial domain will be fixed
and we do not study the curse of dimensionality with respect to $d$.
Instead, we demonstrate that it is possible
to beat the curse of parametric complexity 
with respect to the infinite dimensional nature
of the input function for fixed (and moderate) $d$.

\subsection{Organization}
\label{ssec:ovr}
In section \ref{sec:cod} we present the first contribution:
Theorem \ref{thm:codexp}, together with the closely related
Theorem  \ref{thm:fno-cod} which extends the general but not exhaustive  setting 
Theorem \ref{thm:codexp} to include the FNO, establish 
that the curse of parametric complexity 
is to be expected in operator learning. 
The following sections then focus on the second contribution and hence
on solution operators associated with the HJ equation; 
in Theorem \ref{thm:HJNET} we prove that additional structure in the
solution operator for this equation can be leveraged to beat the curse 
of parametric complexity.
In Section \ref{sec:S} we describe the 
precise setting for the HJ equation employed
throughout this paper; we recall the method of characteristics for solution of
the equations; and we describe a short-time existence theorem. 
Section \ref{sec:NO} introduces the proposed neural operator, the HJ-Net,
based on learning the flow underlying the method of characteristics and combining it with
scattered data approximation. In Section \ref{sec:E} we state our approximation
theorem for the proposed HJ-Net, resulting in complexity estimates which demonstrate
that the curse of parametric complexity is avoided in relation to the infinite dimensional
nature of the input space (of initial conditions). Section \ref{sec:C} contains
concluding remarks. Whilst the high-level structure of the proofs is contained
in the main body of the paper, many technical details are collected in the appendix,
to promote readability.

\section{The Curse of Parametric Complexity}
\label{sec:cod}

Operator learning seeks to employ neural networks to efficiently approximate operators mapping between infinite-dimensional Banach spaces of functions. 
To enable implementation of these methods in practice, maps between the formally infinite-dimensional spaces have to be approximated using only a finite number of degrees of freedom.
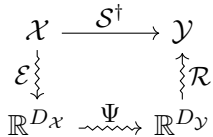
\begin{figure}
\begin{tikzcd}
\cX \arrow{r}{\cS^\dagger} \arrow[swap,squiggly]{d}{\cE} 
& \cY  \\%
\R^{\DX} \arrow[squiggly]{r}{\Psi}& \R^{\DY} \arrow[squiggly,swap]{u}{\cR}
\end{tikzcd}
\caption{Diagrammatic illustration of operator learning based on an encoding $\cE$, a neural network $\Psi$, and a reconstruction $\cR$.}
\label{fig:ear}
\end{figure}

Commonly, operator learning frameworks can therefore be written in terms of an encoding, a neural network and a reconstruction step 
as shown in Figure \ref{fig:ear}. The first step $\mathcal{E}$ encodes the infinite-dimensional input using a finite number of degrees of freedom. The second approximation step $\Psi$ maps the encoded input to an encoded, finite-dimensional output. The final reconstruction step $\mathcal{R}$ reconstructs an output function given the finite-dimensional output of the approximation mapping. The composition of these encoding, approximation and reconstruction mappings thus takes an \emph{input function} and maps it to another \emph{output function}, and hence defines an operator. Existing operator learning frameworks differ in their particular choice of the encoder, neural network architecture and reconstruction mappings.

We start by giving background discussion on the curse of dimensionality (CoD) in finite dimensions,
in subsection \ref{ssec:cursenn}. We then describe the subject in detail for  neural network-based operator learning, resulting in our notion of the curse of parametric complexity, 
in subsection \ref{ssec:curseop}. In subsection \ref{ssec:main0} we state our main theorem concerning
the curse of parametric complexity for neural operators. Subsection \ref{ssec:examples} demonstrates
that the main theorem applies to PCA-Net, DeepONet and the NOMAD neural network architectures.
Subsection \ref{ssec:curseopFNO} extends the main theorem to the FNO since
it sits outside the framework introduced in subsection \ref{ssec:curseop}.

\subsection{Curse of Dimensionality for Neural Networks}
\label{ssec:cursenn} Since the neural network mapping $\Psi: \R^{\DX} \to \R^{\DY}$ in the decomposition shown in Figure \ref{fig:ear} typically maps between \emph{high-dimensional} (encoded) spaces, with $\DX$, $\DY \gg 1$, most approaches to operator learning employ neural networks to learn this mapping. The motivation for this is that, empirically, neural networks have been found to be exceptionally well suited for the approximation of such high-dimensional functions in diverse applications \cite{goodfellow2016deep}. Detailed investigation of the approximation theory of neural networks, including quantitative upper and lower approximation error bounds, has thus attracted a lot of attention in recent years \cite{yarotsky_error_2017,yarotsky2020phase,kohler2021rate,lu2021deep,devore2021neural}. Since we build on this analysis we summarize the
relevant part of it here, restricting attention to ReLU neural networks 
in this work, as defined next; generalization to the use of other 
(piecewise polynomial) activation functions is possible.

\subsubsection{ReLU Neural Networks} 
\label{sec:nndef}
Fix integer $L$ and integers $\{d_\ell\}_{\ell=0}^{L+1}.$ 
Let $A_\ell \in \R^{d_{\ell+1} \times d_{\ell}}$ and
$b_\ell \in \R^{d_{\ell+1}}$ for $\ell=0,\dots, L$. A ReLU neural network $\Psi: \R^{\DX} \to \R^{\DY}$, $x \mapsto \Psi(x)$ is a mapping of the form
\begin{align}
\label{eq:nn}
\left\{
\begin{aligned}
   x_0&=x,\\
   x_{\ell+1}&=\sigma(A_\ell x_\ell + b_\ell), \quad \ell=0, \cdots, L-1,\\
   \Psi(x) & =A_L x_L+b_L,
\end{aligned}
\right.
\end{align}
where $d_0 = \DX$ and $d_{L+1} = \DY$.
Here the activation function $\sigma:\R \to \R$ is extended pointwise to
act on any Euclidean space; and in what follows we employ the ReLU activation function
$\sigma(x)=\min\{0,x\}.$ We let $\theta:=\{A_\ell,b_\ell\}_{\ell=0}^{L}$ and note that we
have defined parametric mapping $\Psi(\slot)=\Psi(\slot;\theta)$. We define the depth of $\Psi$ as the number of layers, and the size of $\Psi$ as the number of non-zero weights and biases, i.e.
\[
\depth(\Psi) = L,
\quad
\size(\Psi) = \sum_{\ell=0}^L \left\{ \Vert A_\ell \Vert_0 + \Vert b_\ell \Vert_0 \right\},
\]
where $\Vert \slot \Vert_0$ counts the number of non-zero entries of a matrix or vector.

\subsubsection{Two Simple Facts from ReLU Neural Network Calculus}
\label{sec:nncalc}
The following facts will be used without further comment (see e.g. \cite[Section 2.2.3]{opschoor2022exponential} for a discussion of more general results): If $\Psi: \R^\DX \to \R^\DY$ is a ReLU neural network, and $A\in \R^{\DX \times d}$ is a matrix, then there exists a ReLU neural network $\tilde{\Psi}: \R^d \to \R^\DY$, such that 
\begin{align}
\label{eq:calculus1}
\left\{
\begin{aligned}
\tilde{\Psi}(x) &= \Psi(Ax), \qquad \text{for all $x\in \R^d$,}
\\
\depth(\tilde{\Psi}) &= \depth(\Psi)+1, \\
\size(\tilde{\Psi}) &\le 2 \Vert A \Vert_0 + 2\,\size(\Psi).
\end{aligned}
\right.
\end{align}
Similarly, if $V\in \R^{d\times \DY}$ is a matrix, then there exists a ReLU neural network $\hat{\Psi}: \R^\DX \to \R^d$, such that 
\begin{align}
\label{eq:calculus2}
\left\{
\begin{aligned}
\hat{\Psi}(x) &= V\Psi(x), \qquad \text{for all $x\in \R^d$,}
\\
\depth(\hat{\Psi}) &= \depth(\Psi)+1, \\
\size(\hat{\Psi}) &\le 2 \Vert V \Vert_0 + 2\,\size(\Psi).
\end{aligned}
\right.
\end{align}
The main non-trivial issue in \eqref{eq:calculus1} and \eqref{eq:calculus2} is to preserve the potentially sparse structure of the underlying neural networks; this is based on a concept of ``sparse concatenation'' from \cite{petersen2018optimal}.

\subsubsection{Approximation Theory and CoD for ReLU Networks}

One overall finding of research into the approximation power of ReLU neural networks is that, for function approximation in spaces characterized by smoothness, neural networks cannot entirely overcome the curse of dimensionality \cite{yarotsky_error_2017,yarotsky2020phase,kohler2021rate,lu2021deep}. 
This is illustrated by the following result, which builds on \cite[Thm. 5]{yarotsky_error_2017} derived by D. Yarotsky:

\begin{restatable}[Neural Network CoD]{proposition}{PropositionNNCoD}
\label{prop:nn-cod}
Let $r\in \N$ be given. For any dimension $D\in \N$, there exists $f_{D,r} \in C^r([0,1]^D;\R)$ and constant $\bar{\epsilon},\gamma > 0$, such that any ReLU neural network $\Psi: \R^D \to \R$ achieving accuracy 
\[
\sup_{x\in [0,1]^D} |f_{D,r}(x) - \Psi(x)| \le \epsilon,
\]
with $\epsilon \le \bar{\epsilon}$, has size at least $\size(\Psi) \ge \epsilon^{-\gamma D/r}$. The constant $\bar{\epsilon} = \bar{\epsilon}(r) > 0$ depends only on $r$, and $\gamma > 0$ is universal.
\end{restatable}

The proof of Proposition \ref{prop:nn-cod} is included in Appendix \ref{A:nn-cod}.
Proposition \ref{prop:nn-cod} shows that neural network approximation 
of a function between high-dimensional Euclidean spaces suffers from a curse of dimensionality, characterised by an algebraic complexity with a potentially large exponent proportional to the dimension $D \gg 1$. 
This lower bound is similar to the approximation rates (upper bounds) 
achieved by traditional methods,\footnote{Ignoring the potentially beneficial factor $\gamma$} 
such as polynomial approximation. This fact suggests that the empirically 
observed efficiency of neural networks may well rely on additional structure 
of functions $f$ of practical interest, beyond their smoothness; for 
relevant results in this direction see, for example,
\cite{mhaskar2016learning,gribonval2022approximation}.

\vspace{2em}

\subsection{Curse of Parametric Complexity in Operator Learning}
\label{ssec:curseop}
 In the present work, we consider the approximation of an underlying operator $\cS^\dagger: \cX \to \cY$ acting between Banach spaces; specifically, we assume that the dimensions of the spaces $\cX,\cY$ are \emph{infinite}.
 Given the curse of dimensionality in the finite-dimensional setting, Proposition \ref{prop:nn-cod}, and letting $D \to  \infty$, one would generally expect a super-algebraic, potentially even exponential, lower bound on the ``complexity'' of neural operators $\cS$
 approximating such $\cS^\dagger$, as a function of the accuracy $\epsilon$. 
 In this subsection, we make this statement precise for a general class of neural operators, in Theorem \ref{thm:codexp}. This is preceded by a discussion of relevant structure of compact sets in infinite-dimensional function spaces and a discussion of a suitable class of ``neural operators''.


\ssection{Infinite-dimensional hypercubes}

 Proposition \ref{prop:nn-cod} was stated for the unit cube $[0,1]^D$ as the underlying domain. In the finite-dimensional setting of Proposition \ref{prop:nn-cod}, the approximation rate turns out to be independent of the underlying compact domain, provided that the domain has non-empty interior and assuming a Lipschitz continuous boundary. This is in contrast to the infinite dimensional case, where compact domains necessarily have empty interior and where the convergence rate depends on the specific structure of the domain. To state our complexity bound for operator approximation, we will therefore need to discuss the prototypical structure of compact subsets $K\subset \cX$. 
 
 To fix intuition, we temporarily consider $\cX$ a function space (for example a H\"older, Lebesgue or Sobolev space). In this case, the most common way to define a compact subset $K\subset \cX$ is via a smoothness constraint, as illustrated by the following concrete example:
\begin{example}
\label{ex:Kexample}
Assume that $\cX = C^s(\Omega)$ is the space of $s$-times continuously differentiable functions on a bounded domain $\Omega \subset \R^d$. Then for $\rho > s$ and upper bound $M>0$, the subset $K \subset \cX$ defined by 
\begin{align}
\label{eq:Kexample}
K = \set{u\in C^\rho(\Omega)}{\Vert u \Vert_{C^\rho} \le M},
\end{align}
is a compact subset of $\cX$. Here,  we define the $C^\rho$-norm as,
\begin{align}
\label{eq:Crho}
\Vert u \Vert_{C^\rho} = \max_{|\nu|\le \rho} \sup_{x\in \Omega} |D^\nu u(x)|.
\end{align}
\end{example}
To better understand the approximation theory of operators $\cS^\dagger: K \subset \cX \to \cY$, we would like to understand the structure of such $K$. Our point of view is inspired by Fourier analysis, according to which smoothness of $u\in K$ roughly corresponds to a decay rate of the Fourier coefficients of $u$. In particular, $u$ is guaranteed to belong to the set \eqref{eq:Kexample}, if $u$ is of the form,
\begin{align}
\label{eq:cube}
u = A \sum_{j=1}^\infty j^{-\alpha} y_j e_j, \qquad y_j\in [0,1] \text{ for all } j\in \N,
\end{align}
for a sufficiently large decay rate $\alpha > 0$, small constant $A>0$, and where $e_j: \R^d\to \R$ 
 denotes the periodic Fourier (sine/cosine) basis, restricted to $\Omega$. We include a proof of this fact at the end of this subsection (see Lemma \ref{lem:Cs}), where we also identify relevant decay rate $\alpha$. 
 In this sense, the set $K$ in \eqref{eq:Kexample} could be said to ``contain'' an infinite-dimensional hypercube $\prod_{j=1}^\infty [0,Aj^{-\alpha}]$, with decay rate $\alpha$. Such hypercubes will replace the finite-dimensional unit cube $[0,1]^D$ in our analysis of operator approximation in infinite-dimensions.
 
 We would like to point out that a similar observation holds for many other sets $K$ defined by a smoothness constraint, such as sets in Sobolev spaces defined by a smoothness constraint, $\{\Vert u \Vert_{H^\rho} \le M\} \subset H^s(\Omega)$, or more generally $\{\Vert u \Vert_{W^{\rho,p}} \le M \} \subset W^{s,p}(\Omega)$, for any $1\le p \le \infty$, but also Besov spaces, spaces of functions of bounded variation and others share a similar structure. Bounded balls in all of these spaces contain infinite-dimensional hypercubes, consisting of elements of the form \eqref{eq:cube}. We note in passing that, in general, it may be more natural to replace the trigonometric basis in \eqref{eq:cube} by some other choice of basis, such as polynomials, splines, wavelets, or a more general (frame) expansion. We refer to e.g. \cite{christensen2003introduction,heil2010basis} for general background and \cite{herrmann2022neural} for an example of such a setting in the context of holomorphic operators. 
 
 The above considerations lead us to the following definition of an abstract hypercube:

 \begin{definition}
 \label{def:hypercube}
Let $e_1,e_2,\dots \in \cX$ be a sequence of linearly independent and normed elements, $\Vert e_j \Vert_{\cX} = 1$. Given constants $A > 0$, $\alpha > 1$, we say that $K\subset \cX$ \define{contains an infinite-dimensional cube $Q_\alpha= Q_\alpha(A; e_1,e_2,\dots)$}, if: 
\begin{enumerate}
\item $K$ contains the set $Q_\alpha$ consisting of all $u$ of the form \eqref{eq:cube};
\item set $\{e_j\}_{j\in \N}$ possesses a bounded bi-orthogonal sequence of functionals, labelled $e_1^\ast,e^\ast_2,\dots\in \cX^\ast$, in the continuous dual of $\cX$\footnote{If $\cX$ is a Hilbert space, such a bi-orthogonal sequence always exists for independent $e_1,e_2,\dots \in \cX$.}; i.e. we assume that $e^\ast_k(e_j) = \delta_{jk}$ for all $j,k\in \N$, and that there exists $M>0$, such that $\Vert e^\ast_k \Vert_{\cX^\ast} \le M$ for all $k\in \N$.
\end{enumerate}
 \end{definition}

\begin{remark}
 Property (2) in Definition \ref{def:hypercube}, i.e. the assumed existence of a bi-orthogonal sequence $e^\ast_j$, ensures that there exist ``coordinate projections'': if $u$ is of the form \eqref{eq:cube}, then the $j$-th coefficient $y_j$ can be retrieved from $u$ as $y_j = A^{-1}j^\alpha e^\ast_j(u)$. This allows us to uniquely identify $u \in Q_\alpha$ with a set of coefficients $(y_1,y_2,\dots) \in [0,1]^\infty$.
\end{remark}

\begin{remark}
The decay rate $\alpha$ of the infinite-dimensional cube $Q_\alpha \subset K$ provides a measure of its ``asymptotic size'' or ``complexity''. In terms of our complexity bounds, this decay rate will play a special role. Hence, we will usually retain this dependence explicitly by writing $Q_\alpha$, but suppress the additional dependence on $A$ and $e_1,e_2,\dots$ in the following.
\end{remark}

The notion of infinite-dimensional cubes introduced in Definition \ref{def:hypercube} is only a minor generalization of an established notion of cube embeddings, introduced by Donoho \cite{donoho2001sparse} in a Hilbert space setting. We refer to \cite[Chap. 5]{dahlke2015harmonic} for a pedagogical exposition of such cube embeddings in the Hilbert space setting, and their relation to the Kolmogorov entropy of $K$.

 \begin{remark}
The complexity bounds established in this work will be based on infinite-dimensional hypercubes.
     An interesting question, left for future work, is whether our main result on the curse of parametric complexity, Theorem \ref{thm:codexp} below, could be stated directly in terms of the Kolmogorov complexity of $K$, or other notions such as the Kolmogorov $n$-width \cite{devore1998nonlinear}.
 \end{remark}

Our definition of an infinite-dimensional hypercube is natural in view of the following lemma,
the discussion following Example \ref{ex:Kexample}, and other similar results.
\begin{lemma}
\label{lem:Cs}
Assume that $\cX = C^s(\Omega)$ is the space of $s-$times
continuously differentiable functions on a bounded domain $\Omega \subset \R^d$.
Choose $\rho>s$ and define $K$, compact in $\cX$, by
\[
K = \set{u\in C^\rho(\Omega)}{\Vert u \Vert_{C^\rho} \le M},
\]
with constant $M>0$.
Then $K$ contains an infinite-dimensional hypercube $Q_\alpha$, for any $\alpha > 1 + 
\frac{\rho-s}{d}$.
\end{lemma}
The proof of Lemma \ref{lem:Cs} is included in Appendix \ref{A:cod}. While we have focused on spaces of continuously differentiable functions, similar considerations also apply to other smoothness spaces, such as the Sobolev spaces $W^{s,p}$, and more general Besov spaces.

\ssection{Curse of Parametric Complexity}

The main question to be addressed in the present section is the following: given $K \subset \cX$ compact, $\cS^\dagger: \cX \to \cY$ an $r$-times Fr\'echet differentiable operator to be approximated by a neural operator $\cS: \cX \to \cY$, and given a desired approximation accuracy $\epsilon>0$, how many tunable parameters (in the architecture of $\cS$) are required to achieve,
\[
\sup_{u\in K} \Vert \cS^\dagger(u) - \cS(u) \Vert_{\cY} \le \epsilon?
\]
The answer to this question clearly depends on our assumptions on $K\subset \cX$, $\cY$ and the class of neural operators $\cS$.

\paragraph{Assume $K \subset \cX$ contains a hypercube $Q_\alpha$.} 
Consistent with our discussion in the last subsection, we will assume that $K\subset \cX$ contains an infinite-dimensional hypercube $Q_\alpha$, as introduced in Definition \ref{def:hypercube}, with algebraic decay rate $\alpha>0$.

\paragraph{Assume $\cY = \R$, i.e. $\cS^\dagger$ is a functional.}
The approximation of an operator $\cS^\dagger: \cX \to \cY$ with potentially infinite-dimensional output space $\cY$ is generally harder to achieve than the approximation of a functional $\cS^\dagger: \cX\to \R$ with one-dimensional output; indeed, if $\dim(\cY)\ge 1$, then $\R$ can be embedded in $\cY$ and any functional $\cX \to \R$ gives rise to an operator $\cX \to \cY$ under this embedding. To simplify our discussion, we will therefore initially restrict attention to the approximation of functionals, with the aim of showing that even the approximation of $r$-times Fr\'echet differentiable functionals is generally intractable. 

\paragraph{Assume $\cS$ is of neural network-type.} Assuming that $\cY = \R$, we must finally introduce a rigorous notion of the relevant class of approximating functionals $\cS: \cX \to \R$, i.e. define a class of ``neural operators/functionals'' approximating the functional $\cS^\dagger$. 

\begin{definition}[Functional of neural network-type]
\label{def:nntype}
We will say that a (neural) functional $\cS: \cX \to \R$ is a ``functional of neural network-type'', if it can be written in the form 
\begin{align}
\label{eq:nof}
\cS(u) = \Phi( \cL u ), \quad \forall \, u \in K,
\end{align}
where for some $\ell \in \N$, $\cL: \cX \to \R^\ell$ is a linear map, and $\Phi: \R^\ell \to \R$ is a ReLU neural network (potentially sparse). 
\end{definition}

If $\cS$ is a functional of neural network-type, we define the complexity of $\cS$, denoted $\cmplx(\cS)$, as the smallest size of a neural network $\Phi$ for which there exists linear map $\cL$ such that a representation of the form \eqref{eq:nof} holds, i.e.
\begin{align}
\label{eq:cmplx}
\cmplx(\cS) := \min_{\Phi} \size(\Phi),
\end{align}
where the minimum is taken over all possible $\Phi$ in \eqref{eq:nof}.

\begin{remark}
\label{rem:ellsize}
Without loss of generality, we may assume that $\ell \le \size(\Phi)$ in \eqref{eq:nof} and \eqref{eq:cmplx}. Indeed, if this is not the case, then $\size(\Phi) < \ell$ and we can show that it is possible to construct another representation pair $(\tilde{\Phi},\tilde{\cL})$ in \eqref{eq:nof}, consisting of a neural network $\tilde{\Phi}:\R^{\tilde{\ell}} \to \R$, linear map $\tilde{\cL}: \cX \to \R^{\tilde{\ell}}$ and such that $\tilde{\ell} \le \size(\tilde{\Phi})$: To see why, let us assume that $\tilde{\ell} := \size(\Phi) < \ell$. Let $A$ be the weight matrix in the first input layer of $\Phi$. Since 
\[
\Vert A \Vert_0 \le \tilde{\ell} < \ell,
\]
at most $\tilde{\ell}$ columns of $A$ can be non-vanishing. Write the matrix $A = [a_1, a_2, \dots, a_\ell]$ in terms of its column vectors. Up to permutation, we may assume that $a_{\tilde{\ell}+1} = \dots = a_{\ell} =0$. We now drop the corresponding columns in the input layer of $\Phi$ and remove these unused components from the output of the linear map $\cL$ in \eqref{eq:nof}. This leads to a new map $\tL: \cX \to \R^{\tilde{\ell}}$, with output components $(\tL u)_j = (\cL u)_j$ for $j=1,\dots, \tilde{\ell}$, and we define $\tilde{\Phi}: \R^{\tilde{\ell}} \to \R$, as the neural network that is obtained from $\Phi$ by replacing the input matrix $A = [a_1, \dots, a_\ell]$ by $\tilde{A} = [a_1, \dots, a_{\tilde{\ell}}]$. Then $\tilde{\Phi}\circ \tL = \Phi \circ \cL$, so that $\tL$ and $\tilde{\Phi}$ satisfy a representation of the form \eqref{eq:nof}, but the dimension $\tilde{\ell}$ satisfies $\tilde{\ell} = \size(\Phi) =\size(\tilde{\Phi})$; the first equality is by definition of $\tilde{\ell}$, and the last equality holds because we only removed zero weights from $\Phi$. In particular, this ensures that $\tilde{\ell}\le \size(\tilde{\Phi})$ for this new representation, without affecting the size of the underlying neural network, i.e. $\size(\Phi) = \size(\tilde{\Phi})$.
\end{remark}

More generally, we can consider $\cY = \cY(\Omega;\R^p)$ a function space, consisting of functions $v: \Omega \to \R^p$ with domain $\Omega \subset \R^d$. Given $y\in \Omega$, we introduce the point-evaluation map, 
\begin{align*}
\ev_y: \cY \to \R^p, \quad \ev_y(v) := v(y).
\end{align*}
Provided that point-evaluation $\ev_y$ is well-defined for all $v\in \cY$, we can readily extend the above notion to operators, as follows:
\begin{definition}[Operator of neural network-type]
\label{def:onntype}
Let $\cY = \cY(\Omega;\R^p)$ be a function space on which point-evaluation is well-defined. We will say that a (neural) operator $\cS: \cX \to \cY$ is an ``operator of neural network-type'', if for any evaluation point $y\in \Omega$, the composition $\ev_y \circ \cS: \cX \to \R^p$, $\ev_y(\cS(u)) := \cS(u)(y)$, can be written in the form 
\begin{align}
\label{eq:onof}
\ev_y \circ \cS(u) = \Phi_y( \cL u ), \quad \forall \, u \in K,
\end{align}
where $\cL: \cX \to \R^\ell$ is a linear operator, and $\Phi_y: \R^\ell \to \R^p$ is a ReLU neural network which may depend on the evaluation point $y\in \Omega$. In this case, we define
\[
\cmplx(\cS) := \sup_{y\in \Omega} \min_{\Phi_y} \size(\Phi_y).
\]
\end{definition}

We next state our main result demonstrating a ``curse of parametric complexity'' for functionals (and operators) of neural network-type. This is followed by a detailed discussion of the implications of this abstract result for four representative examples of operator learning frameworks: PCA-Net, DeepONet, NOMAD and the Fourier neural operator.

\subsection{Main Theorem on Curse of Parametric Complexity}
\label{ssec:main0}
The following result formalizes an analogue of the curse of dimensionality in infinite-dimensions:
\begin{theorem}[Curse of Parametric Complexity]
\label{thm:codexp}
Let $K \subset \cX$ be a compact set in an infinite-dimensional Banach space $\cX$. Assume that $K$ contains an infinite-dimensional hypercube $Q_\alpha$ for some $\alpha > 1$. Then for any $r\in \N$ and $\delta > 0$, there exists $\bar{\epsilon}>0$ and an $r$-times Fr\'echet differentiable functional $\cS^\dagger: K \subset \cX \to \R$, 
such that approximation to accuracy $\epsilon \le \bar{\epsilon}$ by a
functional $\cS_\epsilon$ of neural network-type, 
\begin{align}
\label{eq:accuracy}
\sup_{u \in K} | \cS^\dagger(u) - \cS_\epsilon(u) | \le \epsilon,
\end{align}
implies complexity bound 
$\cmplx(\cS_\epsilon) \ge \exp(c \epsilon^{-1/(\alpha+1+\delta) r})$; here $c$, $\bar{\epsilon}>0$ are constants depending only on $\alpha$, $\delta$ and $r$.
\end{theorem}

Before providing a sketch of the proof of Theorem \ref{thm:codexp}, we note the following simple corollary, whose proof is given in Appendix \ref{Acor:codexp}.

\begin{corollary}
\label{cor:codexp}
Let $K \subset \cX$ be a compact set in an infinite-dimensional Banach space $\cX$. Assume that $K$ contains an infinite-dimensional hypercube $Q_\alpha$ for some $\alpha > 1$. Let $\cY = \cY(\Omega;\R^p)$ be a function space with continuous embedding in $C(\Omega;\R^p)$. Then for any $r\in \N$ and $\delta > 0$, there exists $\bar{\epsilon}>0$ and an $r$-times Fr\'echet differentiable functional $\cS^\dagger: K \subset \cX \to \cY$, 
such that approximation to accuracy $\epsilon \le \bar{\epsilon}$ by an operator $\cS_\epsilon: \cX \to \cY$ of neural network-type, 
\begin{align}
\label{eq:accc}
\sup_{u \in K} \Vert \cS^\dagger(u) - \cS_\epsilon(u) \Vert \le \epsilon,
\end{align}
implies complexity bound 
$\cmplx(\cS_\epsilon) \ge \exp(c \epsilon^{-1/(\alpha+1+\delta) r})$; here $c$, $\bar{\epsilon}>0$ are constants depending only on $\alpha$, $\delta$ and $r$.
\end{corollary}

\begin{proof}[Proof of Theorem \ref{thm:codexp} (Sketch)]
Let $\cS_\epsilon: \cX\to \R$ be any functional of neural network-type, achieving approximation accuracy \eqref{eq:accuracy}. In view of our definition of $\cmplx(\cS_\epsilon)$ in \eqref{eq:cmplx}, to prove the claim, it suffices to show that if $\cL: \cX\to \R^\ell$ is a linear map and $\Phi: \R^\ell \to \R$ is a ReLU neural network, such that 
\[
\cS_\epsilon(u) = \Phi(\cL u), \quad \forall \, u \in \cX,
\]
then $\size(\Phi) \ge \exp(c \epsilon^{-1/(\alpha+1+\delta) r})$.

The idea behind the proof of this fact is that if $K\subset \cX$ contains a hypercube $Q_\alpha$, then for any $D\in \N$, a suitable rescaling of the finite-dimensional cube $[0,1]^D$ can be embedded in $K$. More precisely, for any $D \in \N$ there exists an injective linear map $\iota_D: [0,1]^D \to K$.

If we now consider the composition $\cS_\epsilon \circ \iota_D: \R^D \to \R$, then we observe that we have a decomposition $\cS_\epsilon \circ \iota_D (x) = \Phi ( \cL \circ \iota_D(x))$, where
\[
\cL \circ \iota_D: \R^D \to \R^\ell, 
\]
is linear, and 
\[
\Phi: \R^\ell \to \R,
\]
is a ReLU neural network. In particular, there exists a matrix $A\in \R^{\ell\times D}$, such that $\cL \circ \iota_D(x) = Ax$ for all $x\in \R^D$, and the mapping $\Phi_D(x) := \cS_\epsilon \circ \iota_D(x)$ defines a ReLU neural network $\Phi_D(x) = \Phi(Ax)$, whose size can be bounded by
\begin{align*}
\begin{alignedat}{3}
\size(\Phi_D ) 
&\le 2\,\size(\Phi) + 2\Vert A \Vert_0
&\text{(Equation \ref{eq:calculus1})}
\\
&\le
 2\,\size(\Phi) + 2\ell D 
 &
 \\
 &\le 2\,\size(\Phi) + 2\,\size(\Phi) D 
 \hspace{1cm}
 & \text{(Remark \ref{rem:ellsize})}
 \\
 &\le 4 \,\size(\Phi) D.
 &
 \end{alignedat}
\end{align*}

Using Proposition \ref{prop:nn-cod}, for any $D\in \N$, we are then able to construct a functional $\cF_D: K\subset \cX \to \R$, mimicking the function $f_D: [0,1]^D \subset \R^D \to \R$ constructed in Proposition \ref{prop:nn-cod}, and such that uniform approximation of $\cF_D$ by $\cS_\epsilon$ (implying similar approximation of $f_D$ by $\Phi_D$) incurs a lower complexity bound 
\[
\size(\Phi) \ge (4D)^{-1} \, \size(\Phi_D) \ge C_D \epsilon^{-\gamma D/r},
\]
where $C_D>0$ is a constant depending on $D$. For this particular functional $\cF_D$, and given the uniform lower bound on $\size(\Phi)$ above, it then follows that 
\[
\cmplx(\cS_\epsilon) \ge C_D \epsilon^{-\gamma D/r}.
\]
The first challenge is to make this strategy precise, and to determine the $D$-dependency of the constant $C_D$. As we will see, this argument leads to a lower bound of roughly the form $\cmplx(\cS_\epsilon) \gtrsim (D^{r(1+\alpha)}\epsilon)^{-\gamma D/r}$. 

At this point, the argument is still for fixed $D\in \N$, and would only lead to an algebraic complexity in $\epsilon^{-1}$. To extend this to an \emph{exponential} lower bound in $\epsilon^{-1}$, we next observe that if the estimate $\cmplx(\cS_\epsilon) \gtrsim (D^{r(1+\alpha)}\epsilon)^{-\gamma D/r}$ could in fact be established for all $D\in \N$ simultaneously, i.e. if we could construct a \emph{single} functional $\cS^\dagger$, for which the lower complexity bound $\cmplx(\cS_\epsilon) \gtrsim \sup_{D\in \N} (D^{r(1+\alpha)}\epsilon)^{-\gamma D/r}$ were to hold, then setting $D \approx (e \epsilon)^{-1/(1+\alpha) r}$ on the right would imply that 
\[
\cmplx(\cS_\epsilon) \gtrsim \sup_{D\in \N} (D^{r(1+\alpha)}\epsilon)^{-\gamma D/r} \gtrsim \exp\left(c \epsilon^{-1/(1+\alpha) r}\right),
\]
with suitable $c>0$.
Leading to an exponential lower complexity bound for such $\cS^\dagger$. The second main challenge is thus to construct a single $\cS^\dagger: K \subset \cX \to \R$ which effectively ``embeds'' an infinite family of functionals $\cF_D: K\subset \cX \to \R$ with complexity bounds as above. This will be achieved by defining $\cS^\dagger$ as a weighted sum of suitable functionals $\cF_D$. The detailed proof is provided in Appendix \ref{A:codexp}.
\end{proof}

Several remarks are in order:

\begin{remark}
Theorem \ref{thm:codexp} shows rigorously that \emph{in general}, operator learning suffers from a curse of parametric complexity, in the sense that it is not possible to achieve better than exponential complexity bounds for general classes of operators which are merely determined by their ($C^r$- or Lipschitz-) regularity. As explained above, this is a natural infinite-dimensional analogue of the curse of dimensionality in finite-dimensions (cp. Proposition \ref{prop:nn-cod}), and motivates our terminology. We note that the lower bound of Theorem \ref{thm:codexp} qualitatively matches general upper bounds for DeepONets derived in \cite{liu2024deep}. It would be of interest to determine sharp rates.
\end{remark}

\begin{remark}
Theorem \ref{thm:codexp} is derived for ReLU neural networks. With some effort, the argument could be extended to more general, piecewise polynomial activation functions. While we believe that the curse of parametric complexity has a fundamental character, we would like to point out that, \emph{for non-standard neural network architectures}, algebraic approximation rates have been obtained \cite{schwab2023deep}; these results build on either ``superexpressive'' activation functions or other non-standard architectures. Since these networks are not ordinary feedforward ReLU neural networks, the algebraic approximation rates of \cite{schwab2023deep} are not in contradiction with Theorem \ref{thm:codexp}. While the parametric complexity of the non-standard neural operators in \cite{schwab2023deep} is exponentially smaller than the lower bound of Theorem \ref{thm:codexp}, it is conceivable that storing the neural network weights in practice would require exponential accuracy (number of bits), to account for the highly unstable character of super-expressive constructions. 
\end{remark}

\begin{remark}
Theorem \ref{thm:codexp} differs from previous results on the limitations of operator learning frameworks, as e.g. addressed in \cite{nonlinrec,seidman2022nomad,lanthaler2021error}. Earlier work focuses on the limitations imposed by a linear choice of the reconstruction mapping $\cR$. In contrast, the results of the present work exhibit $C^k$-smooth operators and functionals which are fundamentally hard to approximate by neural network-based methods (with ReLU activation), irrespective of the choice of reconstruction.
\end{remark}

\begin{remark}
We finally link our main theorem to
a related result for PCA-Net \cite[Thm. 3.3]{lanthaler2023operator}, there derived in a complementary Hilbert space setting for $\cX$ and $\cY$; 
the result of \cite{lanthaler2023operator} shows that, for PCA-Net, no fixed algebraic convergence rate can be achieved in the operator learning of 
general $C^r$-operators between Hilbert spaces; this can be viewed as a milder version of the full curse of parametric complexity identified in the present work, expressed by an \emph{exponential lower complexity} bound in Theorem \ref{thm:codexp}.
\end{remark}

To further illustrate an implication of Theorem \ref{thm:codexp}, we provide the following example:

\begin{example}[Operator Learning CoD]
\label{ex:cod}
Let $\Omega \subset \R^d$ be a domain. Let $s,\rho \in \Z_{\ge 0}$ be given, with $s < \rho$, and consider the compact set 
\[
K = \set{u \in C^\rho(\Omega)}{\Vert u \Vert_{C^\rho} \le 1} \subset C^s(\Omega).
\]
Fix $r\in \N$. By Lemma \ref{lem:Cs}, $K$ contains an infinite-dimensional hypercube $Q_\alpha$ for \emph{any} $\alpha > 1 + \frac{\rho-s}{d}$. Fix such $\alpha$. Applying Theorem \ref{thm:codexp}, it follows that there exists a $r$-times Fr\'echet differentiable functional $\cS^\dagger: C^s(\Omega) \to \R$ and constant $c,\bar{\epsilon}>0$, such that any family $\cS_\epsilon: C^s(\Omega) \to \R$  of functionals of neural network-type, achieving accuracy
\[
\sup_{u\in K} | \cS^\dagger(u) - \cS_\epsilon(u) | \le \epsilon,
\quad \forall \, \epsilon \le \bar{\epsilon},
\]
has complexity at least $\cmplx(\cS_\epsilon) \ge \exp(c\epsilon^{-1/(1+\alpha) r})$ for $\epsilon \le \bar{\epsilon}$. Furthermore, the constants $c , \, \bar{\epsilon}>0$ depend only on the parameters $r,s,\rho,\alpha$.
\end{example}

In the next subsection, we aim to show the relevance of the above abstract result for concrete neural operator architectures. Specifically, we show that three operator learning architectures from the literature are of neural network-type (PCA-Net, DeepONet, NOMAD), and relate our notion of complexity to the required number of tunable parameters for each. Finally, we show that even frameworks which are not necessarily of neural network-type could suffer from a similar curse of parametric complexity. We make this explicit for the Fourier neural operator in subsection \ref{ssec:curseopFNO}.

\subsection{Examples of Operators of Neural Network-Type}
\label{ssec:examples}
We describe three representative neural operator architectures and
show that they can be cast in the above framework. 

\paragraph{PCA-Net.}
 \label{par:1}
We start with the PCA-Net architecture from \cite{bhattacharya2021model},
anticipated in the work \cite{hesthaven2018non}.
 If $\cX$ and $\cY$ are Hilbert spaces, then a neural network can be combined with principal component analysis (PCA) for the encoding and reconstruction on the underlying spaces, to define a neural operator architecture termed PCA-Net; the ingredients of this architecture are orthonormal PCA bases $\phi^{\cX}_1,\dots, \phi^{\cX}_{\DX} \in \cX$ and $\phi^{\cY}_1,\dots, \phi^{\cY}_{\DY} \in \cY$, and a neural network mapping $\Psi: \R^{\DX} \to \R^{\DY}$. The encoder $\cE$ is obtained by projection onto the $\{\phi^{\cX}_j\}_{j=1}^{\DX}$, whereas the reconstruction $\cR$ is defined by a linear expansion in $\{\phi^{\cY}_j\}_{j=1}^{\DY}$. The resulting PCA-Net neural operator is defined as 
\begin{align}
\label{eq:s-pca}
\cS(u)(y) = \sum_{k=1}^{\DY}
\Psi_k(\alpha_1,\dots, \alpha_\DX) \phi^{\cY}_k(y),
\qquad
\text{with }\alpha_j := \langle u, \phi^\cX_j\rangle,\;j=1,\dots, \DX.
\end{align}
Here the neural network $\Psi:\R^\DX \to \R^\DY$, with components 
$\Psi_k(\slot) = \Psi_k(\slot;\theta)$, depends on parameters 
which are optimized during training of the PCA-Net. The
PCA basis functions $\phi^\cY_1,\dots \phi^\cY_\DY: \Omega \to \R^{p}$, defining the reconstruction, are precomputed from the data using PCA analysis.

Given an evaluation point $y\in \Omega$, the composition $\ev_y \circ \cS$, between $\cS$ and the point-evaluation mapping $\ev_y(v) = v(y)$, can now be written in the form,
\[
\ev_y \circ \cS(u)
=
\Phi_y(\cL u),
\]
where $\cL: \cX \to \R^\DX$, $\cL u := (\langle u, \phi_1^\cX\rangle, \dots, \langle u, \phi^\cX_{\DX}\rangle$ is a linear mapping, and $\Phi_y(\alpha) := \sum_{k=1}^\DY \Psi_k(\alpha) \phi_k^\cY(y)$, for fixed $y$, is the composition of a neural network $\Psi$ with a linear read-out; thus, $\Phi_y$ is itself a neural network for fixed $y$. This shows that PCA-Net is of neural network-type.

The following lemma shows that the complexity of PCA-Net gives a lower bound on the number of free parameters for the underlying neural network architecture. 
 
\begin{lemma}[Complexity of PCA-Net]
\label{lem:c-pca}
Assume that $\cX$ and $\cY$ are Hilbert spaces, so that PCA-Net is well-defined.
Any PCA-Net $\cS = \cR \circ \Psi \circ \cE$ is of neural network-type, and 
\[
\size(\Psi) \ge (2p+2)^{-1} \cmplx(\cS),
\]
\rev{where the dimension $p\in \N$ is fixed by the output-function space $\cY = \cY(\Omega;\R^p)$.}
\end{lemma}

Thus, Lemma \ref{lem:c-pca} implies a lower complexity bound $\size(\Psi) \gtrsim \cmplx(\cS)$.
The detailed proof is given in Appendix \ref{A:c-pca}.
It thus follows from Corollary \ref{cor:codexp} that operator learning with PCA-Net suffers from a curse of parametric complexity: 
\begin{proposition}[Curse of parametric complexity for PCA-Net]
Assume the setting of Corollary \ref{cor:codexp}, with $\cX$, $\cY$ are Hilbert spaces. Then for any $r\in \N$ and $\delta > 0$, there exists $\bar{\epsilon}>0$ and an $r$-times Fr\'echet differentiable functional $\cS^\dagger: K \subset \cX \to \cY$, 
such that approximation to accuracy $\epsilon \le \bar{\epsilon}$ by a PCA-Net $\cS_\epsilon: \cX \to \cY$
\begin{align}
\sup_{u \in K} \Vert \cS^\dagger(u) - \cS_\epsilon(u) \Vert \le \epsilon,
\end{align}
with encoder $\cE$, neural network $\Psi$ and reconstruction $\cR$, implies complexity bound 
$\size(\Psi) \ge \exp(c \epsilon^{-1/(\alpha+1+\delta) r})$; here $c$, $\bar{\epsilon}>0$ are constants depending only on $\alpha$, $\delta$, $r$ and $p$.
\end{proposition}

\paragraph{DeepONet.}
 \label{par:2}

A conceptually similar approach is followed by the DeepONet 
of \cite{lu2021learning} which differs by learning the form of the
 representation in $\cY$ concurrently with the coefficients,
and by allowing for quite general input linear 
functionals $\{\alpha_j\}_{j=1}^{\DX}.$

The DeepONet architecture defines the encoding $\cE: \cX \to \R^\DX$ by a fixed choice of general linear functionals $\ell_1,\dots, \ell_\DX$; these
may be obtained, for example, by point evaluation at distinct ``sensor points''
or by projection onto PCA modes as in PCA-Net.
The reconstruction $\cR: \R^\DY \to \cY$ is defined by expansion 
with respect to a set of functions $\phi_1,\dots, \phi_\DY \in \cY$ 
which are themselves finite dimensional neural networks to be learned.
The resulting DeepONet can be written as
\begin{align}
\label{eq:s-don}
\cS(u)(y) = \sum_{k=1}^{\DY}
\Psi_k(\alpha_1,\dots, \alpha_\DX) \phi_k(y),
\qquad
\text{with }\alpha_j := \ell_j(u),\;j=1,\dots, \DX.
\end{align}
Here, both the neural networks $\Psi:\R^\DX \to \R^\DY$, with components $\Psi_k = \Psi_k(\slot;\theta)$, and $\phi: \Omega \to \R^{p \times \DY}$ with components $\phi_k = \phi_k(\slot;\theta)$, depend on parameters which are optimized during training of the DeepONet.

Given a evaluation point $y\in \Omega$, the composition $\ev_y \circ \cS$, with the point-evaluation mapping $\ev_y(v) = v(y)$ can again be written in the form,
\[
\ev_y \circ \cS(u)
=
\Phi_y(\cL u),
\]
where $\cL: \cX \to \R^\DX$, $\cL u := (\ell_1(u), \dots, \ell_{\DX}(u))$ is linear, where 
$$\Phi_y(\alpha) := \sum_{k=1}^\DY \Psi_k(\alpha) \phi_k(y),$$ and for fixed $y$ the values $\phi_k(y) \in \R^p$ are just (constant) vectors. Thus, $\Phi_y$ is the composition of a neural network $\Psi$ with a linear read-out, and hence is itself a neural network. This shows that DeepONet is of neural network-type.

The next lemma shows that the size can be related to the 
complexity of DeepONet: Also in this case, $\cmplx(\cS)$ provides a lower bound on
the total number of non-zero degrees of freedom of a DeepONet.
The detailed proof is provided in Appendix \ref{A:c-don}.

\begin{lemma}[Complexity of DeepONet]
\label{lem:c-don}
Any DeepONet $\cS = \cR \circ \Psi \circ \cE$, defined by a branch-net $\Psi$ and trunk-net $\phi$, is of neural network-type, and 
\[
2(\size(\Psi) + \size(\phi)) \ge \cmplx(\cS).
\]
\end{lemma}

The following result is now an immediate consequence of Corollary \ref{cor:codexp} and the above lemma.

\begin{proposition}[Curse of parametric complexity for DeepONet]
Assume the setting of Corollary \ref{cor:codexp}. Then for any $r\in \N$ and $\delta > 0$, there exists $\bar{\epsilon}>0$ and an $r$-times Fr\'echet differentiable functional $\cS^\dagger: K \subset \cX \to \cY$, 
such that approximation to accuracy $\epsilon \le \bar{\epsilon}$ by a DeepONet $\cS_\epsilon: \cX \to \cY$
\begin{align}
\sup_{u \in K} \Vert \cS^\dagger(u) - \cS_\epsilon(u) \Vert \le \epsilon,
\end{align}
with branch net $\Psi$ and trunk net $\phi$, implies complexity bound 
$\size(\Psi) + \size(\phi) \ge \exp(c \epsilon^{-1/(\alpha+1+\delta) r})$; here $c$, $\bar{\epsilon}>0$ are constants depending only on $\alpha$, $\delta$ and $r$.
\end{proposition}

\paragraph{NOMAD}
 \label{par:3}

The linearity in the reconstruction in $\cY$ for both PCA-Net and DeepONet imposes a fundamental limitation on their approximation capability \cite{lanthaler2021error,nonlinrec,seidman2022nomad}. To overcome this limitation, nonlinear extensions of DeepONet have recently been proposed. The following 
NOMAD architecture \cite{seidman2022nomad} provides an example:

 NOMAD (NOnlinear MAnifold Decoder) employs encoding by point evaluation at a fixed set of sensor points, or more general linear functionals $\ell_1,\dots, \ell_\DX: \cX \to \R$. The reconstruction $\cR: \R^\DY \to \cY$ in the output space $\cY = \cY(\Omega;\R^p)$ is defined via a neural network $Q: \R^{\DY} \times \Omega \to \R^p$, which depends jointly on encoded output coefficients in $\R^{\DY}$ and the evaluation point $y\in \Omega$; as in the two previous 
examples, $\Psi:\R^\DX \to \R^\DY$ is again a neural network.
The resulting NOMAD mapping, for $u\in \cX$ and $y\in \Omega$, is given by 
\begin{align}
\label{eq:s-nomad}
\cS(u)(y) = Q(\Psi(\alpha_1,\dots, \alpha_\DX),y),
\qquad
\text{with }\alpha_j := \ell_j(u),
\end{align}
for $j=1,\dots, \DX$. We note that the main difference between DeepONet and NOMAD is that the linear expansion in \eqref{eq:s-don} has been replaced by a nonlinear composition with the neural network $Q$. Both neural networks $\Psi$ and $Q$ are optimized during training.

Given evaluation point $y\in \Omega$, the composition $\ev_y \circ \cS$ with the point-evaluation can be written in the form,
\[
\ev_y \circ \cS(u)
=
\Phi_y(\cL u),
\]
where $\cL: \cX \to \R^\DX$, $\cL u := (\ell_1(u), \dots, \ell_\DX(u))$ is linear, and $\alpha \mapsto \Phi_y(\alpha) := Q(\Psi_k(\alpha),y)$ defines a neural network for fixed $y$. This shows that NOMAD is of neural network-type. Finally, the following lemma provides an estimate on the complexity 
of NOMAD: 

\begin{lemma}[Complexity of NOMAD]
\label{lem:c-nomad}
Any NOMAD operator $\cS = \cR \circ \Psi \circ \cE$ defined by a branch-net $\Psi$ and non-linear reconstruction $Q$ is of neural network-type, and 
\begin{align}
\label{eq:c-nomad}
2(\size(\Psi) + \size(Q)) \ge \cmplx(\cS).
\end{align}
\end{lemma}
The expression on the left hand-side of \eqref{eq:c-nomad} 
represents the total number of non-zero degrees of freedom in the NOMAD 
architecture and, as for PCA-Net and DeepONet, it
is lower bounded by our notion of complexity.
For the proof, we refer to Appendix \ref{A:c-nomad}.

The following result is now an immediate consequence of Corollary \ref{cor:codexp} and the above lemma.

\begin{proposition}[Curse of parametric complexity for NOMAD]
Assume the setting of Corollary \ref{cor:codexp}. Then for any $r\in \N$ and $\delta > 0$, there exists $\bar{\epsilon}>0$ and an $r$-times Fr\'echet differentiable functional $\cS^\dagger: K \subset \cX \to \cY$, 
such that approximation to accuracy $\epsilon \le \bar{\epsilon}$ by a NOMAD neural operator $\cS_\epsilon: \cX \to \cY$
\begin{align}
\sup_{u \in K} \Vert \cS^\dagger(u) - \cS_\epsilon(u) \Vert \le \epsilon,
\end{align}
with neural networks $\Psi$ and $Q$, implies complexity bound 
$\size(\Psi) + \size(Q) \ge \exp(c \epsilon^{-1/(\alpha+1+\delta) r})$; here $c$, $\bar{\epsilon}>0$ are constants depending only on $\alpha$, $\delta$ and $r$.
\end{proposition}

\paragraph{Discussion.}

For all examples above, a general lower bound on the complexity $\cmplx(\cS)$ of operators of neural network-type implies a lower bound on the total number of degrees of freedom of the particular architecture. In particular, a lower bound on $\cmplx(\cS)$ gave a lower bound on the smallest possible number of non-zero parameters that are needed to implement $\cS$ in practice. This observation motivates our nomenclature for the complexity. 

We emphasize that our notion of complexity only relates to the size of the neural network at the core of these architectures; by design, it does not take into account other factors, such as the additional complexity associated with the practical evaluation of inner products in the PCA-Net architecture, evaluation of linear encoding functionals of DeepONets, or the numerical representation of an (optimal) output PCA basis for PCA-Net or neural network basis for DeepONet. The important point is that the aforementioned factors can only increase the overall complexity of any concrete implementation; correspondingly, our proposed notion of $\cmplx(\cS)$, which neglects some of these additional contributions, can be used to derive rigorous \emph{lower} bounds on the overall complexity of any implementation.

\begin{remark}
In passing, we point out similar approaches \cite{HUA202321,nonlinrec,patel2024variationally,zhang2023belnet} to the PCA-Net, DeepONet and NOMAD architectures, which share a closely related underlying structure to the examples given above. We fully expect that the curse of parametric complexity applies to all of these architectures.
\end{remark}

In the next subsection \ref{ssec:curseopFNO} we will show that
even for operator architectures which are not of neural network-type according to the above definition, we may nevertheless be able to link them with an associated operator of neural network-type. Specifically, we will show this for the FNO in Theorem  \ref{thm:fno-cod}. There, we will see that the size (number of tunable parameters) of the FNO can be linked to the complexity of an associated operator of neural network-type. And hence lower bounds on the complexity of operators of neural network-type imply corresponding lower bounds on the FNO.

\subsection{The Curse of (Parametric) Complexity for Fourier Neural Operators.}
\label{ssec:curseopFNO}

The definition of \emph{operators of neural network-type} introduced
in the previous subsection does not include the FNO, a widely adopted
neural operator architecture. However, in this subsection we show that a 
result similar to Theorem \ref{thm:codexp}, stated as Theorem
 \ref{thm:fno-cod} below, can be obtained for the FNO. 

Due to intrinsic constraints on the domain on which FNO can be (readily) applied, we will assume that the spatial domain $\Omega = \prod_{j=1}^d [a_j,b_j] \subset \R^d$ is rectangular.
We recall that an FNO,
\[
\cS: 
\cX(\Omega; \R^k) \to \cY(\Omega; \R^p),
\]
can be written as a composition, $\cS = \cQ \circ \cL_L \circ \dots \circ \cL_1 \circ \cP$, of a lifting layer $\cP$, hidden layers $\cL_1,\dots, \cL_L$ and an output layer $\cQ$. In the following, let us denote by $\cV(\Omega;\R^{d_v})$ a generic space of functions from $\Omega$ to $\R^{d_v}$.

The nonlinear \define{lifting layer}
\[
\cP: \cX(\Omega;\R^k) \to \cV(\Omega; \R^{d_v}), 
\quad
u(x) \mapsto \chi(x,u(x)),
\]
is defined by a neural network $\chi: \Omega \times \R^k \to \R^{d_v}$, 
depending jointly on the evaluation point $x \in \Omega$ and 
the components of the input function $u$ evaluated at $x$,
namely $u(x) \in \R^k$. The dimension $d_v$ is a free hyperparameter, and determines the number of ``channels'' (or the ``lifting dimension''). 

Each \define{hidden layer} $\cL_\ell$, $\ell=1,\dots, L$, of an FNO is of the form 
\[
\cL_\ell: \cV(\Omega; \R^{d_v}) \to \cV(\Omega; \R^{d_v}), 
\quad
v \mapsto \sigma\left(W_\ell v + K_\ell v + b_\ell \right),
\]
where $\sigma$ is a nonlinear activation function applied componentwise, $W_\ell \in \R^{d_v \times d_v}$ is a matrix, the bias $b_\ell\in \cV(\Omega;\R^{d_v})$ is a function and $K_\ell: \cV(\Omega; \R^{d_v}) \to \cV(\Omega; \R^{d_v})$ is a non-local operator defined in terms of \define{Fourier multipliers},
\[
Kv(x) = \cF^{-1} ( \hat{T}_k [\cF v]_k)(x),
\]
where $[\cF v]_k$ denotes the $k$-th Fourier coefficient of $v$ for $k \in \Z^d$, $\hat{T}_k \in \C^{d_v\times d_v}$ is a Fourier multiplier matrix indexed by $k$, and $\cF^{-1}$ denotes the inverse Fourier transform. In practice, a Fourier cut-off $k_{\mathrm{max}} \in \Z$ is introduced, and only a finite number of Fourier modes\footnote{Throughout this paper $\Vert \cdot  \Vert_{\ell^\infty}$ denotes
the maximum norm on finite dimensional Euclidean space.} $\Vert k \Vert_{\ell^\infty} \le k_{\mathrm{max}}$ is retained. In particular, the number of non-zero components of $\hat{T} = \{\hat{T}_k\}_{k\in \Z^d}$ is bounded by $\Vert \hat{T} \Vert_0 \le d_v^2 (2k_{\mathrm{max}}+1)^d$. In the following, we will also assume that the bias functions $b_\ell$ are determined by their Fourier components $[\hat{b}_\ell]_k \in \C^{d_v}$, $\Vert k \Vert_{\ell^\infty} \le k_{\mathrm{max}}$. 

Finally, the \define{output layer} 
\[
\cQ: \cV(\Omega; \R^{d_v}) \to \cY(\Omega; \R^p),
\quad
v \mapsto q(x,v(x)),
\]
is defined in terms of a neural network $q: \Omega \times \R^{d_v} \to \R^p$, a joint function of the evaluation point $x\in \Omega$ and the components
of the output $v$ of the previous layer evaluated at $x$, namely
$v(x) \in \R^{d_v}$.

To define the \define{size of an FNO}, we note that its
tunable parameters are given by: \emph{(i)} the weights and biases of the neural network $\chi$ defining the lifting layer $\cR$, \emph{(ii)} the components of the matrices $W_\ell \in \R^{d_v\times d_v}$, \emph{(iii)} the components of the Fourier multipliers $\hat{T}_k \in \C^{d_v\times d_v}$ for $\Vert k \Vert_{\ell^\infty} \le k_{\mathrm{max}}$, \emph{(iv)} the Fourier coefficients $[\hat{b}_\ell]_k\in \C^{d_v}$, for $\Vert k \Vert_{\ell^\infty} \le k_{\mathrm{max}}$, and \emph{(v)} the number of weights and biases of the neural network $q$ defining the output layer $\cQ$. Given an FNO $\cS: \cX(\Omega;\R^k) \to \cY(\Omega;\R^p)$, we define $\size(\cS)$ of an FNO as the total number of non-zero parameters in this construction. We also follow the convention that for a matrix (or vector) $A$ with complex entries, the number of parameters is defined as $\Vert A \Vert_0 = \Vert \Re(A) \Vert_0 + \Vert \Im(A) \Vert_0$.

\begin{remark}[FNO approximation of functionals]
\label{rem:FNO-func}
If $\cS^\dagger: \cX(\Omega;\R^k) \to \R$ is a (scalar-valued) functional, then we will again identify the output-space $\cY(\Omega;\R^p)$ with a space of constant functions. In this case, it is natural to add an averaging operation after the last output layer $\cQ: \cV(\Omega;\R^{d_v}) \to \cY(\Omega;\R)$, i.e. we replace $\cQ$ by $\tQ: \cV(\Omega;\R^{d_v}) \to \R$, given by
\begin{align}
\label{eq:avg}
\tQ(v) := \frac{1}{|\Omega|} \int_{\Omega} q(x,v(x)) \, dx.
\end{align}
This does not introduce any additional degrees of freedom, and ensures that the output is constant. We also note that \eqref{eq:avg} is a special case of a Fourier multiplier $K$, involving only the $k=0$ Fourier mode.
\end{remark}

In the following, we will restrict attention to the approximation of functionals, taking into account Remark \ref{rem:FNO-func}. 
We first mention the following result, proved in Appendix \ref{A:fno-nnt}, which shows that FNOs are not of neural network-type, in general:
\begin{lemma}
\label{lem:fno-nnt}
Let $\sigma$ be the ReLU activation function. Let $\T \simeq [0,2\pi]$ denote the $2\pi$-periodic torus. The FNO,
\[
\cS: L^2(\T;\R) \to \R, \quad \cS(u) := \int_{\Omega} \sigma(u(x)) \, dx,
\]
is not of neural network-type.
\end{lemma}
The fact that FNO is not of neural network-type is closely related to the fact that the Fourier transforms at the core of the FNO mapping $\cS: \cX(\Omega;\R^k)\to \R$ cannot be computed exactly. In practice, the FNO therefore needs to be \define{discretized}.

A simple discretization $\cS^N$ of $\cS$ is readily obtained and commonly used in applications of FNOs. To this end, fix $N\in \N$, and let $x_{j_1,\dots, j_d} \in \Omega$, $j_1,\dots, j_d \in \{1,\dots, N\}$ be a grid consisting of $N$ equidistant points in each coordinate direction. A numerical approximation $\cS^N: \cX(\Omega;\R^k) \to \R$ of $\cS$ is obtained by replacing the Fourier transform $\cF$ and its inverse $\cF^{-1}$ in each hidden layer by the \emph{discrete} Fourier transform $\cF_N$ and its inverse $\cF_N^{-1}$, computed on the equidistant grid. Similarly, the exact average \eqref{eq:avg} is replaced by an average over the grid values. This ``discretized'' FNO $\cS^N$ thus defines a mapping 
\[
\cS^N: \cX(\Omega; \R^k) \to \R,
\quad
u \mapsto 
\cS^N(u),
\]
which depends only on the grid values $u(x_{j_1,\dots, j_d}) \in \R^k$ of the input function $u$, for multi-indices $(j_1,\dots, j_d) \in \{1,\dots, N\}^d$. In contrast to $\cS$, the discretization $\cS^N$ is readily implemented in practice. We expect that
$\cS^N({u}) \approx \cS(u)$
for sufficiently large $N$. Note also that $\size(\cS^N)=\size(\cS)$ by construction.

Given these preparatory remarks, we can now state our result on the curse of parametric complexity for FNOs, with proof in Appendix \ref{A:cod2}.



\begin{theorem}
 \label{thm:fno-cod}
Let $\Omega \subset \R^d$ be a cube. Let $K \subset \cX$ be a compact subset of a Banach space $\cX = \cX(\Omega;\R^k)$. Assume that $K$ contains an infinite-dimensional hypercube $Q_\alpha$ for some $\alpha > 1$. Then for any $r\in \N$ and $\delta > 0$, there exists $\bar{\epsilon}>0$ and an $r$-times Fr\'echet differentiable functional $\cS^\dagger: K \subset \cX \to \R$, 
such that approximation to accuracy $\epsilon \le \bar{\epsilon}$ by a discretized FNO $\cS_\epsilon^{N_\epsilon}$, 
\[
\sup_{u \in K} | \cS^\dagger(u) - \cS_\epsilon^{N_\epsilon}(u) | \le \epsilon,
\]
requires either (i) complexity bound $\size(\cS_\epsilon^{N_\epsilon}) \ge \exp(c \epsilon^{-1/(\alpha+1+\delta) r})$, or (ii) discretization parameter $N_\epsilon \ge \exp(c \epsilon^{-1/(\alpha+1+\delta) r})$; here $c$, $\bar{\epsilon}>0$ are constants depending only on $\alpha$, $\delta$ and $r$.
\end{theorem}

\begin{proof}{(\emph{Sketch})}
The proof of Theorem  \ref{thm:fno-cod} relies on the curse of parametric complexity for operators of neural network-type in Theorem \ref{thm:codexp}. The first step is to show that discrete FNOs are of neural network-type. As a consequence, Theorem \ref{thm:codexp} implies a lower bound of the form $\cmplx(\cS_\epsilon^{N_\epsilon}) \ge \exp(c \epsilon^{-1/(\alpha+1+\delta) r})$. The main additional difficulty is that the discrete FNO is not a standard ReLU neural network according to our definition, since it employs a very non-standard architecture involving convolution and Fourier transforms. Hence more work is needed, in order to relate $\cmplx(\cS_\epsilon^{N_\epsilon})$ to $\size(\cS_\epsilon^{N_\epsilon})$; while the complexity is the minimal number of parameters required to represent $\cS_\epsilon^{N_\epsilon}$ by an \emph{ordinary} ReLU network architecture, we recall that the $\size(\cS_\epsilon^{N_\epsilon})$ is defined as the number of parameters defining $\cS_\epsilon^{N_\epsilon}$ via the \emph{non-standard} FNO architecture. Our analysis leads to an upper bound of the form 
\begin{align}
\label{eq:cmplx-size}
 \rev{
 \cmplx(\cS_\epsilon^{N_\epsilon}) 
\lesssim N_\epsilon^{2d} \size(\cS_\epsilon^{N_\epsilon}).
}
\end{align}
As a consequence of the exponential lower bound, $\cmplx(\cS_\epsilon^{N_\epsilon}) \ge \exp(c \epsilon^{-1/(\alpha+1+\delta) r})$, inequality \eqref{eq:cmplx-size} implies that either $N_\epsilon$ or $\size(\cS_\epsilon^{N_\epsilon})$ have to be exponentially large in $\epsilon^{-1}$, proving the claim.
For the detailed proof, we refer to Appendix \ref{A:cod2}.
\end{proof}

\rev{
\begin{remark}
We would like to point out that the additional factor in \eqref{eq:cmplx-size}, depending on $N_\epsilon$, is natural in view of the fact that even a linear discretized FNO layer $v \mapsto Wv$, with matrix $W \in \R^{d_v\times d_v}$, actually corresponds to a mapping
$\R^{N_\epsilon^d \times d_v} \to \R^{N_\epsilon^d\times d_v}$, $(v(x_{j_1,\dots, j_d})) \mapsto ( Wv(x)_{j_1,\dots,j_d})$; i.e. the matrix multiplication should be viewed as being carried out in parallel at all $N_\epsilon^d$ grid points. Representing this mapping by an ordinary matrix multiplication $\R^{N_\epsilon^d \times d_v} \to \R^{N_\epsilon^d\times d_v}$ requires $N_\epsilon^d \Vert W\Vert_0$ degrees of freedom, and thus depends on $N_\epsilon$ in addition to the number of FNO parameters $\Vert W \Vert_0$. 
\end{remark}
}

\begin{remark}
\label{rem:MTI}
Theorem  \ref{thm:fno-cod} shows that FNO suffers from a similar curse of complexity as do the variants on DeepONet and PCA-Net covered by Theorem \ref{thm:codexp}: approximation to accuracy $\epsilon$ of general ($C^r$- or Lipschitz-) operators requires either an exponential number of non-zero degree of freedom, or requires exponential computational resources to evaluate even a single forward pass. We note the difference in how the curse is expressed in Theorem  \ref{thm:fno-cod} compared to Theorem \ref{thm:codexp}; this is due to the fact that FNO is not of neural network-type (see Lemma \ref{lem:fno-nnt}). Instead, as outlined in the proof sketch above, only upon discretization does the FNO define an operator/functional of neural network type. The computational complexity of this discretized FNO is determined by both the FNO coefficients and the discretization parameter $N$.
\end{remark}

\subsubsection{Discussion}
To overcome the general curse of parametric complexity implied by Theorem \ref{thm:codexp} (and Theorem  \ref{thm:fno-cod}), efficient operator learning frameworks therefore have to leverage additional structure present in the operators of interest, going beyond $C^r$- or Lipschitz-regularity. Previous work on overcoming the curse of parametric complexity for operator learning has mostly focused on operator holomorphy \cite{herrmann2020deep,schwab2019deep,lanthaler2021error} and the emulation of numerical methods \cite{deng2021convergence,kovachki2021universal,lanthaler2021error} as two basic mechanisms for overcoming the curse of parametric complexity for specific operators of interest.  An abstract characterization of the entire class of operators that allow for efficient approximation by neural operators would be very desirable. Unfortunately, this appears to be out of reach, at the current state of analysis.
Indeed, as far as the authors are aware, there does not even exist such a 
characterization for any class of standard numerical methods, 
such as finite difference, finite element or spectral, viewed as operator approximators. 

The second contribution of the present work is to expose additional structure, different from holomorphy and emulation, that can be leveraged by neural operators. In particular,
in the remainder of this paper we identify such structure in the Hamilton-Jacobi equations, and propose a neural operator framework which can build on this structure to provably overcome the curse of parametric complexity 
in learning the associated solution operator.

\section{The Hamilton-Jacobi Equation}
\label{sec:S}


In the previous section we demonstrate that, generically, a scaling-limit of the curse
of dimensionality is to be expected in operator approximation, if only
$C^r-$ regularity of the map is assumed. However we also outlined in the
introduction the many cases where specific structure can be exploited to
overcome this curse. In the remainder of the paper we show how the curse
can be removed for Hamilton-Jacobi equations. To this end we start, in this
section, by setting up the theoretical framework for operator learning
in this context.

We are interested in deriving error and complexity estimates for the approximation of the solution operator $\cS_t^\dagger: C^r_\per(\Omega) \to C^r_\per(\Omega)$ associated with the following Hamilton-Jacobi PDE:
\begin{align}
\left\{
\begin{aligned}
\partial_t u + H(q,\nabla_q u) &= 0, (x,t) \in \Omega \times (0,T],\\
u(t=0) &= u_0, (x,t) \in \Omega \times \{0\} 
\end{aligned}
\right.
\tag{HJ}
\label{eq:HJ}
\end{align}
To simplify our analysis we consider only the case of a domain $\Omega = [0,2\pi]^d$ with periodic boundary conditions (so that we may identify $\Omega$ with $\mathbb{T}^d$.) We denote by $C^r_\per(\Omega)$ the space of $r$-times continuously differentiable real-valued functions with $2\pi$-periodic derivatives, with norm 
\[
\Vert u \Vert_{C^r(\Omega)} := \sup_{|\alpha|\le r} \sup_{x\in \Omega} |D^\alpha u(x)|.
\]
By slight abuse of notation, we will similarly denote by $u(q,t) \in C^r_\per(\Omega \times [0,T])$ and $H(q,p) \in C^r_\per(\Omega \times \R^d)$ functions that have $C^r$ regularity in all variables, and that are periodic in the first variable $q\in \Omega$, so that in particular,
\[
q \mapsto u(q,t), \quad \text{and} \quad  q\mapsto H(q,p),
\]
belong to $C^r_\per(\Omega)$, for fixed $p\in \R^d$ or $t\in [0,T].$ 

In equation \eqref{eq:HJ}, $u: \Omega \times [0,T] \to \mathbb{R}$, $(q,t) \mapsto u(q,t)$ is a function, depending on the spatial variable $q\in \Omega$ and on time $t\ge0$. We will restrict attention to problems for which a \emph{classical solution} $u \in C^r_\per(\Omega \times [0,T])$, $r\ge 2$, exists. In this setting the initial value problem \eqref{eq:HJ} can be solved by the method of characteristics and a unique solution may be proved to exist for some $T=T(u_0)>0.$ We may
then define solution operator $\cS_t^\dagger$ with property
$u(\cdot,t)=\cS_t^\dagger(u_0);$ \rev{the maximal time $T$ of existence} will,
in general, depend on the input $u_0$ to $\cS_t.$ The next two subsections describe, respectively, the method of characteristics and
the existence of the solution operator on a time-interval $t \in [0,T]$,
for all initial data from compact $\cF$ in $C^r_\per(\Omega),$ for $r \ge 2.$ 
Thus we define
$\{\cS_t^\dagger: \cF \subset C^r_\per(\Omega) \to C^r_\per(\Omega)\}$
for all $t \in [0,T]$, $T=T(\cF)$ sufficiently small.

We recall that throughout this paper, \rev{use of a superscript
$\dagger$ denotes an object defined through construction of an exact solution of \eqref{eq:HJ},} or objects used in the construction of the solution; 
identical notation without a superscript $\dagger$ denotes an approximation of that object.

\subsection{Method of Characteristics for \eqref{eq:HJ}}

We briefly summarize this methodology; for more details see \cite[Section 3.3]{evans}.
Consider the following Hamiltonian system for $t \mapsto (q(t),p(t)) \in \Omega \times \R^d$
defined by
\begin{subequations}
\label{eq:traj}
\begin{align}
&\left\{
\begin{aligned}
\dot{q} &= \nabla_p H(q,p), \quad q(0) = q_0,\\
\dot{p} &= -\nabla_q H(q,p), \quad p(0) = p_0,
\end{aligned}
\right. \\
&\quad p_0=\nabla_qu_0(q_0).
\end{align}
\end{subequations}
If the solution $u(q,t)$ of \eqref{eq:HJ} is twice continuously differentiable, then  $u$ can be evaluated by solving the ODE (\ref{eq:traj}a) with the
specific parameterized initial conditions (\ref{eq:traj}b).
Given these trajectories and $t\ge 0$, the values $u(q(t),t)$ can be computed in terms of the ``action'' along this trajectory:
\begin{align}
\label{eq:quad}
u(q(t),t) = u_0(q_0) + \int_0^t \cL(q(t),p(t)) \, d\tau,
\end{align}
where $\cL: \Omega \times \R^d \to \R$ is the Lagrangian associated with $H$, i.e.
\begin{align}
\cL(q,p) := p\cdot \nabla_p H(q,p) - H(q,p).
\end{align}

Equivalently, the solution $z(t) = u(q(t),t)$ can be expressed in terms of the solution of the following system of ODEs, $t\mapsto (q(t),p(t),z(t))$:
\begin{subequations}
\label{eq:flow}
\begin{align}
&\left\{
\begin{aligned}
\dot{q} &= \nabla_pH(q,p), \quad q(0) = q_0,\\
\dot{p} &= -\nabla_qH(q,p), \quad p(0) = p_0,\\
\dot{z} &= \cL(q,p), \quad z(0) = z_0.
\end{aligned}
\right. \\
&\quad p_0=\nabla_qu_0(q_0), \quad z_0=u_0(q_0).
\end{align}
\end{subequations}
The system of ODEs \eqref{eq:flow} is defined on $\Omega \times \R^d \times \R$, with a $2\pi$-periodic spatial domain $\Omega = [0,2\pi]^d$. It can be shown that 
\begin{align}
\label{eq:p}
p(t) \equiv \nabla_q u(q(t),t), \quad \text{for $t\ge 0$,} 
\end{align} 
tracks the evolution of the gradient of $u$ along this trajectory. 
To ensure the existence of solutions to \eqref{eq:traj}, i.e. to avoid trajectories escaping to infinity, we make the following assumption on $H$,
in which $|\slot|$ denotes the Euclidean distance on $\R^d$:

\begin{assumption}[Growth at Infinity]
\label{ass:blowup}
There exists $L_H>0$, such that
\begin{align}
\label{eq:blowup}
\sup_{q \in \Omega} \left\{ -p\cdot \nabla_q H(q,p)\right\} \le L_H (1+|p|^2),
\end{align}
for all $p\in \R^d$.
\end{assumption}

In the following, we will denote by 
\begin{align}
\label{eq:flowmap}
\left\{
\begin{aligned}
\Psi_t^\dagger: \Omega \times \R^d \times \R &\to \Omega \times \R^d \times \R, \\
(q_0,p_0,z_0) &\mapsto (q(t),p(t),z(t)),
\end{aligned}
\right.
\end{align}
the semigroup (flow map) generated by the system 
of ODEs (\ref{eq:flow}a); 
existence of this semigroup, and hence of the solution operator 
$\cS_t^\dagger$, is the topic of the next subsection.

\subsection{Short-time Existence of $C^r$-solutions}

The goal of the present section is to show that for any $r\ge 2$, and for any compact subset $\cF \subset C^r_\per(\Omega)$, there exists a maximal $T^\ast = T^\ast(\cF)>0$, such that for any $u_0 \in \cF$, there exists a solution $u: \Omega \times [0,T^\ast) \to \R$ of \eqref{eq:HJ}, with property $q \mapsto u(q,t)$ belongs to $C^r_\per(\Omega)$ for all $t\in [0,T^\ast)$. Thus we may take any $T<T^\ast$ in \eqref{eq:HJ}, for data in $\cF$. Our proof of this fact relies on the Banach space version of the implicit function theorem; see 
Appendix \ref{A:E}, Theorem \ref{thm:implicit}, for a precise statement.

In the following, given $t\ge 0$ and given initial values $(q_0,p_0) \in \Omega \times \R^d$, we define $q_t(q_0,p_0)$, $p_t(q_0,p_0)$ as the spatial- and momenta-components of the solution of the Hamiltonian ODE (\ref{eq:traj}a), respectively; 
i.e. the solution of (\ref{eq:traj}a) is given by 
$t \mapsto (q_t(q_0,p_0), p_t(q_0,p_0))$ 
for any initial data $(q_0,p_0) \in \Omega \times \R^d$.

\begin{proposition}[{\bf Short-time Existence of Classical Solutions}]
\label{prop:short-time}
Let Assumption \ref{ass:blowup} hold, let $r\ge 2$ and assume that the Hamiltonian $H\in C^{r+1}_\per(\Omega\times \R^d)$. Then for any compact subset $\cF \subset C^r_\per(\Omega)$, there exists $T^\ast = T^\ast(\cF) > 0$ such that for any $u_0\in \cF$, there exists a classical solution $u\in C^r_\per(\Omega\times [0,T^\ast))$ of the Hamilton-Jacobi equation \eqref{eq:HJ}. Furthermore, for any $T<T^\ast$, there exists a constant $C = C(T,r,\cF,H) > 0$ such that $\sup_{t\in[0,T]}\Vert u(\slot,t) \Vert_{C^r} \le C$.
\end{proposition}

The proof, which may be found in Appendix \ref{A:E}, uses the following
two lemmas, also proved in Appendix \ref{A:E}. The 
first result shows that, under Assumption \ref{ass:blowup}, the semigroup 
$\Psi_t^\dagger$ in \eqref{eq:flowmap} is well-defined, globally in time.

\begin{lemma}
\label{lem:traj-existence}
Let $H\in C_\per^{r}(\Omega \times \R^d)$ for $r\ge 1$. Let Assumption \ref{ass:blowup} hold. Then the mapping $\Psi_t^\dagger$ given by \eqref{eq:flowmap} exists for any $t\ge 0$ and $t\mapsto \Psi_t^\dagger$ defines a semigroup. In particular, for any $(q_0,p_0) \in  \Omega\times \R^d$, there exists a solution $t\mapsto (q(t),p(t))$ of the ODE system (\ref{eq:traj}a)
with initial data $q(0) = q_0$, $p(0) = p_0$, for all $t\ge 0$. 
\end{lemma}

The following result will be used to show that the method of characteristics 
can be used to construct solutions, at least over a sufficiently 
short time-interval.

\begin{lemma}
\label{lem:invert}
Let Assumption \ref{ass:blowup} hold, let $r\ge 2$ and assume that the Hamiltonian $H \in C^{r+1}_\per(\Omega\times \R^d)$. Let $\cF \subset C^r_\per(\Omega)$ be compact. Then there exists a (maximal) time $T^\ast = T^\ast(\cF) > 0$, such that for all $u_0\in \cF$, the {\em spatial characteristic mapping} 
\[
\Phi_t^\dagger(\slot;u_0): \Omega \to \Omega, \quad q_0 \mapsto q_t\bigl(q_0,\nabla_qu_0(q_0)\bigr),
\]
defined by \eqref{eq:traj} on the periodic domain $\Omega = [0,2\pi]^d$, is a $C^{r-1}$-diffeomorphism for any $t\in [0,T^\ast)$.
\end{lemma}

Note that the map $\Phi_t^\dagger(\slot;u_0)$ is defined by the semigroup 
$\Psi_t^\dagger$ for $(q,p,z)$; however it only requires solution of 
the Hamiltonian system \eqref{eq:traj} for $(q,p)$.
 In contrast to the flowmap $\Psi^\dagger_t$, the spatial characteristic map $\Phi_t^\dagger(\slot; u_0)$ is in general only invertible over a sufficiently short time-interval, leading to a corresponding short-time existence result for the solution operator $\cS^\dagger_t$ associated with \eqref{eq:HJ} in Proposition \ref{prop:short-time}.

\subsection{Limitations of the setting considered in this work}
\label{sec:shorttime}
\rev{In this paper we study the curse of parametric complexity for operator learning, and discuss settings where it can be overcome. The primary role of our study of Hamilton-Jacobi equations is
to highlight a setting where operator learning does not suffer from this curse. However our analysis
is restricted to the setting of smooth solutions and non-intersecting characteristics.
We close this section on Hamilton-Jacobi equations by highlighting this important limitation of  the present work, and mention a possible extension.

The main limitation of this work is the assumption that classical solutions of \eqref{eq:HJ} exist. Indeed, it is well-known, e.g. \cite{albano2002propagation,albano2020generation}, that classical solutions of the Hamilton-Jacobi equations can develop singularities in finite time due to the potential crossing of the spatial characteristics $q(t)$ emanating from different points; indeed, if two spatial characteristics emanating from $q_0$ and $q_0'$ cross in finite time, then \eqref{eq:quad} does not lead to a well-defined function $u(q,t)$, and the method of characteristics breaks down. Thus, our existence theorem is generally restricted to a finite time interval $[0,T^\ast]$. 

In certain special cases, the crossing of characteristics is ruled out, and classical solutions exist for all time, i.e. with $T^\ast =\infty$.
One such example is the advection equation with velocity field $v(q)$, 
and corresponding Hamiltonian $H(q,\nabla_q u) = v(q) \cdot \nabla_q u$, for which the complexity of operator approximation is studied computationally in \cite{de2022cost}. 

Finding more general conditions for the existence of classical solutions is a delicate question, which is discussed in \cite{barron1999regularity}; under certain convexity assumptions on the Hamiltonian, it is possible to relate the existence of a classical solution of \eqref{eq:HJ} to its time-reversibility. For example, it is shown in \cite[Thm. 2.5]{barron1999regularity} that if $H(q,p)$ is (i) strictly convex and superlinear in $p$, and (ii) Lipschitz in $q$, and $u(x,t)$ is a locally Lipschitz continuous (viscosity) solution of \eqref{eq:HJ} both forward and backward in time, then $u$ must be smooth. Under certain technical assumptions on the convexity of $H$ and the initial data $u(t=0)$, it can in fact be shown \cite{barron1999regularity} that a forward solution $u(x,t)$ and a backward solution $w(x,t)$ of \eqref{eq:HJ} must be equal everywhere, provided that they coincide at the endpoints, i.e. $u(x,T) = w(x,T)$ and $u(x,0) = w(x,0)$ $\Rightarrow$ $u \equiv w$ and $u,w\in C^1([0,T]\times \R^d)$.

The existence of classical characteristics will play a fundamental role in the neural-network based approach to be developed in the next section. An interesting question for future work is whether classical characteristics could be replaced by a suitable notion of generalized characteristics, as studied in \cite{albano2002propagation,albano2020generation}. The main challenge in such an extension is the non-uniqueness of generalized characteristics; it is not immediately clear whether a neural network-based approach can coherently approximate such generalized characteristics when there is a priori no flow map $(q_t,p_t) = \phi_t(q_0,p_0)$ in the classical sense. In contrast, if such a flow map exists, then this represents sufficient structure to beat the curse of parametric complexity, as shown in the following section.
}

\section{Hamilton-Jacobi Neural Operator: HJ-Net}
\label{sec:NO}

In this section, we will describe an operator learning framework to approximate the solution operator $\cS^\dagger_t$ defined by \eqref{eq:HJ} for initial data $u_0 \in C^r_\per(\Omega)$ with $r \ge 2$. The main motivation for our choice of framework is the observation that the flow map $\Psi_t^\dagger$ 
associated with the system of ODEs (\ref{eq:flow}a) can be computed \emph{independently of the underlying solution $u$}. Hence, an operator learning framework for \eqref{eq:HJ} can be constructed based on a suitable approximation $\Psi_t \approx \Psi_t^\dagger$, where $\Psi_t: \Omega \times \R^d \times \R \to \Omega \times \R^d \times \R$ is a neural network approximation of the flow $\Psi_t^\dagger$. Given such $\Psi_t$, and upon fixing evaluation points $\{q_0^{j}\}_{j=1}^N \subset \Omega$, we propose to approximate the forward operator $\cS_t^\dagger$ of the Hamilton-Jacobi equation \eqref{eq:HJ}  by $\cS_t$,
using the following three steps:
\begin{enumerate}[label=\qquad Step \alph*)]
\item encode the initial data $u_0 \in C^r_\per(\Omega)$ by evaluating it at the points $q_0^{\j}$:
\[
\begin{aligned}
\cE: C^r_\per(\Omega)  &\to [\Omega\times \R^d \times \R]^N, 
\\
u_0  &\mapsto \{ (q_0^{\j},p_0^{\j},z_0^{\j}) \}_{j=1}^N,
\end{aligned}
\]
with $(q_0^{\j},p_0^{\j},z_0^{\j}) := (q_0^{\j},\nabla_q u_0(q_0^{\j}), u_0(q_0^{\j}))$;
\item for each $j=1,\dots, N$, apply the approximate flow $\Psi_t: \Omega \times  \R^d \times \R$ to the encoded data, resulting in a map
\[
\begin{aligned}
\Psi^N_t: [\Omega \times \R^d \times \R]^N &\to [\Omega\times \R^d \times \R]^N,
\\
\{ (q_0^{\j},p_0^{\j},z_0^{\j}) \}_{j=1}^N &\mapsto \{(q_t^{\j},p_t^{\j},z_t^{\j})\}_{j=1}^N,
\end{aligned}
\]
where $(q_t^{\j},p_t^{\j},z_t^{\j}) := \Psi_t(q_0^{\j},p_0^{\j},z_0^{\j})$, for $j=1,\dots, N$;
\item approximate the underlying solution at a fixed time $t \in [0,T]$, for $T$ sufficiently
small, by interpolating the data (input/output pairs) $\{(q_t^{\j}, z_t^{\j})\}_{j=1}^N$, leading 
to a reconstruction map:
\[
\begin{aligned}
\cR: [\Omega\times \R]^N &\to C^r(\Omega), \\
     \{(q_t^{\j}, z_t^{\j})\}&\mapsto s_{z,Q}.
     \end{aligned} 
\]
\end{enumerate}

If we let $\cP$ denote the projection map which takes $[\Omega\times \R^d \times \R]^N$ to
$[\Omega\times \R]^N$ then, for fixed $T$ sufficiently small, we have defined an approximation of $\cS_t^\dagger: C^r_\per(\Omega) \to C^r_\per(\Omega)$, denoted $\cS_t:C^r_\per(\Omega) \to C^r(\Omega)$, and defined by
\begin{equation}
\label{eq:Sapprox}
\cS_t=\cR \circ \cP \circ \Psi^N_t \circ \cE.   
\end{equation}
It is a consequence of Proposition \ref{prop:short-time} that our
approximation $\cS_t$ is well-defined for all inputs $u_0$ from compact subset $\cF \subset C^r_\per(\Omega)$, $r \ge 2$, in some interval $t \in [0,T]$, for
$T$ sufficiently small. However the resulting approximation does not obey the
semigroup property with respect to $t$ and should
be interpreted as holding for a fixed $t \in [0,T]$, $T$ sufficiently small.
Furthermore, iterating the map obtained for any such fixed $t$ is not in general
possible unless $\cS_t$ maps $\cF$ into itself. The requirement
that $\cS_t^\dagger$ maps $\cF$ into itself would also be required to
prove the existence of a semigroup for \eqref{eq:HJ}; for our
operator approximator $\cS_t$ we would additionally need to ensure periodicity
of the reconstruction step, something we do not address in this paper.

If the underlying solution $u(q,t)$ of \eqref{eq:HJ} exists up to time $t$ and if it is $C^2$, then the method of characteristics can be applied, and the above procedure would reproduce the underlying solution, in the absence of approximation errors in steps b) and c); i.e. in the absence of approximation errors of the Hamiltonian flow $\Psi_t \approx \Psi_t^\dagger$, and in the absence of reconstruction errors. We will study the effect of
approximating step b) by use of a ReLU neural network approximation of
the flow $\Psi_t^\dagger$ and by use of
a moving least squares interpolation for step c). 
In the following two subsections we define these two approximation steps, noting that doing so leads to a complete specification of $\cS_t.$
This complete specification is summarized in the final Subsection \ref{ssec:summary}.

\subsection{Step b) ReLU Network}
\label{ssec:RELU}

We seek an approximation $\Psi_t$ to $\Psi_t^\dagger$ in the form of a ReLU neural network \eqref{eq:nn}, as summarized in section \ref{sec:nndef}, with input and output dimensions $\DX = \DY = 2d+1$, and taking the concatenated input $x := (q_0,p_0,z_0) \in \Omega\times \R^d \times \R$ to its image in $\R^d \times \R^d \times \R$. We use $2\pi$-periodicity to identify the output in the first $d$ components (the spatial variable $q$) with a unique element in $\Omega = [0,2\pi]^d$. With slight abuse of notation, this results in a well-defined mapping 
\[
\Psi_t: \Omega \times \R^d \times \R \to \Omega \times \R^d \times \R, 
\quad 
(q_0,p_0,z_0) \to \Psi_t(q_0,p_0,z_0).
\]

\if{0}{
Fix integer $K$ and integers $\{d_k\}_{k=0}^{L-1}.$ Let $A_k \in \R^{d_{k} \times d_{k+1}}$ and
$b_k \in \R^{d_k}.$ We seek an approximation $\Psi_t$ to $\Psi_t^\dagger$ with the form
\begin{align*}
   w_0&=(q_0,p_0,z_0),\\
   w_{k+1}&=\sigma(A_k w_k+b_k), \quad k=0, \cdots, K-1,\\
   \Psi_t(q_0,p_0,z_0) & =A_K w_K+b_K.
\end{align*}
Here the activation function $\sigma:\R \to \R$ is extended pointwise to
act on any Euclidean space; and in what follows we employ the ReLU activation function
$\sigma(x)=\min\{0,x\}.$ We let $\theta:=\{A_k,b_k\}_{k=0}^{K}$ and note that we
have defined parametric approximation $\Psi_t(\slot)=\Psi_t(\slot;\theta)$
of $\Psi_t^\dagger(\slot)$. We define the depth of $\Psi_t$ as the number of layers, and the size of $\Psi_t$ as the number of non-zero weights and biases, i.e.
\[
\depth(\Psi_t) = K,
\quad
\size(\Psi_t) = \sum_{k=0}^K \left\{ \Vert A_k \Vert_0 + \Vert b_k \Vert_0 \right\},
\]
where $\Vert \slot \Vert_0$ counts the number of non-zero entries of a matrix or vector.
}\fi

Let $\mu$ denote a probability measure on $\Omega \times \R^d \times \R$. 
We would like to choose $\theta$ to minimize the loss function
$$L(\theta)=\E^{(q,p,z) \sim \mu}\|\Psi_t(q,p,z;\theta)-\Psi_t^\dagger(q,p,z)\|.$$
In practice we achieve this by an empirical approximation of $\mu$, based on $N$
i.i.d. samples, and numerical simulation to approximate the
evaluation of $\Psi_t^\dagger(\slot)$ on these samples. The resulting
approximate empirical loss can be approximately minimized by stochastic gradient
descent \cite{goodfellow2016deep,robbins1951stochastic}, for example.

\subsection{Step c) Moving Least Squares}
\label{ssec:MLS}

In this section, we define the interpolation mapping
$\cR: \{q^\j_t,z^\j_t\}_{j=1}^N \mapsto u(\slot,t)$, 
employing reconstruction by moving least squares \cite{wendland2004scattered}.
In general, given a function $f^\dagger: \Omega \to \R$, here assumed to be defined on the domain $\Omega = [0,2\pi]^d$, and given a set of (scattered) data $\{\mfq^\j,z^\j\}_{j=1}^N$ where $z^\j = f^\dagger(\mfq^\j)$ for $\mfQ = \{\mfq^1,\dots, \mfq^N\}\subset \Omega$, the method of moving least squares produces an approximation $f_{z,\mfQ} \approx f^\dagger$, which is given by the following minimization \cite[Def. 4.1]{wendland2004scattered}:
\begin{align}
\label{eq:mls}
f_{z,\mfQ,n}(q) 
=
\mathrm{min} 
\set{
\sum_{j=1}^N \left[ z^j - P(\mfq^j) \right]^2 \phi_\delta(q-\mfq^j)
}{
P \in \pi_n
}.
\end{align}
Here, $\pi_n$ denotes the space of polynomials of degree $n$, and $\phi_\delta(q) = \phi(q/\delta)$ is a compactly supported, non-negative weight function. \rev{In the following, we will choose $n:=r-1$, where $r$ is the degree of smoothness of $f^\dagger \in C^r(\Omega)$.} We will assume $\phi(\slot)$ to be smooth, supported in the unit ball $B_1(0)$, and positive on the ball $B_{1/2}(0)$ with radius $1/2$. 


The approximation error incurred by moving least squares can be estimated in the $L^\infty$-norm, in terms of suitable measures of the ``density'' of the scattered data points $\mfq^j$, and the smoothness of the underlying function $f^\dagger$ \cite{wendland2004scattered}. The relevant notions are defined next, in which $|\slot|$ denotes the Euclidean distance on $\Omega \subset \R^d$.

\begin{definition}
\label{def:scatter}
The \define{fill distance} of a set of points $\mfQ= \{\mfq^1,\dots, \mfq^N\} \subset \Omega$ for a bounded domain $\Omega\subset \R^d$ is defined to be
\[
h_{\mfQ,\Omega} := \sup_{q\in \Omega} \min_{j=1,\dots, N} |q - \mfq^\j|.
\]
The \define{separation distance} of $\mfQ$ is defined by
\[
\rho_\mfQ := \frac12 \min_{k\ne j} |\mfq^\k - \mfq^\j|.
\]
A set $\mfQ$ of points in $\Omega$ is said to be \define{quasi-uniform} with 
respect to $\kappa \ge 1$, if
\[
\rho_\mfQ \le h_{\mfQ,\Omega} \le \kappa \rho_\mfQ.
\]
\end{definition}

Combining the approximation error of moving least squares with a stability result
for moving least squares, with respect to the input data, leads to the error estimate
that we require to estimate the error in our proposed HJ-Net.
Proof of the following may be found in Appendix \ref{A:B}.
The statement involves both the input data set $Q$ and a set $Q^\dagger$ which $Q$
is supposed to approximate, together with their respective fill-distances and
separation distances.

\begin{proposition}[{\bf Error Stability}]
\label{p:es}
Let $\Omega = [0,2\pi]^d \subset \R^d$ and consider
function $f^\dagger\in C^{r}(\Omega)$, for some fixed  regularity parameter $r\ge 2$. 
Assume that $Q^\dagger = \{q^{j,\dagger}\}_{j=1}^N$ is quasi-uniform with respect to $\kappa \ge 1$. Let $\{ q^j, z^j \}_{j=1}^N$ be approximate interpolation data, and define $Q := \{q^j\}_{j=1}^N$ where, for some $\rho \in (0, \frac12 \rho_{Q^\dagger})$ and $\epsilon > 0$, we have
\[
|q^j - q^{j,\dagger}| < \rho, \quad |z^j - f(q^{j,\dagger})| < \epsilon.
\] 
\rev{Using this approximate interpolation data, let $f_{z,Q} := f_{z,Q,n}$ be obtained by moving least squares \eqref{eq:mls} with $n:=r-1$.} Then there exist constants $h_0, \gamma, C > 0$, 
depending only on $d$, $r$ and $\kappa\ge 1$, such that, for $h_{Q,\Omega} \le h_0$
and moving least squares scale parameter $\delta := \gamma h_{Q,\Omega}$, we have
\begin{align}
\label{eq:mls-err}
\Vert f^\dagger - f_{z,Q} \Vert_{L^\infty}
\le
C \left( 
\Vert f^\dagger \Vert_{C^{r}} h^{r}_{Q^\dagger,\Omega} 
+ 
\epsilon + \Vert f^\dagger \Vert_{C^1} \rho
\right).
\end{align}
\end{proposition}

\begin{remark}
The constants $C$ and $\gamma$ in the previous proposition can be computed explicitly \cite[see Thm. 3.14 and Cor. 4.8]{wendland2004scattered}.
\end{remark}

Proposition \ref{p:es} reveals two sources of error in the reconstruction by moving least squares: The first term on the right-hand side of \eqref{eq:mls-err} is the error due to the discretization by a finite number of evaluation points $q^\j$. This error persists even in a perfect data setting, i.e. when ${q}^{j} = q^{j,\dagger}$ and ${z}^{j} = f^\dagger(q^{j,\dagger})$. The last two terms in \eqref{eq:mls-err} reflect the fact that in our intended application to HJ-Net, the evaluation points $q^{j}$ and the function values $z^j$ are only given approximately, via the approximate flow map $\Psi_t \approx \Psi_t^\dagger$, introducing additional error sources.

The proof of Proposition \ref{p:es} relies crucially on the fact that the set of evaluation points $Q = \{q^\j\}_{j=1}^N$ is quasi-uniform. In our application to HJ-Net, these points are obtained as the image of $(q_0^\j,p_0^\j,z_0^\j) := (q_0^{\j},\nabla_q u_0(q_0^{\j}), u_0(q_0^{\j}))$ under the approximate flow $\Psi_t$. In particular, they depend on $u_0$ and we cannot ensure any a priori control on the separation distance $\rho_Q$. Our proposed reconstruction $\cR$ 
therefore involves the pruning step, stated as Algorithm \ref{alg:pruning}.
Lemma \ref{lem:stab} in Appendix \ref{A:B} shows that 
Algorithm \ref{alg:pruning} produces a quasi-uniform 
set $Q'\subset Q$ with the desired properties asserted above.

\begin{algorithm}[H]
\caption{Pruning}\label{alg:pruning}
\begin{algorithmic}
\Require Interpolation points $Q = \{q^\j\}_{j=1}^N$. 
\Ensure Pruned interpolation points $Q' = \{q^\jk\}_{k=1}^m$ with fill distance $h_{Q',\Omega} \le 3 h_{Q,\Omega}$, and $Q'$ is quasi-uniform for $\kappa = 3$, i.e.
\[
\rho_{Q'} \le h_{Q',\Omega} \le 3 \rho_{Q'}.
\]
\Procedure{}{}
\State Set $m \leftarrow 1$, $j_1 \leftarrow 1$ and $Q' \leftarrow \{q^1\}$. \;
\While{$m < N$}
\State Given $Q' = \{q^{j_1},\dots, q^{j_m}\}\subset Q$, does there exist $q^k \in Q$, such that
\[
B_h(q^k) \cap \bigcup_{\ell=1}^m B_h(q^{j_\ell}) = \emptyset?
\]
\If{Yes}
    \State Set $j_{m+1} \leftarrow k$, \;
    \State Set $Q' \leftarrow Q' \cup \{q^{k}\}$, \;
    \State Set $m \leftarrow m+1$. \;
\Else
    \State Terminate the algorithm.\;
\EndIf
\EndWhile
\EndProcedure
\end{algorithmic}
\end{algorithm}

Given the interpolation by moving least squares and the pruning algorithm, we finally define the reconstruction map $\cR: \{(q_t^\j,z_t^\j)\}_{j=1}^N \mapsto {f}_{z,Q}$. \rev{In our application, the interpolation data is connected with a ground truth map $f^\dagger \in C^r(\Omega)$. The following algorithm assumes knowledge of the degree of smoothness $r$:}
\begin{algorithm}[H]
\caption{Reconstruction $\cR$}\label{alg:reconstruction}
\begin{enumerate}
\item Given approximate interpolation data $\{q^j_t,z^j_t\}_{j=1}^N$ at data points $Q = \{q^1_t,\dots, q^N_t\}$, determine a quasi-uniform subset $Q' = \{q_t^\jk\}_{k=1}^m \subset Q$ by Algorithm \ref{alg:pruning}.
\item Given a degree of smoothness $r\in \N_{\ge 2}$, define $f_{z,Q} := f_{z,Q,n=r-1}$ as the moving least squares interpolant \eqref{eq:mls}, based on the (pruned) interpolation data $\{(q^\jk_t, z^\jk_t)\}_{k=1}^m$ with $\delta = \gamma h_{Q',\Omega}$, and $\gamma$ the constant in Proposition \ref{p:es}.
\end{enumerate}
\end{algorithm}

\subsection{Summary of HJ-Net}
\label{ssec:summary}
Thus in the final definition of HJ-Net given in 
equation \eqref{eq:Sapprox}
we recall that $\cE$ denotes the encoder mapping, 
\[
u_0 \mapsto \cE(u_0) := (q^j_0,\nabla u_0(q^j_0),u_0(q^j_0)).
\]
The mapping 
\[
(q^\j_0,p^\j_0,z^\j_0) \mapsto (q^\j_t,p^\j_t,z^\j_t) := {\Psi}_{t}(q^\j_0,p^\j_0,z^\j_0;\theta),
\]
approximates the flow $\Psi_t$ for each of the triples, $(q^\j_0,p^\j_0,z^\j_0)$, $j=1,\dots, N$, and
$\theta$ is trained from data to minimize an approximate empirical loss.
And, finally, the reconstruction $\cR$ is obtained by applying the pruned moving least squares Algorithm \ref{alg:reconstruction} to the data 
$\{q^\j_t,p^\j_t,z^\j_t\}_{j=1}^N$ at scattered data points 
$Q_t = \{q^1_t,\dots, q^N_t\}$ and with corresponding values 
$z_t = (z^1_t,\dots, z^N_t)$, to approximate the output 
$u(q,t) \approx f_{z_t,Q_t}(q)$.

\section{Error Estimates and Complexity}
\label{sec:E}

Subsection  \ref{ssec:main} contains statement of the main theorem
concerning the computational complexity of HJ-Net, and the high level structure
of the proof. The theorem demonstrates that the curse of parametric complexity
is overcome in this problem in the sense that the cost to achieve a given error,
as measured by the size of the parameterization, grows only algebraically with
the inverse of the error. 
In the following subsections \ref{ssec:hame} and \ref{ssec:mlse}, we provide a detailed discussion of the approximation of the Hamiltonian flow and the moving last squares algorithm which, together, lead to  proof of Theorem \ref{thm:HJNET}.
Proofs of the results stated in those two subsections are given in an appendix.

\subsection{HJ-Net Beats The Curse of Parametric Complexity}
\label{ssec:main}

\begin{theorem}[{\bf HJ-Net Approximation Estimate}]
\label{thm:HJNET}
Consider  equation \eqref{eq:HJ} on periodic domain $\Omega = [0,2\pi]^d$, with $C^r_\per$ initial data and Hamiltonian $H\in C^{r+1}_\per$, where $r \ge 2$. Assume that $H$ satisfies the no-blowup Assumption \ref{ass:blowup}. Let $\cF\subset C^r_\per$ be a compact set of initial data,  and let $T<T^*(\cF)$ where $T^*(\cF)$ is given by Proposition \ref{prop:short-time}. Then there is constant $C=C(T,d,r,H,\cF)>0$, such that for any $\epsilon > 0$ and $t\in [0,T]$, there exists a set of points $Q_\epsilon = \{q^1,\dots, q^N\}$, optimal parameter $\theta_\epsilon$ and neural network $\Psi_t(\slot)=\Psi_t(\slot;\theta_\epsilon)$ 
such that the corresponding HJ-Net approximation given by \eqref{eq:Sapprox}, 
with Steps b) and c) defined in Subsections \ref{ssec:RELU} and \ref{ssec:MLS} 
respectively, satisfies
\[
\sup_{u_0\in \cF} \Vert \cS_t(u_0) - \cS_t^\dagger(u_0) \Vert_{L^\infty} \le \epsilon.
\]
Furthermore, we have the following complexity bounds:
(i) the number $N$ of encoding points $Q_\epsilon = \{q^j\}_{j=1}^N$ from Step a)
can be bounded by 
\begin{align}\label{eq:Qbound}
N \le C\epsilon^{-d/r};
\end{align}
and (ii) the neural network
$\Psi_t(\slot)=\Psi_t(\slot;\theta_\epsilon)$ 
from Step b), Subsection \ref{ssec:RELU}, satisfies
\begin{align} \label{eq:Phibound}
\depth(\Psi_{t})
\le C \log(\epsilon^{-1}),
\quad
\size(\Psi_{t}) \le C \epsilon^{-(2d+1)/r} \log(\epsilon^{-1}).
\end{align}
\end{theorem}

\begin{proof}
We first note that, for any $u_0\in \cF$ and $T<T^*(\cF)$, Proposition
\ref{prop:short-time} shows that the solution $u$ of \eqref{eq:HJ} can be computed by the method of characteristics up to time $T$. Thus $\cS^\dagger_t$ is well-defined for
any $t \in [0,T]$.
We break the proof into three steps, relying on propositions established
in the following subsections, and then conclude in a final fourth step.

\textbf{Step 1: (Neural Network Approximation)} Let $M$ be given by
\[
M := 
1\vee \sup_{u_0 \in \cF} \sup_{q\in \Omega} \Vert \nabla u_0(q) \Vert_{\ell^\infty} \vee \sup_{u_0 \in \cF} \sup_{q\in \Omega} |u_0(q)|,
\] 
where $a\vee b$ denotes the maximum. By this choice of $M$, we have $\nabla u_0(q)\in [-M,M]^d$, $u_0(q) \in [-M,M]$ for all $u_0\in \cF$, $q\in \Omega$. By Proposition \ref{prop:ahf}, there exists a constant $\beta = \beta(d,L_H,t) \ge 1$, and a constant $C>0$, depending only on $M$, $d$, $t$, and the norm $\Vert H \Vert_{C^{r+1}(\T^d \times [-\beta M,\beta M]^d)}$, such that the Hamiltonian flow map $\Psi_t^\dagger$ \eqref{eq:flowmap} can be approximated by a neural network $\Psi_t$ with 
\[
\size(\Psi_t) \le C \epsilon^{-(2d+1)/r} \log(\epsilon^{-1}),
 \quad \depth(\Psi_t) \le C \log(\epsilon^{-1}),
\]
and
\begin{align}
\label{eq:NNapprox}
\sup_{(q_0,p_0,z_0) \in \Omega \times [-M,M]^d \times [-M,M]}
\left|\Psi_t(q_0,p_0,z_0) - \Psi^\dagger_t(q_0,p_0,z_0) \right|
\le \epsilon.
\end{align}

\textbf{Step 2: (Choice of Encoding Points)}
Fix $\rho > 0$, to be determined below. Let $Q := \rho \Z^d \cap [0,2\pi]^d$ denote an equidistant grid on $[0,2\pi]^d$ with grid spacing $\rho$. Enumerating the elements of $Q$, we write $Q = \{q_0^{1},\dots, q_0^{N}\}$, where we note that there exists a constant $C_1 = C_1(d)>0$ depending only on $d$, such that $N \le C_1 \rho^{-d}$; equivalently,
\begin{align}
\label{eq:rhoest}
\frac{\rho^d}{C_1} \le \frac{1}{N}.
\end{align}
For any $u_0 \in \cF$, let
\begin{align}
\label{eq:impt}
Q_{u_0}^\dagger := \set{
q_t^{\j,\dagger}
}{
q_t^{\j,\dagger} = q_t(q_0^{\j},p_0^{\j}), \, q_0^{\j} \in Q_\epsilon, \, p_0^{\j} = \nabla_q u_0(q_0^{\j}) 
},
\end{align}
be the set of image points under the characteristic mapping defined by $u_0$. Since $Q_{u_0}^\dagger = \{q^{j,\dagger}_t\}_{j=1}^N$ is a set of $N$ points, it follows from the definition of the fill distance that $N$ balls of radius $h_{Q^\dagger_{u_0},\Omega}$ cover $\Omega=[0,2\pi]^d$. A simple volume counting argument then implies that there exists a constant $C_0 = C_0(d) > 0$, such that $1/N \le C_0 \, h_{Q^\dagger_{u_0},\Omega}^d$; equivalently,
\begin{align}
\label{eq:Nhest}
\frac{1}{C_0 N} \le h_{Q^\dagger_{u_0},\Omega}^d, \quad \forall \, u_0 \in \cF.
\end{align}
Given $\epsilon$ from Step 1, we now choose $\rho := (C_0 C_1)^{1/d}\epsilon^{1/r}$, so that by \eqref{eq:rhoest} and \eqref{eq:Nhest},
\[
\epsilon^{d/r} = \frac{1}{C_0} \frac{\rho^d}{C_1} \le \frac{1}{C_0 N} \le h_{Q^\dagger_{u_0},\Omega}^d,
\quad \forall \, u_0 \in \cF.
\]
We emphasize that $C_0,C_1$ depend only on $d$, and are independent of $u_0 \in \cF$ and $\epsilon$.
We have thus shown that if $Q_\epsilon := \{q^1,\dots, q^N\}$ is an enumeration of $\rho \Z^d \cap [0,2\pi]^d$ with $\rho := (C_0 C_1)^{1/d}\epsilon^{1/r}$, then the fill distance of the image points $Q^\dagger_{u_0}$ under the characteristic mapping satisfies
\begin{align}
\label{eq:epsest}
\epsilon \le h_{Q^\dagger_{u_0},\Omega}^r, \quad \forall \, u_0 \in \cF.
\end{align}
In particular, this step defines our encoding points $Q_\epsilon$.

\textbf{Step 3: (Interpolation Error Estimate)}
Let $Q_\epsilon$ be the set of evaluation points determined in Step 2. Since $Q_\epsilon$ is an equidistant grid with grid spacing proportional to $\epsilon^{1/r}$, the fill distance of $Q_\epsilon$ is bounded by $h_{Q_{\epsilon},\Omega} \le C \epsilon^{1/r}$, where the constant $C = C(d) \ge 1$ depends only on $d$. By Proposition \ref{prop:re}, there exists a (possibly larger) constant $C = C(d,t,H,\cF) \ge 1$, such that for all $u_0\in \cF$, the set of image points $Q_{u_0}^\dagger$ under the characteristic mapping \eqref{eq:impt}, has uniformly bounded fill distance 
\begin{align}
\label{eq:Qdag}
h_{Q^\dagger_{u_0},\Omega} \le C \epsilon^{1/r}, \quad \forall \, u_0 \in \cF.
\end{align}
Furthermore, taking into account \eqref{eq:epsest}, the upper bound \eqref{eq:NNapprox} implies that 
\[
\sup_{(q_0,p_0,z_0) \in \T^d \times [-M,M]^d \times [-M,M]}
\left|\Psi_t(q_0,p_0,z_0) - \Psi^\dagger_t(q_0,p_0,z_0) \right|
\le  h_{Q^\dagger_{u_0},\Omega}^r,
\]
for any $u_0\in \cF$. In turn, this shows  that the approximate interpolation data $({q}_t^{\j}, {z}_t^{\j}) = \cP \circ \Psi_t(q_0^{\j},p_0^{\j},z_0^{\j})$, $j=1,\dots, N$, obtained from the neural network approximation $\Psi_t\approx \Psi_t^\dagger$ by (orthogonal) projection $\cP$ onto the first and last 
components, satisfies 
\[
|{q}_t^{\j} - q_t^{j,\dagger} | \le h_{Q^\dagger_{u_0},\Omega}^r, 
\quad
|{z}_t^{\j} - u(q_t^{j,\dagger},t) | \le h_{Q^\dagger_{u_0},\Omega}^r,
\]
where $u(q,t)$ denotes the exact solution of the Hamilton-Jacobi PDE \eqref{eq:HJ} with initial data $u_0\in \cF$. By Proposition \ref{prop:recerr}, there exists a constant $C > 0$, depending only on $d$ and $r$, such that the pruned moving least squares approximant $f_{{z},Q_{u_0}}$ obtained by Algorithm \ref{alg:reconstruction} with (approximate) data points $Q_{u_0} = \{q^1_t,\dots, q^N_t\}$ and corresponding data values $z = \{z^1_t,\dots, z^N_t\}$, satisfies
\begin{align} 
\label{eq:1uh}
\Vert u(\slot,t) - f_{{z},Q_{u_0}} \Vert_{L^\infty(\Omega)}
\le
C \left(1 + \Vert u(\slot,t) \Vert_{C^r(\Omega)} \right) h_{Q^\dagger_{u_0},\Omega}^{r}.
\end{align}

\textbf{Step 4: (Conclusion)}
By the short-time existence result of Proposition \ref{prop:short-time}, there exists $C = C(H,\cF,t) > 0$, such that $\Vert u(\slot,t) \Vert_{C^r(\Omega)}\le C$ for any solution $u$ of the Hamilton-Jacobi equation \eqref{eq:HJ} with initial data $u(\slot,t=0) = u_0 \in \cF$. By definition of the HJ-Net approximation, we have $\cS_t(u_0) \equiv f_{{z},Q_{u_0}}$ and by definition of the solution operator, $\cS_t^\dagger(u_0) \equiv u(\slot, t)$. We thus conclude that 
\[
\Vert \cS_t(u_0) - \cS^\dagger_t(u_0) \Vert_{L^\infty}
\explain{\le}{\eqref{eq:1uh}}
C h_{Q_{u_0}^\dagger,\Omega}^r
\explain{\le}{\eqref{eq:Qdag}}
C \epsilon,
\]
for a constant $C = C(T,d,r,H,\cF)>0$, independent of $\epsilon$. Since 
$\epsilon$ is arbitrary, replacing $\epsilon$ by 
$\epsilon / C$ throughout the above argument implies the claimed error 
and complexity estimate of Theorem \ref{thm:HJNET}.
\end{proof}

\revv{
\begin{remark}[{\bf Non-uniform sampling}]
In practice, an input-function dependent, \emph{non-uniform} sampling could improve the results of the moving least squares interpolant. This is especially relevant given the potentially rapid growth of second and higher-order derivatives along characteristics. However, this would likely not lead to faster asymptotic convergence rates, motivating us to restrict attention to the uniformly sampled setting for our analysis, for simplicity.
\end{remark}
}

\begin{remark}[{\bf Overall Computational Complexity of HJ-Net}]
The error and complexity estimate of Theorem \ref{thm:HJNET} implies that for moderate dimensions $d$, and for sufficiently smooth input functions $u_0 \in \cF \subset C^{r}_\per$, with $r > 3d+1$, the overall complexity of this approach scales at most linearly in $\epsilon^{-1}$: Indeed, the mapping of the data points $(q^j_0,p^j_0,z^j_0) \mapsto (q^j_t,p^j_t,z^j_t)$ for $j=1,\dots, N$ requires $N$ forward-passes of the neural network $\Psi_{t}$, which has $O(\epsilon^{-(2d+1)/r}\log(\epsilon^{-1}))$ internal degrees of freedom. Since $N = O(\epsilon^{-d/r})$ the computational complexity of this mapping is thus bounded by $O(\epsilon^{-(3d+1)/r} \log(\epsilon^{-1})) = O(\epsilon^{-1})$. A naive implementation of the pruning algorithm requires at most $O(N^2) = O(\epsilon^{-2d/r}) = O(\epsilon^{-1})$ comparisons. The computational complexity of the reconstruction by the moving least squares method with $m \le N$ (pruned) interpolation points and with $M \sim \epsilon^{-d/r}$ evaluation points can be bounded by $O(m+M) = O(\epsilon^{-d/r} + M) = O(\epsilon^{-1})$ \cite[last paragraph of Chapter 4.2]{wendland2004scattered}. In particular, the overall complexity to obtain an $\epsilon$-approximation of the output function $\cS_t(u_0) \approx \cS_t^\dagger(u_0)$ at e.g. the encoding points $Q_\epsilon$ (with $M = |Q_\epsilon| \sim \epsilon^{-d/r}$) is at most linear in $\epsilon^{-1}$.
\end{remark}

Theorem \ref{thm:HJNET} shows that for fixed $d$ and $r$, the proposed HJ-Net architecture can overcome the general curse of parametric complexity in the operator approximation $\cS \approx \cS^\dagger$ implied by Theorem \ref{thm:codexp} even though the underlying operator does not have higher than $C^r$-regularity. This is possible because HJ-Net leverages additional structure inherent to the Hamilton-Jacobi PDE \ref{eq:HJ} (reflected in the method of characteristics), and therefore does not rely solely on the $C^r$-smoothness 
of the underlying operator $\cS^\dagger$.

\subsection{Approximation of Hamiltonian Flow and Quadrature Map}
\label{ssec:hame}

In this subsection we quantify the complexity of $\epsilon$-approximation 
by a ReLU network as defined in Subsection \ref{ssec:RELU}.

Recall that $\Psi_t^\dagger: \Omega \times \R^d \times \R \to \Omega \times \R^d \times \R$ comprises solution of the Hamiltonian
equations \eqref{eq:traj} and quadrature \eqref{eq:quad}, leading to \eqref{eq:flow}. An approximation $\Psi_t: \Omega \times \R^d \times \R \to \Omega \times \R^d \times \R$ of this Hamiltonian flow map is provided the by the following proposition, 
proved in Appendix \ref{A:C}.

\begin{proposition}
\label{prop:ahf}
Let $\Omega = [0,2\pi]^d$. Let $r\ge 2$, and $t>0$ be given, and assume that $H\in C^{r+1}_\per(\Omega \times \R^d)$ satisfies the no-blowup Assumption \ref{ass:blowup} with constant $L_H>0.$ Then, for any $M\ge 1$, there exist constants $\beta := (1+\sqrt{d})\exp(L_H t)$ and $C = C\left(\Vert H \Vert_{C^{r+1}(\Omega\times [-\beta M, \beta M]^d)},M,r,d,t\right) > 0$, such that for all $\epsilon \in (0,\frac12]$, there is a ReLU neural network 
$\Psi_t(\slot)=\Psi_t(\slot;\theta_\epsilon)$ satisfying 
\begin{align} \label{eq:err-NNapprox}
\sup_{(q_0,p_0,z_0) \in \Omega \times [-M,M]^{d+1}}
\left|
\Psi_t(q_0,p_0,z_0) - \Psi_{t}^\dagger(q_0,p_0,z_0)
\right|
\le \epsilon,
\end{align}
 satisfying
 \begin{align*}
\size(\Psi_t) \le C \epsilon^{-(2d+1)/r} \log\left( \epsilon^{-1} \right),
\quad
\depth(\Psi_t) \le C \log\left( \epsilon^{-1} \right).
\end{align*}
\end{proposition}

\begin{remark}
Using the results of \cite{lu2021deep,kohler2021rate,yarotsky2020phase}, one could in fact improve the size bound of Proposition \ref{prop:ahf} somewhat: neglecting logarithmic terms in $\epsilon^{-1}$, it can be shown that a neural network with $\size(\Psi_t) \lesssim \epsilon^{-(d+1/2)/r}$ suffices. However, this comes at the expense of substantially increasing the depth from a logarithmic scaling $\depth(\Psi_t) \lesssim \log(\epsilon^{-1})$, to an algebraic scaling $\depth(\Psi_t) \lesssim \epsilon^{-(d+1/2)/r}$.
\end{remark}

\subsection{Moving Least Squares Reconstruction Error}
\label{ssec:mlse}

In this subsection we discuss error estimates for the reconstruction
by moving least squares, based on imperfect input-output pairs,
as defined in Subsection \ref{ssec:MLS}. 

We recall that the reconstruction $\cR$ in the HJ-Net approximation is obtained by applying the pruned moving least squares Algorithm \ref{alg:reconstruction} to the data $\{q^j_t,z^j_t\}_{j=1}^N$, where $(q^j_t,p^j_t,z^j_t)$ are obtained as $(q^j_t,p^j_t,z^j_t) = \Psi_t(q^j_0,p^j_0,z^j_0)$ with fixed evaluation points $\{q^j_0\}_{j=1}^N \subset \Omega$, and where $p^j_0 := \nabla_q u_0(q^j_0)$, $z^j_0 := u_0(q^j_0)$ are determined in terms of a given input function $u_0$, so that $z^j_t \approx u(q^j_t,t)$ is an approximation of the corresponding solution $u(\slot,t)$ at time $t$. 

In the following, we first derive an error estimate in terms of the fill distance of $Q_t = \{q^j_t\}_{j=1}^N$, in Proposition \ref{prop:recerr}. 
Subsequently, in Proposition \ref{prop:re}, we provide a bound on the 
fill distance $h_{Q_t,\Omega}$ at time $t$
in terms of the fill distance $h_{Q,\Omega}$ at time $0$. Proof of 
both propositions can be found in Appendix \ref{A:C}.

\begin{proposition}
\label{prop:recerr}
Let $\Omega = [0,2\pi]^d \subset \R^d$ and fix a regularity parameter $r\ge 2$. There exist constants $h_0, C > 0$ such that the following holds: 
Assume that $Q^\dagger = \{q^{1,\dagger},\dots, q^{N,\dagger}\} \subset \Omega$ is a set of $N$ evaluation points with fill distance $h_{Q^\dagger,\Omega} \le h_0$. Then for any $2\pi$-periodic function $f^\dagger \in C^r_\per(\Omega)$, and approximate input-output data $\{({q}^j,{z}^j)\}_{j=1}^N$, such that
\[
|{q}^j - q^{j,\dagger}| \le h_{Q^\dagger,\Omega}^r, \quad |{z}^{j} - f^\dagger(q^{j,\dagger})|\le h_{Q^\dagger,\Omega}^r,
\]
the pruned moving least squares approximant $f_{{z},{Q}} := f_{z,Q,n=r-1}$ of Algorithm \ref{alg:reconstruction} with interpolation data points $Q = \{q^1,\dots, q^N\}$ and data values $z = \{z^1,\dots,z^N\}$, satisfies
\[
\Vert f^\dagger - f_{{z},{Q}} \Vert_{L^\infty(\Omega)} 
\le 
C \left( 1 + \Vert f^\dagger \Vert_{C^r(\Omega)}\right) h_{Q^\dagger,\Omega}^r.
\]
\end{proposition}

In contrast to Proposition \ref{p:es},  Proposition \ref{prop:recerr} includes the pruning step in the reconstruction, and does not assume quasi-uniformity of either the underlying exact point distribution $Q^\dagger$, nor the approximate point distribution $Q$. To obtain a bound on the reconstruction error, we can combine the preservation of $C^r$-regularity implied by the short-time existence Proposition \ref{prop:short-time}, with the following: 

\begin{proposition}
\label{prop:re}
Let $\Omega =[0,2\pi]^d$, let $r\ge 2$, and assume that the Hamiltonian $H\in C^{r+1}_\per(\Omega\times \R^d)$ is periodic in $q$ and satisfies Assumption \ref{ass:blowup} with constant $L_H>0$. Let $\cF \subset C^r_\per(\Omega)$ be a compact subset and fix $t < T^\ast$. Then there exists a constant $C = C(H,\cF,L_H,t) > 0$, such that for any set of evaluation points $Q = \{q^j_0\}_{j=1}^N \subset \Omega$, and for any $u_0\in \cF$, the image points 
\[
Q_{u_0}^\dagger := \set{q^j_t}{q^j_t = \Phi^\dagger_{t,u_0}(q^j_0), \; j=1,\dots, N} \subset \Omega,
\]
under the spatial characteristic mapping $\Phi^\dagger_{t,u_0}: q_0 \mapsto q_t(q_0,\nabla_q u_0(q_0))$ satisfy the following uniform bound on the fill distance:
\[
h_{Q^\dagger_{u_0},\Omega} \le C h_{Q,\Omega}.
\]
\end{proposition}

\section{Conclusions}
\label{sec:C}

The first contribution of this work is to study the curse of 
dimensionality in the context of operator learning, here interpreted in 
terms of the infinite dimensional nature of the input space. We state 
a theorem which, for the first time, establishes a general form 
of the curse of parametric complexity, a natural scaling limit of the curse of dimensionality in high-dimensional approximation, characterized by lower bounds which are
\emph{exponential} in the desired error. The theorem demonstrates
that in general it is not possible to obtain complexity estimates,
for the size of the approximating neural operator, that grow algebraically
with inverse error unless specific structure in the underlying operator 
is leveraged; in particular, we prove that this additional structure 
has to go beyond $C^r$- or Lipschitz-regularity. This considerably generalizes and strengthens earlier work on the curse of parametric complexity in \cite{lanthaler2023operator}, where a mild form of this curse had been identified for PCA-Net. As shown in the present work, our result applies to many proposed operator learning architectures including PCA-Net, the FNO, and DeepONet and recent nonlinear extensions thereof. The lower complexity bound in this work is obtained for neural operator architectures based on standard feedforward ReLU neural networks, and could likely be extended to feedforward architectures employing piecewise polynomial activations. It is of note that algebraic complexity bounds, which seemingly overcome the curse of parametric complexity of the present work, have recently been derived for the approximation of Lipschitz operators  \cite{schwab2023deep}; these results build on non-standard neural network architectures with either superexpressive activation functions, or non-standard connectivity, and therefore do not contradict our results. In fact, a recent follow-up  \cite{lanthaler2024lipschitzoperators} to the present work sheds further light on this question from an information-theoretic perspective, resulting in a statement of the curse of parametric complexity which is independent of the activation function.

The second contribution of this paper is to present an operator learning
framework for Hamilton-Jacobi equations, and to provide a complexity analysis
demonstrating that the methodology is able to tame the curse of parametric complexity
for these PDE. We present the ideas in a simple setting, and there are 
numerous avenues for future investigation. \rev{For example, as pointed out in subsection \ref{sec:shorttime}, one main limitation of the proposed approach based on characteristics is that we can only consider finite time intervals on which classical solutions of the HJ equations exist. As already pointed out there, it would be of interest to extend the methodology to viscosity solutions, after the potential formation of singularities.} It would also be of interest 
to combine our work with the work on curse of dimensionality with respect to dimension of Euclidean space, cited in Section \ref{sec:I}. Furthermore, in practice we recommend learning the Hamiltonian flow, which underlies the method of characteristics for the HJ equation, using symplectic neural networks \cite{jin2020sympnets}. However, the analysis of these neural networks is not yet developed to the extent needed for our complexity analysis in this paper, which builds on the work in \cite{yarotsky_error_2017}. Extending the analysis of symplectic neural 
networks, and using this extension to analyze generalizations of HJ-Net as defined here,
are interesting directions for future study.
Finally, we note that it is of interest to iterate the learned operator.
In order to do this, we would need to generalize the error estimates to the $C^1$
topology. This could be achieved either by interpolation between higher-order $C^{s}$ bounds of the proposed methodology for $s>1$ combined with the existing error analysis, or by using the gradient variable $p$ in the interpolation.

\vspace{0.3in}

\noindent{\bf Acknowledgments} The work of SL is supported by Postdoc.Mobility grant P500PT-206737 from the Swiss National Science Foundation.
The work of AMS is supported by a Department of Defense 
Vannevar Bush Faculty Fellowship.

\vspace{0.2in}

\bibliographystyle{abbrv}
\bibliography{references}

\appendix

\section{Proof of Curse of Parametric Complexity}
\label{A:cod}

\subsection{Preliminaries in Finite Dimensions}
\label{A:nn-cod}

Our goal in this section is to prove Proposition \ref{prop:nn-cod}, which we repeat here:

\PropositionNNCoD*

To this end, we recall and extend several results from \cite{yarotsky_error_2017} on the ReLU neural network approximation of functions $f: \R^D \to \R$. Subsequently, these results will be used as building blocks to construct functionals in the infinite-dimensional context, leading to a curse of parametric complexity in that setting, made precise in Theorem \ref{thm:codexp}. Here, we consider the space $C^{r}(\R^D)$ consisting of $r$-times continuously differentiable functions $f: \R^D \to \R$. We denote by 
\[
F_{D,r} := \set{
f \in C^{r}(\R^D)
}{
\sup_{|\alpha|\le r} \sup_{x\in \R^D}  \left| D^\alpha f(x) \right| \le 1
},
\]
the unit ball in $C^{r}(\R^D)$. For technical reasons, it will often be more convenient to consider the subset $\oF_{D,r} \subset F_{D,r}$ consisting of functions $f\in F_{D,r}$ vanishing at the origin, $f(0) = 0$. 

\begin{remark}
We note that previous work \cite{yarotsky_error_2017} considers the Sobolev space $W^{r,\infty}([0,1]^D)$  and the unit ball in $W^{r,\infty}([0,1]^D)$, rather than $C^r(\R^D)$ and our definition of $F_{D,r}$. It turns out that the complexity bounds of \cite{yarotsky_error_2017} remain true also in our setting (with essentially identical proofs). We include the necessary arguments below.
\end{remark}

Let $f: \R^D \to \R$ be a function. Following \cite[Sect. 4.3]{yarotsky_error_2017}, we will denote by $\cN(f,\epsilon)$ the minimal number of hidden computation units that is required to approximate $f$ to accuracy $\epsilon$ by an arbitrary ReLU feedforward network $\Psi$, \emph{uniformly over the unit cube $[0,1]^D$}; i.e. $\cN(f,\epsilon)$ is the minimal integer $M$ such that there exists a ReLU neural network $\Psi$ with at most $M$ computation units\footnote{The number of computation units equals the number of scalar evaluations of the activation $\sigma$ in a forward-pass, cf. \cite{anthony_bartlett_1999}.} and such that
\[
\sup_{x\in [0,1]^D} |f(x) - \Psi(x)| \le \epsilon.
\]

\begin{remark}
We emphasize that, even though the domain of a function $f \in F_{D,r}$ is by definition the whole space $\R^D$, the above quantity $\cN(f,\epsilon)$ only relates to the approximation of $f$ over the unit cube $[0,1]^D$. The reason for this seemingly awkward choice is that it will greatly simplify our arguments later on, when we construct functionals on infinite-dimensional Banach spaces with a view towards proving an infinite-dimensional analogue of the curse of dimensionality.
\end{remark}

Note that the number of (non-trivial) hidden computation units $M$ of a neural network $\Psi: \R^D \to \R$ obeys the bounds $M \le \size(\Psi) \le 5 D M^4$: The lower bound follows from the fact that each (non-trivial) computation unit has at least one non-zero connection or a bias feeding into it. To derive the upper bound, we note that any network with at most $M$ computation units can be embedded in a fully connected enveloping network (allowing skip-connections) \cite[cf. Fig. 6(a)]{yarotsky_error_2017} with depth $M$, width $M$, where each hidden node is connected to all other $M^2-1$ hidden nodes plus the output, and where each node in the input layer is connected to all $M^2$ hidden units plus the output. In addition, we take into account that each computation unit and the output layer have an additional bias. The existence of this enveloping network thus implies the size bound
\begin{align*}
\size(\Psi) 
&\le \underbrace{ M^2(M^2-1) + M^2 }_{\text{ hidden connections}} + \underbrace{D (M^2+1)}_{\text{input conn.}} + \underbrace{M+1}_{\text{biases}} 
\\
&\le M^4 + 2DM^4 + M + 1\le 5DM^4,
\end{align*}
valid for any neural network $\Psi: \R^D \to \R$ with at most $M$ computation units.

In view of the lower size bound, $\size(\Psi) \ge M$, Proposition \ref{prop:nn-cod} above is now clearly implied by the following:
\begin{proposition}
\label{prop:y0}
Fix $r\in \N$. There is a universal constant $\gamma > 0$ and a constant $\bar{\epsilon} = \bar{\epsilon}(r) > 0$, depending only on $r$, such that for any $D \in \N$, there exists a function $f_D \in \oF_{D,r}$ for which 
\[
\cN(f_D, \epsilon) \ge \epsilon^{-\gamma D/ r}, \quad \forall \, \epsilon \le \bar{\epsilon}.
\]
\end{proposition}

As an immediate corollary of Proposition \ref{prop:y0}, we conclude that $f_D$ cannot be approximated to accuracy $\epsilon$ by a family of ReLU neural networks $\Psi_\epsilon$ with 
\[
\size(\Psi_\epsilon) = o(\epsilon^{-\gamma D/r}).
\]
The latter conclusion was established by Yarotsky \cite[Thm. 5]{yarotsky_error_2017} (with $\gamma = 1/9$). Proposition \ref{prop:nn-cod} improves on Yarotsky's result in two important ways: firstly, it implies that the lower bound $\size(\Psi_\epsilon) \ge \epsilon^{-\gamma D/r}$ holds for \emph{all} $\epsilon \le \bar{\epsilon}$, and not only along a (unspecified) sequence $\epsilon_k \to 0$; secondly, it shows that the constant $\bar{\epsilon}$ can be chosen \emph{independently} of $D$. This generalization of Yarotsky's result will be crucial for the extension to the infinite-dimensional case, which is the goal of this work.

To prove Proposition \ref{prop:y0}, we will need two intermediate results, which build on \cite{yarotsky_error_2017}.
We start with the following lemma providing a lower bound on the required size of a fixed neural network architecture, which is able to approximate arbitrary $f\in \oF_{D,r}$ to accuracy $\epsilon \ge 0$. This result is based on \cite[Thm. 4(a)]{yarotsky_error_2017}, but explicitly quantifies the dependence on the dimension $D$. This dependence was left unspecified in earlier work.
\begin{lemma}
\label{lem:y1}
Fix $r\in \N$. Let $\Psi = \Psi(\slot;\theta)$ be a ReLU neural network architecture depending on parameters $\theta \in \R^W$. There exists a constant $c_0 = c_0(r)>0$, such that if
\[
\sup_{f \in \oF_{D,r}} \inf_{\theta \in \R^W} \Vert f - \Psi(\slot;\theta) \Vert_{L^\infty([0,1]^D)} \le \epsilon,
\]
for some $\epsilon > 0$, 
then $W \ge (c_0\epsilon)^{-D/2r}$.
\end{lemma}

\begin{proof}
The proof in \cite{yarotsky_error_2017} is based on the Vapnik-Chervonenkis (VC) dimension. We will not repeat the entire argument here, but instead discuss only the required changes in the proof of Yarotsky, which are needed to prove our extension of his result.
In fact, there are only two differences between the proof our Lemma \ref{lem:y1}, and \cite[Thm. 4(a)]{yarotsky_error_2017}, which we now point out: The first difference is that we consider $\oF_{D,r}$, consisting of $C^r(\R^D)$ functions vanishing at the origin, whereas \cite{yarotsky_error_2017} considers the unit ball in $W^{r,\infty}([0,1]^D)$. Nevertheless the proof of \cite{yarotsky_error_2017} applies to our setting in a straightforward way. To see this, we recall that the construction in \cite[Proof of Thm. 4, eq. (35)]{yarotsky_error_2017}, considers $f \in F_{D,r}$ of the form 
\begin{align}
\label{eq:fx}
f(x) = \sum_{m=1}^{N^D} y_m \phi(N(x-x_m)),
\end{align}
with coefficients $y_m$ to be determined later, and where $\phi: \R^D \to \R$ is a $C^\infty$ bump function, which is required to satisfy\footnote{We note that choosing the $\ell^\infty$ to define the set where $\phi(x) = 0$, rather than the $\ell^2$ Euclidean distance as in \cite{yarotsky_error_2017}, is immaterial. The only requirement is that the support of the functions on the right-hand side in the definition \eqref{eq:fx} do not overlap.}
\begin{align}
\label{eq:phi-req}
\phi(0) = 1, \quad \text{and} \quad \phi(x) = 0, \text{ if } \Vert x \Vert_{\ell^\infty} > 1/2.
\end{align}
In \eqref{eq:fx}, the points $x_1,\dots, x_{N^D} \in [0,1]^D$ are chosen such that the $\ell^\infty$-distance between any two of them is at least $1/N$. We can easily ensure that $f(0) = 0$, by choosing the points $x_m$ to be of the form $(j_1,\dots, j_D)/N$, where $j_1,\dots, j_D \in \{1,\dots, N\}$. Then, since for any multi-index $\alpha$ of order $|\alpha|\le r$, the mixed derivative
\[
\max_x |D^\alpha f(x)| \le N^r \max_m |y_m| \max_x |D^\alpha \phi(x)|,
\]
any function $f$ of the form \eqref{eq:fx} belongs to $\oF_{D,r}$, provided that
\[
\max_m |y_m| \le \tilde{c}_1 N^{-r},
\]
where 
\begin{align}
\label{eq:c1}
\tilde{c}_1 = \left(\sup_{|\alpha|\le r} \sup_{x\in [0,1]^D} |D^\alpha \phi(x)| \right)^{-1}.
\end{align}
As shown in \cite{yarotsky_error_2017}, the above construction allows one to encode arbitrary Boolean values $z_1,\dots, z_{N^D} \in \{0,1\}$ by construction suitable $f \in \oF_{D,r}$; this in turn can be used to estimate a VC-dimension related to $\Psi(\slot;\theta)$ from below, following verbatim the arguments in \cite{yarotsky_error_2017}. As no changes are required in this part of the proof, we will not repeat the details here; Instead, following the argument leading up to \cite[eq. (38)]{yarotsky_error_2017}, and under the assumptions of Lemma \ref{lem:y1}, Yarotsky's argument immediately implies the following lower bound
\[
W \ge \tilde{c}_0 \left( \frac{3\epsilon}{\tilde{c}_1} \right)^{-D/2r},
\]
where $\tilde{c}_0$ is an absolute constant.

To finish our proof of Lemma \ref{lem:y1}, it remains to determine the dependence of the constant $\tilde{c}_1$ in \eqref{eq:c1}, on the parameters $D$ and $r$. To this end, we construct a specific bump function $\phi: \R^D \to \R$. Recall that our only requirement on $\phi$ in the above argument is that \eqref{eq:phi-req} must hold. To construct suitable $\phi$, let $\phi_0: \R \to \R$, $\xi \mapsto \phi_0(\xi)$ be a smooth bump function such that $\phi_0(0) = 1$, $\Vert \phi_0 \Vert_{L^\infty} \le 1$ and $\phi_0(\xi)=0$ for $|\xi| > 1/2$. We now make the particular choice
\[
\phi(x_1,\dots, x_D) := \prod_{j=1}^D \phi_0(x_j).
\]
Let $\alpha = (\alpha_1,\dots, \alpha_D)$ be a multi-index with $|\alpha| = \sum_{j=1}^D \alpha_j \le r$. Let $\kappa := |\{ \alpha_j \ne 0\}$ denote the number of non-zero components of $\alpha$. We note that $\kappa \le r$. We thus have
\[
|D^\alpha \phi(x)|
=
\prod_{j=1}^D |D^{\alpha_j}\phi_0(x_j)|
\le
\prod_{\alpha_j \ne 0} |D^{\alpha_j}\phi_0(x_j)|
\le
\Vert \phi_0 \Vert_{C^r(\R)}^{\kappa}
\le
\Vert \phi_0 \Vert_{C^r(\R)}^{r}.
\]
In particular, we conclude that $\tilde{c}_1 = [\sup_{|\alpha|\le r} \sup_{x\in [0,1]^D} |D^\alpha \phi(x)|]^{-1}$ can be bounded from below by $\tilde{c}_1 \ge \tilde{c}_2 := \Vert \phi_0 \Vert_{C^r(\R)}^{-r}$, where $\tilde{c}_2 = \tilde{c}_2(r)$ clearly only depends on $r$, and is \emph{independent} of the ambient dimension $D$. The claimed inequality of Lemma \ref{lem:y1} now follows from
\[
W \ge \tilde{c}_0 \left( \frac{3\epsilon}{\tilde{c}_1} \right)^{-D/2r}
\ge \left( \frac{3\epsilon}{(\tilde{c}_0 \wedge 1)^{2r/D}\tilde{c}_2} \right)^{-D/2r}
\ge \left( \frac{3\epsilon}{(\tilde{c}_0 \wedge 1)^{2r}\tilde{c}_2} \right)^{-D/2r},
\]
i.e. we have $W \ge (c_0 \epsilon)^{-D/2r}$ with constant
\[
c_0 = \frac{3}{(\tilde{c}_0 \wedge 1)^{2r}\tilde{c}_2(r)},
\]
depending only on $r$.
\end{proof}

While Lemma \ref{lem:y1} applies to a \emph{fixed} architecture capable of approximating all $f\in \oF_{D,r}$, our goal is to instead construct a single $f\in \oF_{D,r}$ for which a similar lower complexity bound holds for \emph{arbitrary} architectures. The construction of such $f$ will rely on the following lemma, based on \cite[Lem. 3]{yarotsky_error_2017}:
\begin{lemma}
\label{lem:y2}
Fix $r\in \N$. For any (fixed) $\epsilon > 0$, there exists $f_\epsilon\in \oF_{D,r}$, such that 
\[
\cN(f_\epsilon, \epsilon) \ge D^{-1/4}(c_1 \epsilon)^{-D/8r},
\]
where $c_1 = c_1(r) > 0$ depends only on $r$.
\end{lemma}

\begin{proof}
As explained above, any neural network (with potentially sparse architecture), $\Psi: \R^D \to \R$, with $M$ computation units can be embedded in the fully connected architecture $\tilde{\Psi}(\slot;\theta)$, $\theta \in \R^W$, with size bound $W \le 5DM^4$. By Lemma \ref{lem:y1}, it follows that if $W \le 5DM^4 < (c_0 \epsilon)^{-D/2r}$, then there exists $f_\epsilon\in \oF_{D,r}$, such that 
\[
\inf_{\theta\in \R^W} \Vert f_\epsilon - \tilde{\Psi}(\slot;\theta) \Vert_{L^\infty} > \epsilon.
\]
Expressing the above size bound in terms of $M$, it follows that for any network $\Psi$ with 
$M < (5D)^{-1/4}(c_0 \epsilon)^{-D/8r}$
computation units, we must have $\Vert f_\epsilon - \Psi \Vert_{L^\infty} > \epsilon$. Thus, approximation of $f_\epsilon$ to within accuracy $\epsilon$ requires at least $M \ge (5D)^{-1/4}(c_0 \epsilon)^{-D/8r}$ computation units. From the definition of $\cN(f_\epsilon,\epsilon)$, we conclude that
\[
\cN(f_\epsilon, \epsilon) \ge D^{-1/4} (c_1 \epsilon)^{-D/8r},
\]
for this choice of $f_\epsilon\in \oF_{D,r}$, with constant $c_1 = 5^{2r}c_0(r)$ depending only on $r$.
\end{proof}

Lemma \ref{lem:y2} will be our main technical tool for the proof of Proposition \ref{prop:y0}. We also recall that $\cN(f,\epsilon)$ satisfies the following properties,  \cite[eq. (42)--(44)]{yarotsky_error_2017}:
\begin{subequations}
\label{eq:N}
\begin{gather}
\cN(af, |a|\epsilon) 
= \cN(f,\epsilon), \\
\cN(f\pm g, \epsilon + \Vert g \Vert_{L^\infty}) 
\le \cN(f,\epsilon), \\
\cN(f_1 \pm f_2, \epsilon_1 + \epsilon_2) 
\le \cN(f_1,\epsilon_1) + \cN(f_2,\epsilon_2).
\end{gather}
\end{subequations}

\begin{proof}{(Proposition \ref{prop:y0})}
We define a rapidly decaying sequence $\epsilon_k \to 0$, by $\epsilon_1 = 1/4$ and $\epsilon_{k+1} = \epsilon^2_k$, so that by recursion $\epsilon_k = 2^{-2^k}$. We also define $a_k := \frac12 \epsilon_{k}^{1/2} = \frac12 \epsilon_{k-1}$. For later reference, we note that since $\epsilon_k \le 1/2$ for all $k$, we have
\begin{align}
\label{eq:atail}
\sum_{s=k+1}^\infty a_s = \frac12 \left( \epsilon_k + \epsilon_k^2 + \epsilon_k^{2^2} + \dots  \right) \le \frac12 \epsilon_k \sum_{s=0}^\infty 2^{-s} =  \epsilon_k.
\end{align}
Our goal is to construct $f\in \oF_{D,r}$ of the form 
\[
f = \sum_{k=1}^\infty a_k f_k, \quad \text{with } f_k \in \oF_{D,r} \; \forall \, k\in\N.
\]
We note that $a_k \le 2^{-k}$, hence if $f$ is of the above form then,
\begin{align*}
\Vert f \Vert_{C^{r}}
&\le 
\sum_{s=1}^\infty a_s \Vert f_s \Vert_{C^{r}}
\le 1,
\quad 
f(0) = \sum_{s=1}^\infty a_s \underbrace{f_s(0)}_{=0} = 0,
\end{align*}
implies that $f\in \oF_{D,r}$ 
irrespective of the specific choice of $f_k \in \oF_{D,r}$. In the following construction, we choose $\gamma := 1/32$ throughout. We determine $f_k$ recursively, formally starting from the empty sum, i.e. $f=0$. In the recursive step, given $f_1,\dots, f_{k-1}$, we distinguish two cases:

\textbf{Case 1:} Assume that 
\[
\cN\left(\sum_{s=1}^{k-1} a_s f_s, 2\epsilon_k\right) \ge \epsilon_k^{-\gamma D/r}.
\]
In this case, we set $f_k := 0$, and trivially obtain
\begin{align}
\label{eq:fkc}
\cN\left(\sum_{s=1}^{k} a_s f_s, 2\epsilon_k\right) \ge  \epsilon_k^{-\gamma D/r}.
\end{align}

\textbf{Case 2:} In the other case, we have
\[
\cN\left(\sum_{s=1}^{k-1} a_s f_s, 2\epsilon_k\right) <   \epsilon_k^{-\gamma D/r}.
\]
Our first goal is to again choose $f_k$ such that \eqref{eq:fkc} holds, at least for sufficiently large $k$.
We note that, by the general properties of $\cN$, (\ref{eq:N}c), and the assumption of Case 2:
\begin{align}
\cN\left(\sum_{s=1}^k a_s f_s, 2\epsilon_k\right)
&\ge 
\cN(a_k f_k, 4\epsilon_k) - \cN\left(\sum_{s=1}^{k-1} a_s f_s, 2\epsilon_k\right)
\notag\\
&\ge 
\cN(a_k f_k, 4\epsilon_k) - \epsilon_k^{-\gamma D/r}.
\label{eq:Nafe}
\end{align}
By Lemma \ref{lem:y2}, we can find $f_k\in \oF_{D,r}$, such that
\begin{align}
\label{eq:Nafe2}
\cN(a_k f_k, 4\epsilon_k) 
\explain{=}{(\ref{eq:N}a)} \cN(f_k, 4\epsilon_k/a_k) 
= \cN(f_k, 4\epsilon_k^{1/2})
\ge D^{-1/4} (8 c_1 \epsilon_k^{1/2})^{-D/8r}.
\end{align}
This defines our recursive choice of $f_k$ in Case 2. By \eqref{eq:Nafe} and \eqref{eq:Nafe2}, to obtain \eqref{eq:fkc} it now suffices to show that
\[
D^{-1/4} (8 c_1 \epsilon_k^{1/2})^{-D/8r} \ge 2\epsilon_k^{-\gamma D/r}.
\]
It turns out that there exists $\bar{\epsilon}>0$ depending only on $r$, such that
\[
(8^2 c_1^2 \epsilon_k)^{-D/16r} \ge 2 D^{1/4} \epsilon_k^{-\gamma D/r} = (2^{-r/\gamma D} D^{-r/(4\gamma D)} \epsilon_k)^{-\gamma D/r},
\]
for all $\epsilon_k \le \bar{\epsilon}$, and where $\gamma = 1/32$. In fact, upon raising both sides to the exponent $-32r/D$, it is straightforward to see that this inequality holds independently of $D\in\N$, provided that
\[
 \epsilon_k \le \frac{\inf_{D} D^{-r/(4\gamma D)}}{[2^{8r}8 c_1(r)]^4},
\]
where we note that $\inf_D D^{-r/(4\gamma D)} > 0$ on account of the fact that $\lim_{D\to \infty} D^{-1/D} = 1$.
Define $\bar{\epsilon}$ by 
\[
\bar{\epsilon}(r) := \min \set{k\in \N}{\epsilon_k \le \frac{\inf_{D} D^{-r/(4\gamma D)}}{[2^{8r}8 c_1(r)]^4}}.
\]
With this choice of $\bar{\epsilon}, \gamma > 0$, and by construction of $f_k$, we then have
\begin{align} 
\label{eq:fkc1}
\cN\left(
\sum_{s=1}^k a_s f_s, 2\epsilon_k
\right) \ge \epsilon_k^{-\gamma D/r},
\end{align}
for all $\epsilon_k \le \bar{\epsilon}$.
This is inequality \eqref{eq:fkc}, and concludes our discussion of Case 2. 

Continuing the above construction by recursion, we obtain a sequence $f_1,f_2, \dots \in \oF_{D,r}$, such that \eqref{eq:fkc1} holds for any $k\in\N$ such that $\epsilon_k \le \bar{\epsilon}$. Define $f = \sum_{k=1}^\infty a_k f_k$. We claim that for any $\epsilon \le \bar{\epsilon}$ we have
\[
\cN(f,\epsilon) \ge \epsilon^{-\gamma D/2r}.
\]
To see this, we fix $\epsilon \le \bar{\epsilon}$. Choose $k\in \N$, such that $\epsilon_k \le \bar{\epsilon}$ and $\epsilon_k^2 \le \epsilon \le \epsilon_k$. Then,
\begin{align*}
\cN(f,\epsilon) 
&\ge 
\cN(f,\epsilon_{k})
= 
\cN\left(\sum_{s=1}^\infty a_s f_s ,\epsilon_k\right)
\\
&\explain{\ge}{(\ref{eq:N}b)}
\cN\left(\sum_{s=1}^k a_s f_s ,\epsilon_k + \left \Vert \sum_{s=k+1}^\infty a_s f_s \right \Vert_{L^\infty}\right)
\\
&\ge
\cN\left(\sum_{s=1}^k a_s f_s, \epsilon_k + \sum_{s=k+1}^\infty a_s\right)
\\
&\explain{\ge}{(\ref{eq:atail})} 
\cN\left(\sum_{s=1}^k a_s f_s, 2\epsilon_k \right)
\\
&\explain{\ge}{\eqref{eq:fkc1}}  \epsilon_k^{-\gamma D / r} 
\ge \epsilon^{-\gamma D / 2r},
\end{align*}
where the last inequality follows from $\epsilon_k \le \epsilon^{1/2}$. The claim of Proposition \ref{prop:y0} thus follows for all $\epsilon \le \bar{\epsilon}$, upon redefining the universal constant as $\gamma = (1/32)/2 = 1/64$.
\end{proof}

\subsection{Proof of Lemma \ref{lem:Cs}}
\label{A:Cs}

\begin{proof}{(Lemma \ref{lem:Cs})}
Since the interior of $\Omega$ is non-empty, then upon a rescaling and shift of the domain, we may wlog assume that $[0,2\pi]^d\subset \Omega$. Let us define $e_\kappa \propto \sin(\kappa \cdot x)$ as a suitable re-normalization of the standard Fourier sine-basis, indexed by $\kappa \in \N^d$, and normalized such that $\Vert e_\kappa \Vert_{C^{s}} = 1$. We note that for each $e_\kappa$ we can define a bi-orthogonal functional $e^\ast_\kappa$ by
\[
e^\ast_\kappa (u) := \frac{2}{(2\pi)^d} \int_{[0,2\pi]^d} u(x) \sin(\kappa \cdot x) \, dx.
\]
The norm of $e^\ast_\kappa$ is easily seen to be bounded by $2$. Hence, for any enumeration $j\mapsto \kappa(j)$, the sequence $e_{\kappa(j)}$ satisfies the assumptions in the definition of a infinite-dimensional hypercube.

Furthermore, if $j \mapsto \kappa(j)$ is an enumeration of $\kappa \in \N^d$, such that $j \mapsto \Vert \kappa(j) \Vert_{\ell^\infty}$ is monotonically increasing (non-decreasing), we note that any series of the form 
\[
u = A\sum_{j=1}^\infty j^{-\alpha} y_j e_{\kappa(j)}, \quad y_j \in [0,1],
\]
is absolutely convergent in $C^\rho(\Omega)$, provided that 
\[
\sum_{j=1}^\infty j^{-\alpha} \Vert e_{\kappa(j)}\Vert_{C^\rho(\Omega)} \sim \sum_{j=1}^\infty j^{-\alpha} \Vert {\kappa(j)}\Vert_{\ell^\infty}^{\rho-s} < \infty.
\]
Inverting the enumeration $j = j(\kappa)$ for $\kappa \in \N^d$, and setting $K:=\Vert \kappa\Vert_{\ell^\infty}$, we find that
\begin{align*}
\#\set{\kappa' \in \N^d}{\Vert \kappa'\Vert_{\ell^\infty} < K} 
\le  j(\kappa) 
\le 
\#\set{\kappa' \in \N^d}{\Vert \kappa'\Vert_{\ell^\infty} \le K},
\end{align*}
where the number of elements in the lower and upper bounding sets are $\sim K^d$. We thus conclude that $j(\kappa) \sim K^d$. We also note that each shell $\set{\kappa \in \N^d}{\Vert \kappa \Vert_{\ell^\infty} = K}$, with $K \in \N$, contains $\sim K^{d-1}$ elements. Thus, we have
\begin{align*}
\sum_{j=1}^\infty j^{-\alpha} \Vert {\kappa(j)}\Vert_{\ell^\infty}^{\rho-s}
&\sim 
\sum_{K=1}^\infty \sum_{\Vert \kappa \Vert_{\ell^\infty} = K} j(\kappa)^{-\alpha} \Vert {\kappa}\Vert_{\ell^\infty}^{\rho-s}
\\
&\sim
\sum_{K=1}^\infty K^{-\alpha d} \sum_{\Vert \kappa \Vert_{\ell^\infty} = K} \Vert {\kappa}\Vert_{\ell^\infty}^{\rho-s}
\\
&\sim \sum_{K=1}^\infty K^{-\alpha d} K^{d-1} K^{\rho-s}
\\
&= \sum_{K=1}^\infty K^{-1 - d \left(\alpha - 1 - \frac{\rho-s}{d} \right)}.
\end{align*}
The last sum converges if $\alpha > 1- \frac{\rho-s}{d}$. Thus, choosing $A>0$ sufficiently small then ensures that $Q_\alpha = Q_\alpha(A;e_1,e_2,\dots)$ is a subset of $K=\set{u\in C^\rho(\Omega)}{\Vert u \Vert_{C^\rho}\le M}$.
\end{proof}

\subsection{Proof of Theorem \ref{thm:codexp}}
\label{A:codexp}
We now provide the detailed proof of Theorem \ref{thm:codexp}, which builds on Proposition \ref{prop:y0} above.
\begin{proof}
Fix $\delta > 0$.
By assumption on $K\subset \cX$, there exists $A>0$ and a linearly independent sequence $e_1,e_2,\dots \in \cX$ of normed elements, such that the set $Q_\alpha$ consisting of all $u\in \cX$ of the form
\[
u = A \sum_{j=1}^\infty j^{-\alpha} y_j e_j, \quad y_j \in [0,1],
\]
defines a subset $Q_\alpha \subset K$.

\textbf{Step 1: (Finding embedded cubes $\simeq [0,1]^D$)}
We note that for any $k\in \N$ and for any choice of $y_j \in [0,1]$, with indices $j=2^k,\dots, 2^{k+1}-1$, we have
\begin{align}
\label{eq:QD}
u =
A2^{-(k+1)\alpha} \sum_{j=2^k}^{2^{k+1}-1} y_j e_j \in Q_\alpha.
\end{align}
Set $D = 2^k$, and let us denote the set of all such $u$ by $\bar{Q}_D$, in the following.
Note that up to the constant rescaling by $A2^{-(k+1)\alpha}$, $\bar{Q}_D$ can be identified with the $D$-dimensional unit cube $[0,1]^D$. In particular, since $\bar{Q}_D \subset Q_\alpha \subset K$, it follows that $K$ contains a rescaled copy of $[0,1]^D$ for any such $D$. Furthermore, it will be important in our construction that all of the embedded cubes, defined by \eqref{eq:QD} for $k\in \N$, only intersect at the origin. 

By Proposition \ref{prop:y0}, there exist constants $\gamma, \bar{\epsilon}>0$, independent of $D$, such that given any $D = 2^k$, we can find $f_D: \R^D \to \R$, $f_D\in \oF_{D,r}$, for which the following lower complexity bound  holds:
\begin{align}
\label{eq:fD}
\cN(f_D,\epsilon) \ge \epsilon^{-\gamma D/r}, \quad \forall \epsilon \le \bar{\epsilon}.
\end{align}
Our aim is to construct $\cS^\dagger:\cX \to \R$, such that the restriction $\cS^\dagger|_{\bar{Q}_D}$ to the cube $\bar{Q}_D \simeq [0,1]^D$ ``reproduces'' this $f_D$. If this can be achieved, then our intuition is that $\cS^\dagger$ embeds all $f_D$ for $D=2^k$, $k\in\N$, at once, and hence a rescaled version of the lower complexity bound $\gtrsim_D \epsilon^{-\gamma D/r}$ should hold for any $D$. Our next aim is to make this precise, and determine the implicit constant that arises due to the fact that $\bar{Q}_D$ is only a rescaled version of $[0,1]^D$.

\textbf{Step 2: (Construction of $\cS^\dagger$)}
To construct suitable $\cS^\dagger$, we first recall that we assume the existence of ``bi-orthogonal'' elements $e^\ast_1, e^\ast_2,\dots$ in the continuous dual space $\cX^\ast$, such that 
\[
e^\ast_i(e_j) = \delta_{ij}, \; \forall\, i,j\in \N,
\]
and furthermore, there exists a constant $M>0$, such that $\Vert e^\ast_j \Vert_{\cX^\ast} \le M$ for all $j\in \N$. 
Given the functions $f_D = f_{2^k}$ from Step 1, we now make the following ansatz for $\cS^\dagger$:
\begin{align}
\cS^\dagger(u) 
=
\sum_{k=1}^\infty 2^{-\alpha^\ast r k} f_{2^k}
\left(
A^{-1}2^{(k+1)\alpha} \left[ e^\ast_{2^k}(u),\dots, e^\ast_{2^{k+1}-1}(u)\right]
\right),
\end{align}
Here $f_{2^k} = f_D$ (for $D = 2^k$) satisfies \eqref{eq:fD} and $\alpha^\ast = 1+\alpha+\delta/2$. We note in passing that $\cS^\dagger$ defines a $r$-times Fr\'echet differentiable functional. This will be rigorously shown in Lemma \ref{lem:frechet} below (with $c = A^{-1}2^{\alpha}$).

\textbf{Step 3: (Relating $\cS^\dagger$ with $f_D$)}
We next show in which way ``the restriction $\cS^\dagger|_{\bar{Q}_D}$ to the cube $\bar{Q}_D \simeq [0,1]^D$ reproduces $f_D$''. Note that if $u\in \bar{Q}_{D}$ is of the form \eqref{eq:QD} with $D=2^k$ and with coefficients
\[
y := [y_{2^k},\dots, y_{2^{k+1}-1}] \in [0,1]^{2^k} = [0,1]^D,
\]
then, if $k'\ne k$,
\[
[e^\ast_{2^{k'}}(u),\dots, e^\ast_{2^{k'+1}-1}(u)] = [0,\dots, 0],
\]
 while for $k'=k$ we find
\[
[e^\ast_{2^{k}}(u),\dots, e^\ast_{2^{k+1}-1}(u)] = [y_{2^k},\dots, y_{2^{k+1}-1}] = y.
\]
From the fact that $f_{2^{k'}}(0) = 0$ for all $k'$ by construction (recall that $f_{2^{k'}}\in \oF_{2^{k'},r}$), we conclude that
\begin{align}
\label{eq:Sdag-fD}
\cS^\dagger(u) = 2^{-\alpha^\ast r k} f_{2^k}(y) = D^{-\alpha^\ast r} f_D(y),
\end{align}
for any $u\in \bar{Q}_D$. In this sense, ``$\cS^\dagger|_{\bar{Q}_D}$ reproduces $f_D$''.

\textbf{Step 4: (Lower complexity bound, uniform in $D$)}
Let $\cS_\epsilon: \cX \to \R$ be a family of operators of neural network-type, such that 
\[
\sup_{u\in K} |\cS^\dagger(u) - \cS_\epsilon(u)| \le \epsilon, \quad \forall \, \epsilon > 0.
\]
By assumption on $\cS_\epsilon$ being of neural network-type, there exists $\ell = \ell_\epsilon \in \N$, a linear mapping $\cL_\epsilon: \cX \to \R^\ell$, and $\Phi_\epsilon: \R^{\ell} \to \R$ a neural network representing $\cS_\epsilon$: 
\[
\cS_\epsilon(u) = \Phi_\epsilon(\cL_\epsilon u), \quad \forall \, u\in \cX.
\]
For $D=2^k$, $k\in \N$, define 
$\iota_D: \R^D \to \cX$ by 
\[
\iota_D(y) = A2^{-\alpha k} \sum_{j=1}^{2^k} y_j e_{2^k+j-1}.
\]
Then, since $\cL_\epsilon \circ \iota_D: \R^D \to \R^\ell$ is a linear mapping, there exists a matrix $W_D \in \R^{\ell\times D}$, such that $\cL_\epsilon \circ \iota_D(y) = W_D y$ for all $y\in \R^D$. In particular, it follows that
\[
\cS_\epsilon(\iota_D(y)) = \Phi_\epsilon(W_d y).
\]
By \eqref{eq:QD}, we clearly have $\iota_D([0,1]^D) = \bar{Q}_D$. 
Let $\hat{\Phi}_\epsilon(y) := D^{\alpha^\ast r} \Phi_\epsilon(W_D y)$.
It now follows that
\begin{align*}
\sup_{y\in [0,1]^D} \left| f_D(y) - \hat{\Phi}_\epsilon(y) \right|
&=
D^{\alpha^\ast r} \sup_{y\in [0,1]^D} \left| D^{-\alpha^\ast r} f_D(y) -  \Phi_\epsilon(W_D y) \right|
\\
&= 
D^{\alpha^\ast r} \sup_{u \in \bar{Q}_D} \left| \cS^\dagger(u) -  \cS_\epsilon(u) \right|
\\
&\le 
D^{\alpha^\ast r} \sup_{u \in K} \left| \cS^\dagger(u) -  \cS_\epsilon(u) \right|
\\
&\le D^{\alpha^\ast r}\epsilon.
\end{align*}
Let $M$ denote the number of hidden computation units of $\hat{\Phi}_\epsilon$.
By construction of $f_D$ (cp. Proposition \ref{prop:y0}), we have
\[
\size(\hat{\Phi}_\epsilon) \ge M \ge \cN(f_D, D^{\alpha^\ast r} \epsilon) 
\ge (D^{\alpha^\ast r} \epsilon)^{-\gamma D/r}, 
\]
whenever $D^{\alpha^\ast r}\epsilon \le \bar{\epsilon}$.
On the other hand, we can also estimate (cp. equation \eqref{eq:calculus1} in Section \ref{sec:nncalc}),
\[
\size(\hat{\Phi}_\epsilon) \le 2\Vert W_D \Vert_0 + 2\,\size(\Phi_\epsilon) \le 2D\size(\Phi_\epsilon) + 2\,\size(\Phi_\epsilon) \le 4D\size(\Phi_\epsilon).
\]
Combining these bounds, we conclude that
\[
\size(\Phi_\epsilon) \ge \frac{1}{2D}  (D^{\alpha^\ast r} \epsilon)^{-\gamma D/r} = \frac12 (D^{\alpha^\ast r + r/\gamma D} \epsilon)^{-\gamma D/r},
\]
holds for any neural network representation of $\cS_\epsilon$, whenever $D^{\alpha^\ast r}\epsilon \le \bar{\epsilon}$. And hence 
\begin{align}
\label{eq:cmplxS}
\cmplx(\cS_\epsilon) \ge \frac12 (D^{\alpha^\ast r + r/\gamma D} \epsilon)^{-\gamma D/r}, 
\end{align}
whenever $D^{\alpha^\ast r}\epsilon \le \bar{\epsilon}$.
By Lemma \ref{lem:sup} below, taking the supremum on the right over all admissible $D$ implies the lower bound
\begin{align*}
\cmplx(\cS_\epsilon) \ge \exp(c \epsilon^{-1/(\alpha^\ast + \delta/2)r} ), \quad \forall \epsilon \le \tilde{\epsilon},
\end{align*}
where $\tilde{\epsilon},c>0$ depend only on $\alpha^\ast, r, \bar{\epsilon}(r), \delta$ and $\gamma$. Recalling our choice of $\alpha^\ast = 1+\alpha + \delta/2$, and the fact that the constant $\bar{\epsilon} = \bar{\epsilon}(r)$ depends only on $r$, while $\gamma$ is universal, it follows that
\begin{align*}
\cmplx(\cS_\epsilon) \ge \exp(c \epsilon^{-1/(\alpha+ 1+ \delta) r} ), \quad \forall \epsilon \le \tilde{\epsilon}.
\end{align*}
with $\tilde{\epsilon},c>0$ depending only on $\alpha, r, \delta$.
Up to a change in notation, this is the claimed complexity bound.
\end{proof}

The following lemma addresses the optimization of the lower bound in \eqref{eq:cmplxS}:
\begin{lemma}
\label{lem:sup}
Let $r\in \N$, and $\alpha^\ast, \bar{\epsilon}, \gamma > 0$ be given. Assume that 
\[
\cmplx(\cS_\epsilon) \ge \frac12 \left(D^{\alpha^\ast r+ r/\gamma D} \epsilon\right)^{-\gamma D/r},
\]
for any $D$ of the form $D = 2^k$, $k\in \N$, and whenever $D^{\alpha^\ast r} \epsilon \le \bar{\epsilon}$.
Fix a small parameter $\delta > 0$.
There exist $\tilde{\epsilon},c>0$, depending only on $r,\alpha^\ast, \gamma, \bar{\epsilon}, \delta$, such that 
\[
\cmplx(\cS_\epsilon)
\ge 
\exp\left(c \epsilon^{-1/(\alpha^\ast + \delta)r}\right), \quad \forall \epsilon \le \tilde{\epsilon},
\]
\end{lemma}

\begin{proof}
Write $\bar{\epsilon} = e^{-\beta}$ for $\beta \in \R$. Fix a small parameter $\delta > 0$. Since we restrict attention to $\epsilon \le \tilde{\epsilon}$, we have
\[
\left(
e^{\beta}\epsilon
\right)^{-1/(\alpha^\ast+\delta)r} 
\ge 
\left(e^{\beta}\tilde{\epsilon}
\right)^{-1/(\alpha^\ast+\delta)r} = 1,
\]
provided that $\tilde{\epsilon} \le \bar{\epsilon}$.
Given $\epsilon \le \tilde{\epsilon}$, choose $k\in \N$, such that
\[
2^{k-1} < \left(e^\beta\epsilon\right)^{-1/(\alpha^\ast+\delta)r} \le 2^{k}.
\]
Let $D = 2^k$. Note that this defines a function $D = D(\epsilon)$. For any $\epsilon$, we can write 
\begin{align}
\label{eq:De}
D(\epsilon) = \xi \, \left(e^\beta \epsilon\right)^{-1/(\alpha^\ast+\delta)r},
\end{align}
for some $\xi \in (1/2,1]$. We now note that for $\epsilon \le \tilde{\epsilon}$,
\[
\frac{1}{\gamma D} 
= \frac{\left(e^\beta\epsilon\right)^{1/(\alpha^\ast+\delta)r}}{\gamma\xi} 
\le \frac{2\left(e^\beta\tilde{\epsilon}\right)^{1/(\alpha^\ast+\delta)r}}{\gamma}.
\]
Decreasing the size of $\tilde{\epsilon} = \tilde{\epsilon}(r,\gamma,\alpha^\ast,\bar{\epsilon},\delta)$ further, we can ensure that for $\epsilon \le \tilde{\epsilon}$,
\[
\frac{1}{\gamma D(\epsilon)} \le \frac{2 \left(e^\beta\tilde{\epsilon}\right)^{1/(\alpha^\ast+\delta)r}}{\gamma} \le \frac{\delta}{2}.
\]
Note also that $2^{r/\gamma D} \le 2^{\delta r/2} \le D^{\delta r/2}$ for any $D$ of the form $D= 2^k$, $k\in \N$.
It thus follows that for any given $\epsilon \le \tilde{\epsilon}$, and with our particular choice of $D = D(\epsilon)$ satisfying \eqref{eq:De}, we have
\[
2^{r/\gamma D} D^{\alpha^\ast r + r/\gamma D} \epsilon 
\le
D^{(\alpha^\ast + \delta)r} \epsilon 
=
e^{-\beta} \xi^{(\alpha^\ast + \delta)r} \le e^{-\beta}.
\]
Note that this in particular implies that $D^{\alpha^\ast r} \epsilon
\le e^{-\beta} = \bar{\epsilon}$.
We conclude that
\begin{align*}
\cmplx(\cS_\epsilon) 
&\ge 
\frac12 (D^{\alpha^\ast r + r/\gamma D} \epsilon )^{-\gamma D / r}
\\
&=
(2^{r/\gamma D}D^{\alpha^\ast r + r/\gamma D} \epsilon )^{-\gamma D / r}
\\
&\ge 
(D^{(\alpha^\ast + \delta)r} \epsilon )^{-\gamma D / r}
\\
&\ge 
\exp\left( \frac{\beta\gamma D}{r}\right)
\\
&=
\exp\left( \frac{\beta\gamma \xi e^{-1/(\alpha^\ast+\delta)r} }{r}\epsilon^{-1/(\alpha^\ast+\delta)r}\right)
\\
&\ge 
\exp\left( \frac{\beta\gamma e^{-1/(\alpha^\ast+\delta)r} }{2r}\epsilon^{-1/(\alpha^\ast+\delta)r}\right).
\end{align*}
Upon defining $c = c(r,\gamma,\alpha^\ast,\bar{\epsilon},\delta)$ as
\[
c := \frac{\beta\gamma e^{-1/(\alpha^\ast+\delta)r} }{2r},
\]
we obtain the lower bound
\[
\cmplx(\cS_\epsilon)
\ge
\exp\left( c\epsilon^{-1/(\alpha^\ast+\delta)r}\right), \quad \forall \, \epsilon \le \bar{\epsilon}.
\]
This concludes the proof of Lemma \ref{lem:sup}.
\end{proof}

In the lemma below, we provide a simple result on Fr\'echet differentiability which was used in our proof of Proposition \ref{thm:codexp}:

\begin{lemma}{(Fr\'echet differentiability of a series)}
\label{lem:frechet}
Assume that we are given a bounded family of functions $f_{D} \in C^{r}(\R^D)$ indexed by integers $D = 2^k$, $k\in \N$, such that $\Vert f_D \Vert_{C^{r}(\R^D)} \le 1$ for all $D$. Let $e^\ast_1,e^\ast_2,\dots: \cX \to \R$ be a sequence of linear functionals, such that $\Vert e^\ast_j \Vert_{\cX^\ast} \le M$ for all $j\in \N$. Let $c,\alpha > 0$ be given, and assume that $\alpha^\ast > 1+\alpha$. Then, the functional $\cS^\dagger: \cX \to \R$ defined by the series
\[
\cS^\dagger(u) := \sum_{k=1}^\infty 2^{- \alpha^\ast r k} 
f_{2^k}
\left(
c2^{\alpha k} (e^\ast_{2^{k}}( u), \dots, e^\ast_{2^{k+1}-1}(u))
\right),
\]
is $r$-times Fr\'echet differentiable.
\end{lemma}

\begin{proof}
By assumption on $f_{2^k}$ and the linearity of the functionals $e^\ast_j$, each nonlinear functional $\cF_k(u) :=
f_{2^k}
\left(
c2^{\alpha k} (e^\ast_{2^{k}}( u), \dots, e^\ast_{2^{k+1}-1}(u))
\right)$ in the series defining $\cS^\dagger$ is $r$-times continuously differentiable. Fixing $u\in \cX$, let us denote $x = c2^{\alpha k} (e^\ast_{2^{k}}( u), \dots, e^\ast_{2^{k+1}-1}(u))$. The $\ell$-th total derivative $d^\ell \cF_k$ of $\cF_k$ ($\ell \le r$) is given by
\[
d^\ell \cF_k(u)[v_1,\dots, v_\ell]
=
c^\ell 2^{\alpha \ell k} 
\sum_{j_1,\dots,j_\ell =1}^{2^{k}}
\frac{\partial^\ell f_{2^k}(x)}{\partial x_{j_1} \dots \partial x_{j_\ell}} 
\prod_{s=1}^\ell e^\ast_{2^{k}+j_s-1}(v_s).
\]
By assumption, we have
\[
\left|
\frac{\partial^\ell f_{2^k}(x)}{\partial x_{j_1} \dots \partial x_{j_\ell}}
\right|
\le 
1.
\]
Since the sum over $j_1,\dots, j_\ell$ has $2^{k\ell}$ terms, and since the functionals are bounded $\Vert e^\ast_j \Vert\le M$ by assumption, we can now readily estimate the operator norm $\Vert d^\ell \cF_k(u) \Vert$ for $\ell \le r$ by
\[
\Vert d^\ell \cF_k(u) \Vert
\le c^\ell 2^{\alpha \ell k} 2^{k\ell} M^\ell
\le
(cM)^r 2^{(\alpha+1)r k}.
\]
In particular, for any $\ell \le r$, the series
\[
\sum_{k=1}^\infty 2^{-\alpha^\ast r k} \Vert d^\ell \cF_k(u) \Vert
\le 
\sum_{k=1}^\infty 2^{-[\alpha^\ast -(\alpha+1)] r k} (cM)^r < \infty,
\]
is uniformly convergent. Thus, $\cS^\dagger$ is a uniform limit of $r$-times continuously differentiable functionals, all of whose derivatives of order $\ell \le r$ are also uniformly convergent. From this, it follows that $\cS^\dagger$ is itself $r$-times continuously Fr\'echet differentiable.
\end{proof}

\subsection{Proof of Corollary \ref{cor:codexp}}
\label{Acor:codexp}

Let $\cY = \cY(\Omega;\R^p)$ be a function space with continuous embedding in $C(\Omega;\R^p)$. We will only consider the case $p=1$, the case $p>1$ following by an almost identical argument; Let $\phi \ne 0 \in \cY$ be a non-trivial function. Since $\cY = \cY(\Omega)$ is continuously embedded in $C(\Omega)$, it follows that point evaluation $\ev_y(\phi) = \phi(y)$ is continuous. Given that $\phi$ is non-trivial, there exists $y\in \Omega$, such that $\ev_y(\phi) \ne 0$. We may wlog assume that $\ev_y(\phi) = \phi(y) = 1$. Let $\cF^\dagger: \cX \to \R$ be a functional exposing the curse of parametric complexity, as constructed in Theorem \ref{thm:codexp}. We define an $r$-times Fr\'echet differentiable operator $\cS^\dagger: \cX \to \cY$ by $\cS^\dagger(u) := \cF^\dagger(u) \phi$.

The claim now follows immediately by observing that 
\[
\ev_y \circ \cS^\dagger(u) = \cF^\dagger(u), \quad \forall \, u \in \cX,
\]
and by noting that if $\cS_\epsilon: \cX \to \cY$ is an operator of neural network-type, then $\ev_y \circ \cS_\epsilon: \cX \to \R$ is a functional of neural network-type, and by assumption, with $C := \Vert \ev_y \Vert_{\cY\to \R} $, 
\begin{align*}
\sup_{u\in K} |\cF^\dagger(u) - \ev_y \circ \cS_\epsilon(u)|
&=
\sup_{u\in K} |\ev_y \circ \cS^\dagger(u) - \ev_y \circ \cS_\epsilon(u)|
\\
&\le
\sup_{u\in K} C \Vert \cS^\dagger(u) -  \cS_\epsilon(u)\Vert_{\cY}
\\
&\le
C \epsilon.
\end{align*}
By our choice of $\cF^\dagger:\cX \to \R$, this implies that the complexity of $\ev_y \circ \cS_\epsilon$ is lower bounded by an exponential bound $\ge \exp(c\epsilon^{-1/(\alpha+1+\delta)r})$ for some constant $c = c(\alpha, \delta, r)$. This in turn implies that
\[
\cmplx(\cS_\epsilon) 
= \sup_{y\in \Omega} \cmplx(\ev_y \circ \cS_\epsilon) 
\ge \exp(c\epsilon^{-1/(\alpha+1+\delta)r}).
\]
This lower bound implies the exponential lower bound of Corollary \ref{cor:codexp}.

\subsection{Proofs of Lemmas \ref{lem:c-pca} -- \ref{lem:Cs}}
\label{A:lemmas}

\subsubsection{Proof of Lemma \ref{lem:c-pca}}
\label{A:c-pca}

\begin{proof}{(Lemma \ref{lem:c-pca})}
We want to show that a PCA-Net neural operator $\cS = \cR \circ \Psi \circ \cE$ is of neural network-type, and aim to estimate $\size(\Psi)$ in terms of $\cmplx(\cS)$. To this end, we 
assume that $\cY = \cY(\Omega; \R^p)$ is a Hilbert space of functions. Since $\cR$ is by definition linear, then given an evaluation point $y \in \Omega$, the mapping
$\beta \mapsto \cR(\beta)(y) \equiv \ev_y \circ \cR(\beta)$ defines a linear map $\ev_y \circ \cR: \R^{\DY} \to \R^p$. We can represent $\ev_y\circ \cR$ by  matrix multiplication: $\ev_y \circ \cR(\beta) = V_y  \beta$, with $V_y \in \R^{p\times \DY}$. The encoder $\cE: \cX \to \R^{\DX}$ is linear by definition, thus we can take $\cL := \cE$ for the linear map in the definition of ``operator of neural network-type''. Define a neural network $\Phi_y: \R^{\DX} \to \R^p$ by $\Phi_y(\alpha) = V_y\Psi(\alpha)$. Then, we have the identity
\begin{align*}
\ev_y \circ \cS(u)
= (\ev_y \circ \cR) \circ \Psi (\cE u)
= \Phi_y(\cL u), 
\end{align*}
for all $u \in \cX$. This shows that $\cS$ is of neural network-type. We now aim to estimate $\cmplx(\cS)$ in terms of $\size(\Psi)$. To this end, write $\Psi(\alpha) = [\Psi_1(\alpha),\dots, \Psi_{\DY}(\alpha)]$ with component mappings $\Psi_j: \R^{\DX} \to \R$. Let $\cJ = \set{j}{j\in \{1,\dots, \DY\}, \; \Psi_j \ne 0}$ be the subset of indices for which $\Psi_j: \R^{\DX} \to \R$ is not the zero function. Define a (sparsified) matrix $\hat{V}_y \in \R^{p \times \DY}$ with $j$-th column $[\hat{V}_y]_{:,j}$ defined by
\begin{align*}
[\hat{V}_y]_{:,j} := 
\begin{cases}
[{V}_y]_{:,j}, & j\in \cJ, \\
0 & j\notin \cJ.
\end{cases}
\end{align*}
Then, we have $\Vert \hat{V}_y \Vert_0 \le p |\cJ| \le p\, \size(\Psi)$, and identity $\Phi_y(\alpha) = \hat{V}_y \Psi(\alpha)$ for all $\alpha \in \R^{\DX}$.
Thus, using the concept of sparse concatenation \eqref{eq:calculus2}, we can upper bound
the complexity, $\cmplx(\cS)$, in terms of the $\size(\Psi)$ of the neural network $\Psi$  as follows:
\[
\cmplx(\cS) \le \sup_y \size(\Phi_y) \le \sup_y 2\Vert V_y\Vert_0 + 2\size(\Psi) \le 2(p+1) \size(\Psi).
\]
This is the claimed lower bound on $\size(\Psi)$.
\end{proof}

\subsubsection{Proof of Lemma \ref{lem:c-don}}
\label{A:c-don}

\begin{proof}{(Lemma \ref{lem:c-don})}
We observe that with $D := \DX$, for any $y\in \Omega$ the encoder $\cL := \cE: \cX \to \R^D$ is linear, and 
\[
\ev_y \circ \cS(u) = \Phi_y(\cL u), \quad \forall \, u \in  \cX,
\]
where $\Phi_y(\alpha) = \sum_{j=1}^{\DY} \Psi_j(\alpha) \phi_j(y)$ defines a neural network, $\Phi_y: \R^{D} \to \R^p$.
Thus, DeepONet $\cS = \cR \circ \Psi \circ \cE$ is of neural network-type. To estimate the complexity, $\cmplx(\cS)$, we let $\cJ^2$ be the set of indices $(i,j)\in \{1,\dots, \DY\}^2$, such that the $i$-th component, $[\phi_j(y)]_i$, of the vector $\phi_j(y) \in \R^p$ is non-zero. Let $\hat{V}_y \in \R^{p\times \DY}$ be the matrix with entries $[\hat{V}_y]_{i,j} = [\phi_j(y)]_i$, so that $\Phi_y(\alpha) = \hat{V}_y \Psi(\alpha)$ for all $\alpha$. Note that $\hat{V}_y$ has precisely $|\cJ^2|$ non-zero entries, and that $|\cJ^2|\le \size(\phi)$, since $[\phi_j(y)]_i \ne 0$ is non-zero for all $(i,j) \in \cJ^2$. Thus, it follows that
\[
\cmplx(\cS) \le \sup_y \size(\Phi_y) \le 2|\cJ^2| +2\, \size(\Psi) \le
2(\size(\phi)+\size(\Psi)).
\] 
\end{proof}

\subsubsection{Proof of Lemma \ref{lem:c-nomad}}
\label{A:c-nomad}

\begin{proof}{(Lemma \ref{lem:c-nomad}})
To see the complexity bound, we recall that for any $y\in \Omega$, we can choose $\Phi_y(\slot) := Q(\Psi(\slot),y)$ to obtain the representation 
\[
\ev_y \circ \cS(u) = \Phi_y(\cL u),
\]
where $\cL u \equiv \cE(u)$ is given by the linear encoder.
The composition of two neural networks $Q(\slot,y)$ and $\Psi(\slot)$ can be represented by a neural network of size at most $2\size(\Psi) + 2\size(Q(\slot,y)) \le 2\size(\Psi) + 2\size(Q)$. We thus have the lower bound,
 \[
\cmplx(\cS) \le \sup_y\size(\Phi_y) \le 2\size(Q)+2\size(\Psi).
 \]
 This shows the claim.
\end{proof}

\subsection{Proof of Lemma \ref{lem:fno-nnt}}
\label{A:fno-nnt}

\begin{proof}
Our aim is to show that $\cS: L^2(\T;\R) \to \R$, 
\[
\cS(u) := \int_{\T} \sigma(u(x)) \, dx,
\]
is not of neural network-type. We argue by contradiction. Suppose that $\cS$ was of neural network type. By definition, there exists a linear mapping $\cL: \cL^2(\T;\R) \to \R^\ell$,  and a neural network $\Phi: \R^\ell \to \R$, for some $\ell\in \N$ such that 
\begin{align}
\label{eq:id1}
\cS(u) = \Phi(\cL u).
\end{align}
In the following we will consider $\varphi_j(x) := \sin(jx)$ for $j\in \N$. Since $\sigma(t) \ge 0$ for all $t\in \R$, and $\sigma(t) > 0$ for $t>0$, we have
\begin{align}
\label{eq:id2}
\cS(u) = \int_\Omega \sigma(u(x)) \, dx = 0 
\quad \iff \quad
u(x) \le 0, \quad \forall x \in [0,2\pi].
\end{align}
Now, fix any $D>\ell$, and consider $\iota: \R^D \to L^2(\T;\R)$, $\iota \beta := \sum_{j=1}^D \beta_j \sin(jx)$. Since $\iota$ and $\cL$ are linear mappings, it follows that 
\[
\cL \circ \iota: \R^D \to \R^\ell, \quad \beta \mapsto \cL \iota \beta,
\]
is a linear mapping. Represent this linear mapping by a matrix $W\in \R^{\ell\times D}$. In particular, by \eqref{eq:id1}, we have
\begin{align}
\label{eq:id3}
\cS(\iota \beta) = \Phi(W\beta), \quad \forall \beta \in \R^D.
\end{align}
Since $D>\ell$, it follows that $\ker(W) \ne \{0\}$ is nontrivial. Let $\beta\ne 0$ be an element in the kernel, and consider $u_\beta(x) := \iota \beta(x) = \sum_{j=1}^D \beta_j \sin(jx)$. Since $u_\beta(x)$ is not identically equal to $0$, either $u_\beta(x)$ or $-u_\beta(x) = u_{-\beta}(x)$ must be positive somewhere in the domain $\T$. Upon replacing $\beta \to -\beta$ if necessary, we may wlog assume that $u_\beta(x) > 0$ for some $x\in \T$. From \eqref{eq:id2} and \eqref{eq:id3}, it now follows that
\[
0 \ne \cS(u_\beta) = \cS(\iota \beta) = \Phi(W\beta) = \Phi(0) = \cS(0) = 0.
\]
A contradiction. Thus, $\cS$ cannot be of neural network-type.
\end{proof}

\subsection{Proof of curse of parametric complexity for FNO, Theorem  \ref{thm:fno-cod}}
\label{A:cod2}

Building on the curse of parametric complexity for operators of neural network-type, we next show that FNOs also suffer from a similar curse, as stated in Theorem  \ref{thm:fno-cod}.

\begin{proof}{(Theorem  \ref{thm:fno-cod})}
Let $\cS^\dagger: \cX \to \R$ be an $r$-times Fr\'echet differentiable functional satisfying the conclusion of Theorem \ref{thm:codexp} (CoD for functionals of neural network-type). In the following, we show that $\cS^\dagger$ also satisfies the conclusions of Theorem  \ref{thm:fno-cod}. Our proof argues by contradiction: we assume that $\cS^\dagger$ can be approximated by a family of discrete FNOs satisfying the error and complexity bounds of Theorem  \ref{thm:fno-cod}, i.e. 
\begin{enumerate}
\item \emph{Complexity bound:} There exist constant $c>0$, such that the discretization parameter $N_\epsilon\in \N$, and the total number of non-zero parameters $\size(\cS^{N_\epsilon}_\epsilon)$ are bounded by $N_\epsilon,\size(\cS^{N_\epsilon}_\epsilon) \le \exp(c\epsilon^{-1/(1+\alpha+\delta) r})$, for all $\epsilon \le \bar{\epsilon}$.
\item \emph{Accuracy:} We have
\[
\sup_{u\in K} \left| \cS^\dagger(u) - \cS^{N_\epsilon}_\epsilon(u) \right| \le \epsilon,
\quad \forall \, \epsilon > 0.
\]
\end{enumerate}
Then we show that this implies the existence of a family of operators of \emph{neural network-type} $\tS_\epsilon$, which satisfy for some $\delta'>0$, and for all sufficiently small $\epsilon > 0$,
\begin{itemize}
\item complexity bound
$\cmplx(\tS_\epsilon) \le \exp(c\epsilon^{-1/(1+\alpha+\delta') r})$,
\item and error estimate
$
\max_{u\in K} \left| \cS^\dagger(u) -  \tS_\epsilon(u)\right| \le \epsilon,
$
\end{itemize}
with $c>0$ a potentially different constant.
By choice of $\cS^\dagger$, the existence of $\tS_\epsilon$ is ruled out by Theorem \ref{thm:codexp}, providing the desired contradiction.

In the following, we discuss the construction of $\tS_\epsilon$:
Let $\cS_\epsilon^{N_\epsilon}$ be a family of FNOs satisfying (1) and (2) above. Fix $\epsilon > 0$ for the moment. To simplify notation, we will write $N={N_\epsilon}$, in the following. We recall that, by definition, the discretized FNO $\cS_\epsilon^{N}$ can be written as a composition $\cS^{N}_\epsilon = \tQ \circ \cL_L \circ \dots \circ \cL_1 \circ \cR$, where \[
\cR: u(x) \mapsto \chi(x_{j_1,\dots, j_d},u(x_{j_1,\dots, j_d})),
\]
and 
\[
\tQ: \left( v_{j_1,\dots, j_d} \right) \mapsto \frac{1}{N^d}\sum_{j_1,\dots, j_d} q(y_{j_1,\dots, j_d},v_{j_1,\dots, j_d}),
\]
are defined by neural networks $\chi$ and $q$, respectively, $\cL_\ell$ is of the form 
\[
\cL_\ell: v \mapsto \sigma\left(W_\ell v + \cF_N^{-1} \left(\hat{T}_\ell \cF_N v\right) + b_\ell \right),
\quad
v = \left(v_{j_1,\dots, j_d}\right)_{j_1,\dots, j_d = 1}^N,
\]
with $W_\ell \in \R^{d_v\times d_v}$, Fourier multiplier $\hat{T}_\ell = \{[\hat{T}_\ell]_k\}_{\Vert k \Vert_{\ell^\infty}\le k_{\mathrm{max}}}$, where $[\hat{T}_\ell]_k \in \C^{d_v\times d_v}$, and $\cF_N$ ($\cF_N^{-1}$) denote discrete (inverse) Fourier transform, and where the bias $b$ is determined by its Fourier coefficients $\hat{b}_k$, $\Vert k \Vert_{\ell^\infty} \le k_{\mathrm{max}}$. We also recall that the size of $\cS^N_\epsilon$ is given by the total number of non-zero components, i.e.
\[
\size(\cS^N_\epsilon) 
=
\size(\chi) + \size(q) + \sum_{\ell=1}^L \left\{
\Vert W_\ell \Vert_{0} + \Vert \hat{T}_\ell \Vert_0 + \sum_{\Vert k\Vert_{\ell^\infty} \le k_{\mathrm{max}}} \Vert \hat{b}_k \Vert_{0}
\right\}.
\]

We now observe that, after flattening the tensor 
\[
\left(v_{j_1,\dots, j_d}\right) \in \R^{N\times \dots \times N \times d_v} \simeq \R^{d_v N^d},
\]
the (linear) mapping $v \mapsto W_\ell v$ can be represented by multiplication against a sparse matrix with at most $\Vert W_\ell \Vert_0 N^d$ non-zero entries. For the non-local operator $\cF^{-1}_N \hat{T}_\ell \cF_N$, we note that for $v\in \R^{d_v N^d}$, the number $\kappa$ of components (channels) that need to be considered is bounded by $\kappa \le \min(d_v,\Vert \hat{T}_\ell \Vert_0)\le \Vert \hat{T}_\ell \Vert_0$. Discarding any channels that are zeroed out by $\hat{T}_\ell$, a naive implementation of $\cF^{-1}_N \hat{T}_\ell \cF_N$ thus amounts to a matrix representation of a linear mapping $\R^{\kappa N^d} \to \R^{\kappa N^d}$,
requiring at most $\kappa^2 N^{2d} \le \Vert \hat{T}_\ell \Vert_0^2 N^{2d}$ non-zero components. Thus, each discretized hidden layer $\cL_\ell$ can be represented \emph{exactly} by an ordinary neural network layer $L_\ell = \sigma(A_\ell v + c_\ell)$ with matrix $A_\ell \in \R^{d_v N^d \times d_v N^d}$ and bias $c_\ell\in \R^{d_v N^d}$, satisfying the following bounds on the number of non-zero components:
\[
\Vert A_\ell \Vert_0 \le \Vert W_\ell \Vert_0 N^d + \Vert \hat{T}_\ell \Vert_0^2 N^{2d}, 
\quad
\Vert c_\ell \Vert_0 \le \sum_{|k|\le k_{\mathrm{max}}} \Vert [\hat{b}_\ell]_k \Vert_0 N^d.
\]
Similarly, the input and output layers $\cR$ and $\tQ$ can be represented \emph{exactly} by ordinary neural networks $R: \R^{k N^d} \to \R^{d_v N^d}$ and $Q: \R^{d_v N^d} \to \R$, with 
\[
\size(R) \le N^d \size(\chi), 
\quad
\size(Q) \le N^d \size(q),
\]
obtained by parallelization of $\chi$, resp. $q$, at each grid point.
Given the above observations, we conclude that, with canonical identification $\R^{N\times \dots \times N \times k} \simeq \R^{kN^d}$ and $\R^{N\times \dots \times N \times p} \simeq \R^{pN^d}$, the discretized FNO $\cS^N_\epsilon$ can be represented by an ordinary neural network $\Phi: \R^{kN^d} \to \R^{pN^d}$, $\Phi = Q \circ L_L \circ \dots \circ L_1 \circ R$, with 
\begin{align*}
\size(\Phi) 
&\le 
\sum_{\ell=1}^L 
\left\{
\Vert W_\ell \Vert_0 N^d + \Vert \hat{T}_\ell \Vert_0^2 N^{2d} + \sum_{|k|\le k_{\mathrm{max}}}\Vert \hat{b}_k \Vert_0 N^d
\right\}
\\
&\qquad
+N^d \size(\chi) + N^d \size(q)
\\
&\le 
N^{2d} \size(\cS_\epsilon^N)^2.
\end{align*}
By assumption on $\cS_\epsilon^N$ (for which we aim to show that it leads to a contradiction), we have $\size(\cS_\epsilon^N)\le \exp(c\epsilon^{-1/(1+\alpha+\delta) r})$, and $N_\epsilon\le \exp(c\epsilon^{-1/(1+\alpha+\delta) r})$. It follows that
\[
\size(\Phi)
\le N^{2d} \size(\cS_\epsilon^N)^2
\le
\exp((2d+2) c\epsilon^{-1/(1+\alpha+\delta) r}).
\]
In addition, $\Phi$ trivially defines an operator of neural network-type, $\tS_\epsilon: C^s(\Omega; \R^k) \to \R$, by 
\[
\tS_\epsilon(u) := \Phi\left( (u(x_{j_1,\dots, j_d}))_{j_1,\dots, j_d=1}^N \right),
\]
To see this, we simply note that the point-evaluation mapping $\cL: u \mapsto u(x_{j_1,\dots, j_d})_{j_1,\dots, j_d=1}^N$ is linear, and hence we have the representation
\[
\tS_\epsilon(u) = \Phi(\cL u).
\]
By the above construction, we have $\tS_\epsilon(u) \equiv \cS_\epsilon^{N_\epsilon}(u)$ for all input functions $u$.

To summarize, assuming that a family of FNOs $\cS_\epsilon^{N_\epsilon}$ exists, satisfying (1) and (2) above, we have constructed a family $\tS_\epsilon$ of operators of neural network type, with 
\[
\sup_{u\in K} \left| \cS^\dagger(u) - \tS_\epsilon(u) \right| 
=
\sup_{u\in K} \left| \cS^\dagger(u) - \cS_\epsilon^{N_\epsilon}(u) \right| 
\le \epsilon,
\]
and
\[
\cmplx(\tS_\epsilon) \le \size(\Phi) \le \exp(c''\epsilon^{-1/(1+\alpha+\delta) r}),
\]
where $c'' = (2d+2)c>0$ is a constant independent of $\epsilon$.

By Theorem \ref{thm:codexp}, and fixing any $0 < \delta'  < \delta$, we also have the following lower bound for all sufficiently small $\epsilon$:
\[
\exp(c' \epsilon^{-1/(1+\alpha+\delta') r}) \le \cmplx(\tS_\epsilon),
\]
where $c'>0$ is independent of $\epsilon$. From the above two-sided bounds, it thus follows that
\[
\exp(c' \epsilon^{-1/(1+\alpha+\delta') r}) 
\le 
\cmplx(\tS_\epsilon)
\le 
\exp(c''\epsilon^{-1/(1+\alpha+\delta) r}),
\]
for all $\epsilon$ sufficiently small, and where by our choice of $\delta',\delta$: $0< 1+\alpha+\delta' < 1+\alpha+\delta$ and $c',c''> 0$ are constants. Since $\epsilon^{-1/(1+\alpha+\delta') r}$ grows faster than $\epsilon^{-1/(1+\alpha+\delta) r}$ as $\epsilon \to 0$, this leads to the desired contradiction. We thus conclude that a family $\cS_\epsilon^{N_\epsilon}$ of discretized FNOs as assumed above cannot exist. This concludes the proof.
\end{proof}

\section{Short-time existence of $C^r$-solutions}\label{A:E}

The proof of short-time existence of solutions of the Hamilton-Jacobi equation \eqref{eq:HJ} is based on the following Banach space implicit function theorem:

\begin{theorem}[{Implicit Function Theorem, see e.g. \cite[Section 2.2]{chow2012methods}}]
\label{thm:implicit}
Let $U\subset X$, $V \subset Y$ be open subsets of Banach spaces $X$ and $Y$. Let 
\[
F: U\times V \to Z, \quad (u,v) \mapsto F(u,v),
\]
be a $C^p$-mapping ($p$-times continuously Fr\'echet differentiable). Assume that there exist $(u_0,v_0) \in U\times V$ such that $F(u_0,v_0) = 0$, and such that the Jacobian with respect to $v$, evaluated at the point $(u_0,v_0)$,
\[
D_v F(u_0,v_0): X \to Z,
\]
is a linear isomorphism. Then there exists a neighbourhood $U_0\subset U$ of $u_0$, and a $C^p$-mapping $\psi: U_0 \to V$, such that
\[
F(u,\psi(u)) = 0, \quad \forall \, u \in U_0.
\]
Furthermore, $\psi$ is unique, in the sense that for any $u\in U_0$, $F(u,v) = 0$ implies $v = \psi(u)$.
\end{theorem}

As a first step toward proving short-time existence for \eqref{eq:HJ}, we prove that under the no-blowup Assumption \ref{ass:blowup}, the semigroup $\Psi_t^\dagger$ \eqref{eq:flowmap} exists for any $t\ge 0$. 

\begin{proof}{(Lemma \ref{lem:traj-existence})}
By classical ODE theory, it is immediate that for any initial data $(q_0,p_0,z_0) \in \Omega \times \R^d \times \R$ a maximal solution of the ODE system 
(\ref{eq:flow}) exists, and is unique, over a short time-interval. It thus remains to prove that this solution in fact exists globally, i.e. that the solution does not escape to infinity in finite time. Since $z$ solves $\dot{z} = \cL(q,p)$, with a right-hand side that only depends on $q$ and $p$, it will suffice to prove that the Hamiltonian trajectory $\dot{q} = \nabla_p H(q,p)$, $\dot{p} = -\nabla_q H(q,p)$ with initial data $(q_0,p_0)$ exists globally in time. To this end, we note that, by Assumption \ref{ass:blowup}, we have
\[
\frac{d}{dt} (1+|p|^2) 
=
2p \cdot \dot{p} 
= -2p \cdot \nabla_q H(q,p)
\le 
2L_H (1+|p|^2).
\]
Gronwall's lemma thus implies that $|p| \le \sqrt{1+|p_0|^2} \exp(L_H t)$ remains bounded for all $t\ge 0$. This shows that blow-up cannot occur 
in the $p$-variable in a finite time. On the other hand, the right-hand side of the ODE system is periodic in $q$, and hence (by the bound on $p$) blow-up is also ruled out for $q$. In particular, Assumption \ref{ass:blowup} ensures that the Hamiltonian trajectory $t \mapsto (q(t),p(t))$ exists globally in time. As argued above, this in turn implies the global existence of the flow map $t\mapsto \Psi^\dagger_t$.
\end{proof}

We now apply the Implicit Function Theorem \ref{thm:implicit} to prove Lemma \ref{lem:invert}.

\begin{proof}{(Lemma \ref{lem:invert})}
Let $X = C^r_\per(\Omega) \times \R^d \times \R$ with $r\ge 2$. Define 
\[
F : X \times \R^d \to \R^d, 
\quad
(u_0,q,t; q_0) \mapsto q - q_t(q_0,\nabla_q u_0(q_0)),
\]
i.e. we set 
\[
F(u_0,q,t;q_0) := q - q_t(q_0, \nabla_q u_0(q_0)),
\]
with $q_t: \Omega \times \R^d \to \Omega$, $(q_0,p_0) \mapsto q_t(q_0,p_0)$ the spatial characteristic mapping of the Hamiltonian system \eqref{eq:flow} at time $t\ge 0$, where, recall, we have assumed that $u_0 \in \cF.$
Under Assumption \ref{ass:blowup}, the spatial characteristic mapping, and hence $F$, is well-defined for any $t\ge0$ (see Lemma \ref{lem:traj-existence}).
Since $H\in C^{r+1}(\Omega \times \R^d)$, the mapping $\Omega \times \R^d \times \R \to \Omega$, $(q_0,p_0,t) \mapsto q_t(q_0,p_0)$ is $C^r$. The mapping $F_1: \Omega \times C^r_\per(\Omega) \to \R^d$, $(q_0,u_0) \mapsto F_1(q_0,u_0) := \nabla_q u_0(q_0)$ is $C^{r-1}$ in the first argument, and is a continuous linear mapping in the second argument (and hence infinitely Frechet differentiable in the second argument). Hence, the composition 
\[
(q_0, u_0) \mapsto (q_0, \nabla_q u_0(q_0)) \mapsto q_t(q_0, \nabla_q u_0(q_0)),
\]
is a $C^{r-1}$ mapping. And, as a consequence, the mapping $F: X \times \R^d \to \R^d$ is a $C^{r-1}$ mapping. 

Since 
\[
F(u_0,q,t;q_0) = q - q_t(q_0, \nabla_q u_0(q_0)),
\]
we clearly have 
\[
F(u_0,q_0,0;q_0) = 0,
\]
for any $q_0 \in \Omega$ and $u_0 \in C^r_\per(\Omega)$, and the derivative with respect to the last argument
\[
D_{q_0}F(u_0,q_0,0;q_0): \R^d \to \R^d,
\]
is given by
\[
D_{q_0}F(u_0,q_0,0;q_0)
=
D_{q_0}\left[ q - q_0 \right] = -\bm{1}_{d\times d},
\]
which defines an isomorphism $\R^d \to \R^d$.

By the implicit function theorem, for any $\bar{q}_0 \in \Omega$ and $\bar{u}_0 \in C^r_\per(\Omega)$, there exist $\epsilon = \epsilon(\bar{q}_0,\bar{u}_0), r  = r(\bar{q}_0,\bar{u}_0), t^\ast = t^\ast(\bar{q}_0,\bar{u}_0) > 0$, and a mapping
\[
\psi_{\bar{q}_0,\bar{u}_0}: B_{\epsilon}(\bar{q}_0) \times B_{r}(\bar{u}_0) \times [0,t^\ast) \to \R^d,
\]
such that $F(q,u_0,t;q_0) = 0$ for 
\[
(q,u_0,t) \in  B_{\epsilon}(\bar{q}_0) \times B_{r}(\bar{u}_0) \times [0,t^\ast),
\]
if, and only if, 
\[
q_0 = \psi_{\bar{q}_0,\bar{u}_0}(q,u_0,t).
\]
Fix $\bar{u}_0\in C^r_\per(\Omega)$ for the moment. Since $\Omega$ is compact, we can choose a finite number of points $\bar{q}_0^{(1)}, \dots, \bar{q}_0^{(m)}$, such that 
\[
\Omega \subset \bigcup_{j=1}^m B_{\epsilon(\bar{q}_0^{\j},\bar{u}_0)}(\bar{q}_0^{\j}).
\]
Let 
\begin{align*}
t^\ast(\bar{u}_0) &:= \min_{j=1,\dots, m} t^\ast(\bar{q}^{\j}_0, \bar{u}_0),
\\
r(\bar{u}_0) &:= \min_{j=1,\dots, m} r(\bar{q}_0^{\j},\bar{u}_0).
\end{align*}
Due to the uniqueness property of each $\psi_{\bar{q}^{\j}_0,\bar{u}_0}$, $j=1,\dots,m$, all of these mappings have the same values on overlapping domains. Hence, we can combine them into a single map
\[
\psi_{\bar{u}_0}: \Omega \times B_{r(\bar{u}_0)}(\bar{u}_0) \times [0,t^\ast(\bar{u}_0)) \to \R^d.
\]
Furthermore, since $\psi_{\bar{q}^{\j}_0,\bar{u}_0} \in C^{r-1}$, we also have $\psi_{\bar{u}_0} \in C^{r-1}$.
Similarly, we can cover the compact set $\cF \subset C^r_\per(\Omega)$ by a finite number of open balls $B_{r(\bar{u}_0^{(k)})}(\bar{u}_0^{(k)})$, $k=1,\dots, K$,
\[
\cF \subset \bigcup_{k=1}^K B_{r(\bar{u}_0^{(k)})}(\bar{u}_0^{(k)}).
\]
Setting $T^\ast := \min_{k=1,\dots, K} t^\ast(\bar{u}_0^{(k)}) > 0$, the uniqueness property of $\psi_{\bar{u}_0^{(k)}}$ again implies that these mappings agree on overlapping domains. Hence,  we can combine them into a global map, and obtain a map
\[
\psi: \Omega \times \cF \times [0,T^\ast) \to \R^d,
\]
which satisfies $F(q,u_0,t;\psi(q,u_0,t)) = 0$ for all $q\in \Omega$, $u_0\in \cF$ and $t<T^\ast$. Furthermore, this $\psi$ is still a $C^{r-1}$ map and it is unique, in the sense that 
\[
F(q,u_0,t;q_0) = 0 
\quad \Leftrightarrow \quad
q_0 = \psi(q,u_0,t), \quad \forall \, q\in \Omega, \; u_0\in C^r_\per(\Omega), \; t < T^\ast,
\]
i.e. for any $u_0\in \cF$ and $t\in [0,T^\ast)$, we have $q = q_t(q_0,\nabla_q u_0(q_0))$ if and only if $q_0 = \psi(q,u_0,t)$. In particular, this shows that for any $u_0\in \cF$, $t \in [0,T^\ast)$, the $C^{r-1}$-mapping
\[
\Phi^{\dagger}_t(\slot;u_0): \Omega \to \Omega, 
\quad 
q_0 \mapsto \Phi^{\dagger}_t(q_0;u_0) := q_t(q_0,\nabla_q u_0(q_0)),
\]
is invertible, with inverse $q \mapsto \psi(q,u_0,t)\in C^{r-1}$. This implies the claim.
\end{proof}

Next, we apply Lemma \ref{lem:invert} to prove short-time existence of solutions for \eqref{eq:HJ}.

\begin{proof}{(Proposition \ref{prop:short-time})}
Let $T^\ast$ be the maximal time such that the spatial characteristic mapping $\Phi^{\dagger}_{t,u_0}: \Omega \to \Omega$, $q_0 \mapsto q_t(q_0,\nabla_qu_0(q_0))$ is invertible, for any $t\in [0,T^\ast)$ and for all $u_0\in \cF$. We have $T^\ast > 0$, by Lemma \ref{lem:invert}. We denote by $\Phi^\dagger_{-t,u_0}: \Omega \to \Omega$ the inverse $\Phi^{\dagger}_{-t,u_0} := [\Phi^{\dagger}_{t,u_0}]^{-1}$. By the method of characteristics, a solution $u(q,t)$ of the Hamilton-Jacobi equations must satisfy
\begin{align}
\label{eq:S}
u(q,t) = u_0(q_0) + \int_0^t \cL(q_\tau,p_\tau) \, d\tau,
\quad 
\cL(q,p):= p\cdot \nabla_p H(q,p) - H(q,p),
\end{align}
where $q_0 = \Phi^{\dagger}_{-t,u_0}(q)$, $q_\tau,p_\tau$ are the trajectories of the Hamiltonian ODE \eqref{eq:traj} with initial data $q_0,\nabla_qu_0(q_0)$. Given a fixed $u_0 \in \cF$, we use the above expression to \emph{define} $u: \Omega\times [0,T^\ast) \to \R$. 

We first observe that $u\in C^{r-1}(\Omega\times [0,T^\ast))$, as it is the composition of $C^{r-1}$ mappings, 
\[
u(q,t) = u_0(q_0) + \int_0^t \cL(q_\tau(q_0,p_0),p_\tau(q_0,p_0)) \, d\tau,
\]
where $q_0 = \Phi_{-t,u_0}(q)$, $p_0 = \nabla_q u_0(\Phi_{-t,u_0}(q))$ are $C^{r-1}$-functions of $q$.
In particular, since $r\ge 2$, this implies that $u$ is at least $C^1$. Evaluating $du(q_t,t)/dt$ along a fixed trajectory, we find that
\begin{align}
\label{eq:1}
\partial_t u(q_t,t) + H(q_t,p_t) = \left(p_t - \nabla_q u(q_t,t)\right)\cdot D_p H(q_t,p_t).
\end{align}
Thus, to show that $u$ is a classical solution of \eqref{eq:HJ}, it remains to show that $p_t = \nabla_q u(q_t,t)$ for all $t \in [0,T^\ast)$. Assume that $u\in C^2_\per$ for the moment. We first note that for the $j$-th component of $p_t - \nabla_q u(q_t,t)$, we have (with implicit summation over repeated indices, following the ``Einstein summation rule'')
\begin{align*}
\frac{d}{dt}\left[ p^j_t - \partial_{q^j} u(q_t,t) \right]
&=
-\partial_{q^j} H(q_t,p_t) - \partial^2_{q^j,q^k}u(q_t,t)  \partial_{p^k} H(q_t,p_t)
\\
&\qquad - \partial_{q^j} \partial_t u(q_t,t).
\end{align*}
Next, we note that by the invertibility of the spatial characteristic mapping $q_0 \mapsto q_t$, we can write \eqref{eq:1} in the form 
\[
\partial_t u(q,t) = -H(q,P(q,t)) + \left( P(q,t) - \nabla_q u(q,t) \right)\cdot D_p H(q,P(q,t)),
\]
where $P(q,t) := p_t\left(\Phi^{\dagger}_{-t,u_0}(q), \nabla_qu_0(\Phi^{\dagger}_{-t,u_0}(q))\right)$.
This implies that
\begin{align*}
\partial_{q^j} \partial_t u(q,t) 
&= -\partial_{q^j} H(q,P(q,t)) -  \partial_{p^k} H(q,P(q,t)) \partial_{q^j} P^k(q,t) \\
&\qquad 
+ \left( \partial_{q^j} P^k(q,t) - \partial^2_{q^j,q^k} u(q,t) \right)\cdot \partial_{p^k} H(q,P(q,t))
\\
&\qquad + 
\left( P^k(q,t) - \partial_{q^k} u(q,t) \right)\cdot \partial_{q^j} \left[\partial_{p^k} H(q,P(q,t))\right]
\\
&= -\partial_{q^j} H(q,P(q,t))  - \partial^2_{q^j,q^k} u(q,t) \cdot \partial_{p^k} H(q,P(q,t))
\\
&\qquad + 
\left( P^k(q,t) - \partial_{q^k} u(q,t) \right)\cdot A_{j,k}(q,P(q,t)).
\end{align*}
We point out that on the last line, we have introduced $A_{j,k}(q,t) := \partial_{q^j} \left[\partial_{p^k} H(q,P(q,t))\right]$ which is a continuous function of $q$ and $t$. Choosing now $q = q_t$, so that $P(q,t) = p_t$, we thus obtain
\begin{align*}
\partial_{q^j} \partial_t u(q_t,t) 
&= -\partial_{q^j} H(q_t,p_t)  - \partial^2_{q^j,q^k} u(q_t,t) \cdot \partial_{p^k} H(q_t,p_t)
\\
&\qquad + 
\left[ p^k_t - \partial_{q^k} u(q_t,t) \right]\cdot A_{j,k}(q,t).
\end{align*}
Substitution in the ODE for $p_t - \nabla_q u(q_t,t)$ yields
\[
\frac{d}{dt}\left[ p_t - \nabla_q u(q_t,t) \right]
=
- A(q_t,t) \cdot \left[p_t - \nabla_q u(q_t,t) \right].
\]
where $A(q_t,t)$ is the matrix with components $(A_{j,k}(q_t,t))$. Since $[p_t - \nabla_q u(q_t,t)]|_{t=0} = 0$, this implies that 
\begin{align} \label{eq:2}
p_t = \nabla_q u(q_t,t), \quad \forall \, t \in [0,T^\ast),
\end{align}
along the trajectory. At this point, the conclusion \eqref{eq:2} has been obtained under the assumption that $u\in C^2_\per$, which is only ensured \emph{a priori} if $r \ge 3$. 

To prove the result also for the case $r=2$, we can apply the above argument to smooth $H^\epsilon$, $u_0^\epsilon$, which approximate the given $H$ and $u_0$, and for $u^\epsilon$ defined by the method of characteristics \eqref{eq:S} with $u_0$, $H$ replaced by the smooth approximations $u_0^\epsilon$, $H^\epsilon$. Then, by the above argument, for any $\epsilon$-regularized trajectory $(q_t^\epsilon, p^\epsilon_t)$, we have 
\[
p_t^\epsilon = \nabla_q u^\epsilon(q_t^\epsilon,t).
\]
Choosing a sequence such that $H^\epsilon \overset{C^{r+1}}{\to} H$, $u_0^\epsilon \overset{C^r}{\to} u_0$ as $\epsilon \to 0$, the corresponding characteristic solution $u^\epsilon$ defined by \eqref{eq:S} (with $H^\epsilon$ and $u^\epsilon_0$ in place of $H$ and $u_0$) converges in $C^{r-1}$. Since $r-1=1$, this implies that $p_t = \lim_{\epsilon \to 0} p^\epsilon_t = \lim_{\epsilon\to 0} \nabla_q u^\epsilon(q^\epsilon_t,t) = \nabla_q u(q_t,t)$.

Thus, we conclude that $u \in C^{r-1}(\Omega\times [0,T^\ast))$ defined by \eqref{eq:S} is a classical solution of the Hamilton-Jacobi equations \eqref{eq:HJ}. We finally have to show that, in fact, $u\in C^r_\per(\Omega\times [0,T^\ast))$. $C^r$-differentiability in space follows readily from the fact that, by \eqref{eq:2}, we have $\nabla_q u(q_t,t) = p_t(q_0, \nabla_q u_0(q_0))$. By the invertibility of the spatial characteristic map, this can be equivalently written in the form 
\begin{align}
\label{eq:gradu}
\nabla_q u(q,t) = p_t(\Phi^\dagger_{-t,u_0}(q), \nabla_q u_0(\Phi^\dagger_{-t,u_0}(q))),
\end{align}
where on the right hand side, $(q_0,p_0) \mapsto p_t(q_0,p_0)$, $q_0 \mapsto \nabla_q u_0(q_0)$ and $(q,t)\mapsto \Phi^\dagger_{-t,u_0}(q,t)$ are all $C^{r-1}$ mappings. Thus, $\nabla_q u$ is a $C^{r-1}$-function. Furthermore, by \eqref{eq:HJ}, this implies that $\partial_t u = - H(q,\nabla_q u)$ is also a $C^{r-1}$ function. This allows us to conclude that $u \in C^r_\per(\Omega\times [0,T^\ast))$. The additional bound on $\Vert u(\slot,t) \Vert_{C^r}$ follows from the trivial estimate
\[
\Vert u(\slot,t) \Vert_{C^r}
\le
C\left(
\Vert u \Vert_{L^\infty}
+
\Vert \nabla_q u(\slot,t) \Vert_{C^{r-1}}
\right),
\]
combined with \eqref{eq:S} and \eqref{eq:gradu}, and the fact that $q\mapsto \Phi^\dagger_{-t,u_0}(q)$ is a $C^{r-1}$ mapping with continuous dependence on $u_0\in \cF$ and $t\in [0,T]$, and $(q_0,p_0) \mapsto p_t(q_0,p_0)$ is a $C^r$-mapping, so that 
\[
\Vert
\nabla_q u(\slot,t)
\Vert_{C^{r-1}}
\le
\sup_{t\in[0,T]} \sup_{u_0\in \cF} 
\Vert
p_t(\Phi^\dagger_{-t,u_0}(\slot), \nabla_q u_0(\Phi^\dagger_{-t,u_0}(\slot)))
\Vert_{C^{r-1}}
< \infty,
\]
for any initial data $u_0\in \cF$.
\end{proof}

\section{Reconstruction from scattered data}
\label{A:B}

The purpose of this appendix is to prove Proposition \ref{p:es}. This is achieved
through three lemmas, followed by the proof of the proposition itself, employing
these lemmas.

The first lemma is the following special case of the Vitali covering lemma, a well-known geometric result from measure theory (see e.g. \cite[Lemma 1.2]{stein2009real}):
\begin{lemma}[Vitali covering]
\label{lem:vitali}
If $h>0$ and $Q = \{q^1,\dots, q^N\}$ is a set for which the domain
\[
\Omega \subset \bigcup_{j=1}^N B_h(q^j),
\]
is contained in the union of balls of radius $h$ around the $q^j$, then there exists a subset $Q' = \{q^{i_1}, \dots, q^{i_m} \}$ such that $B_h(q^{i_k})\cap B_h(q^{i_\ell}) = \emptyset$ for all $k,\ell$, and 
\[
\Omega \subset \bigcup_{k=1}^m B_{3h}(q^{i_{k}}),
\]
\end{lemma}

\begin{remark}
\label{rem:vitali-alg}
Given $Q=\{q^1,\dots, q^N\}$, the proof of Lemma \ref{lem:vitali} (see e.g. \cite[Lemma 1.2]{stein2009real}) shows that the subset $Q'\subset Q$ of Lemma \ref{lem:vitali} can be found by the following greedy algorithm, which proceeds by iteratively adding elements to $Q'$ (the following is in fact the basis for Algorithm \ref{alg:pruning} in the main text):
\begin{enumerate}
\item Start with $j_1 = 1$ and $Q'_1 = \{q^1\}$. 
\item Iteration step: given $Q'_m = \{q^{j_1},\dots, q^{j_m}\}\subset Q$, check whether there exists $q^k \in Q$, such that
\[
B_h(q^k) \cap \bigcup_{\ell=1}^m B_h(q^{j_\ell}) = \emptyset.
\]
\begin{itemize}
\item If yes: define $j_{n+1} := k$ and $Q'_{m+1} := \{q^{j_1},\dots, q^{j_{m+1}}\}$.
\item If not: terminate the algorithm and set $Q' := Q'_m = \{q^{j_1},\dots, q^{j_{m}}\}$.
\end{itemize}
\end{enumerate}
\end{remark}

Based on Lemma \ref{lem:vitali}, we can now state the following basic result:

\begin{lemma}
\label{lem:stab} 
Given a set $Q = \{q^1,\dots, q^N\} \subset \Omega$ with fill distance $h_{Q,\Omega}$, the subset $Q'\subset Q$ determined by the pruning Algorithm \ref{alg:pruning} has fill distance $h_{Q',\Omega} \le 3h_{Q,\Omega}$ and separation distance $\rho_{Q'} \ge h_{Q,\Omega}$, and $Q'$ is quasi-uniform with distortion constant $\kappa = 3$, i.e.
\[
\rho_{Q'} \le h_{Q',\Omega} \le 3\rho_{Q'}.
\]
\end{lemma}

\begin{proof}
By definition of $h_{Q,\Omega}$ (cf. Definition \ref{def:scatter}), we have 
\[
\Omega \subset \bigcup_{j=1}^N B_{h_{Q,\Omega}}(q^j).
\]
Let $Q' = \{q^{j_1}, \dots, q^{j_m}\} \subset Q$ be the subset determined by Algorithm \ref{alg:pruning} (reproduced in Remark \ref{rem:vitali-alg}). $Q'$ satisfies the conclusion of the Vitali covering lemma \ref{lem:vitali} with $h = h_{Q,\Omega}$; thus,
\begin{align} 
\label{eq:cover}
\Omega \subset \bigcup_{k=1}^m B_{3h}(q^{j_k}), 
\end{align}
and $B_{h}(q^{j_k}) \cap B_{h}(q^{j_\ell}) = \emptyset$ for all $k\ne \ell$. The inclusion \eqref{eq:cover} implies that 
\[
h_{Q',\Omega}
=
\sup_{q\in \Omega} \min_{k=1,\dots, m} \left|q - q^{j_k}\right| \le 3h = 3h_{Q,\Omega}. 
\]
On the other hand, the fact that $B_{h}(q^{j_k}) \cap B_{h}(q^{j_\ell}) = \emptyset$ implies that $\frac12 \left| q^{j_\ell} - q^{j_k}\right| \ge h = h_{Q,\Omega}$ for all $k\ne \ell$, and hence $\rho_{Q'} \ge h_{Q,\Omega}$. Thus, we have
\[
h_{Q',\Omega} \le 3 h_{Q,\Omega} \le 3\rho_{Q'}.
\]
The bound $\rho_{Q'} \le h_{Q',\Omega}$ always holds (for convex sets); to see this, choose $q^{j_k}, q^{j_\ell}$ such that $\rho_{Q'} = \frac12 \left| q^{j_k}- q^{j_\ell}\right|$ realizes the minimal separation distance. Let $\bar{q}$ be the mid-point $\frac12 \left( q^{j_{k}} + q^{j_\ell}\right)$, so that $\left| \bar{q} - q^{j_k} \right| = \left| \bar{q} - q^{j_\ell}\right| = \frac12 \left| q^{j_k} - q^{j_\ell}\right| = \rho_{Q'}$. We note that $\left|\bar{q} - q^{j_r}\right| \ge \rho_{Q'}$ for any $q^{j_r}\in Q'$, since
\[
2\rho_{Q'}
\le
\left|q^{j_r}- q^{j_k}\right|
\le
\left| q^{j_r} - \bar{q} \right| + \left|\bar{q} - q^{j_k} \right|
=
\left|\bar{q} - q^{j_r}\right| + \rho_{Q'},
\]
i.e. $\left|\bar{q} - q^{j_r}\right| \ge \rho_{Q'}$ for $r=1,\dots, m$.
We conclude that 
\[
\rho_{Q'} \le \min_{r=1,\dots, m} \left|\bar{q} - q^{j_r}\right| \le \sup_{q\in Q} \min_{r=1,\dots, m} \left| q - q^{j_r}\right| = h_{Q',\Omega}.
\]
\end{proof}

\begin{lemma}
\label{prop:recerror}
Let $\Omega = [0,2\pi]^d \subset \R^d$, let $r\ge 2$ and let $\kappa \ge 1$. There exists $C = C(d,r,\kappa)>0$ and $\gamma = \gamma(d,r,\kappa) >0$, such that for all $f^\dagger\in C^{r}(\Omega)$ and all $Q\subset \Omega$, quasi-uniform with respect to $\kappa$, and with fill distance $h_{Q,\Omega}$, the approximation error of the moving least squares interpolation \eqref{eq:mls} with exact data $z^j = f^\dagger(q^j)$, and parameter $\delta := \gamma h_{z,Q}$, is bounded as follows:
\[
\Vert f^\dagger - f_{z,Q} \Vert_{L^\infty} \le C h_{Q,\Omega}^{r} \Vert f^\dagger \Vert_{C^{r}}.
\]
\end{lemma}
\begin{proof}
This is immediate from \cite[Corollary 4.8]{wendland2004scattered}.
\end{proof}

Based on Lemma \ref{prop:recerror}, we can also derive a reconstruction estimate, when the interpolation data $\{q^{j,\dagger},f^\dagger(q^{j,\dagger})\}_{j=1}^N$ is only known up to small errors in both the position $q^j \approx q^{j,\dagger}$ and the values $z^j \approx f^\dagger(q^{j,\dagger})$, and this is Proposition \ref{p:es} stated in the
main body of the text. We now prove this proposition.

\begin{proof}{(Proposition \ref{p:es})}
Let $\psi: \R^d \to [0,1]$ be a compactly supported $C^\infty$ function, such that $\psi(0) = 1$, $\Vert \psi \Vert_{L^\infty} \le 1$ and $\psi(q) = 0$ for $|q|\ge 1/2$. Define 
\[
\tilde{f}(q) := f^\dagger(q) + \sum_{j=1}^N \left(z^j - f^\dagger({q}^j)\right) \psi\left(\frac{q-{q}^j}{\rho_Q}\right).
\]
Note that, by assumption, we have 
\[
|{q}^j - {q}^k| \ge |q^{j,\dagger} - q^{k,\dagger}| - 2\delta \ge 2(\rho_{Q}-\rho) \ge \rho_Q,
\]
for all $j\ne k$. In particular this implies that the functions
\[
q \mapsto \psi\left(\frac{q-{q}^j}{\rho_Q}\right),
\]
have disjoint support for different values of $j$. Thus, for any $j=1,\dots, N$, we have
$\tilde{f}({q}^j) = z^j$, and ${f}_{z,Q}$ is the (exact) moving least squares approximation of $\tilde{f}$. From Proposition \ref{prop:recerror} it follows that
\[
\Vert \tilde{f} - {f}_{z,Q} \Vert_{L^\infty} \le C \Vert \tilde{f} \Vert_{C^{r}} h^{r}_{Q,\Omega}.
\]
We now note that 
\begin{align*}
|z^j - f^\dagger({q}^j)|
&\le
|z^j - f^\dagger({q}^{j,\dagger})| + |f^\dagger(q^{j,\dagger}) - f^\dagger({q}^j) |
\le
\epsilon + \Vert f^\dagger \Vert_{C^{r}} \rho.
\end{align*}
On the one hand  (due to the disjoint supports of the bump functions), we can now make the estimate
\[
\Vert \tilde{f} - f^\dagger \Vert_{L^\infty}
\le
\max_{j=1,\dots, N} |z^j - f^\dagger({q}^j)|
\le
\epsilon + \Vert f \Vert_{C^{r}} \rho.
\]
On the other hand (again due to the disjoint supports of the bump functions), we also have
\begin{align*}
\Vert \tilde{f}\Vert_{C^{r}}
&\le
\Vert f^\dagger \Vert_{C^{r}}
+ 
\max_{j=1,\dots, N} |z^j - f^\dagger({q}^j)| \rho_Q^{-r} \Vert \psi \Vert_{C^{r}}
\\
&\le 
\Vert f^\dagger \Vert_{C^{r}}
+ 
\frac{\epsilon + \Vert f^\dagger \Vert_{C^{r}} \rho}{\rho_Q^{r}} \Vert \psi \Vert_{C^{r}}
\end{align*}
These two estimates imply that
\begin{align*}
\Vert f^\dagger - {f}_{z,Q} \Vert_{L^\infty}
&\le
\Vert f^\dagger -\tilde{f}\Vert_{L^\infty} + \Vert \tilde{f} - {f}_{z,Q} \Vert_{L^\infty}
\\
&\le 
\left(\epsilon + \Vert f^\dagger \Vert_{C^{1}} \rho\right) + C \left(\Vert f^\dagger \Vert_{C^{r}} + \left(\epsilon + \Vert f^\dagger \Vert_{C^{1}} \rho\right) \frac{\Vert \psi \Vert_{C^{r}}}{\rho^{r}_Q} \right) h_{Q,\Omega}^{r}
\end{align*}
Taking into account that $h_{Q,\Omega} / \rho_Q \le \kappa$ is bounded, that $\Vert \psi \Vert_{C^{r}}$ is independent of $f^\dagger$, and introducing a new constant $C = C(d,r,\kappa, \Vert \psi \Vert_{C^{r}})>0$, we obtain
\begin{align*}
\Vert f^\dagger - {f}_{z,Q} \Vert_{L^\infty}
&\le
C\left( \Vert f^\dagger \Vert_{C^{r}} h^{r}_{Q,\Omega} +  \epsilon + \Vert f^\dagger \Vert_{C^1} \rho \right),
\end{align*}
as claimed.
\end{proof}

\section{Complexity estimates for HJ-Net approximation}
\label{A:C}

In the last section, we have shown that given a set of initial data $\cF \subset C^r_\per(\Omega)$, with $r \ge 2$ and $\Omega = [0,2\pi]^d$, the method of characteristics is applicable for a positive time $T^\ast > 0$. In the present section, we will combine this result with the proposed HJ-Net framework to derive quantitative error and complexity estimates for the approximation of the solution operator of \eqref{eq:HJ}. This will require estimates for the approximation of the Hamiltonian flow $\Psi_t \approx \Psi^\dagger_t$ by neural networks, as well as an estimate for the reconstruction error resulting from the pruned moving least squares operator $\cR$.

\subsection{Approximation of Hamiltonian flow}

\begin{proof}{(Proposition \ref{prop:ahf})}
It follows from \cite[Theorem 1]{yarotsky_error_2017} (and a simple scaling argument), that there exists a constant $C = C(r,d)>0$, such that for any $\epsilon>0$, there exists a neural network $\Psi_t \approx \Psi_t^\dagger$ satisfying the bound, 
\begin{align}
\label{eq:aprox0}
\sup_{(q_0,p_0,z_0) \in \Omega \times [-M,M]^d \times [-M,M]}
\left|
\Psi_t(q_0,p_0,z_0) - \Psi_t^\dagger(q_0,p_0,z_0)
\right|
\le 
\epsilon,
\end{align}
and such that 
\begin{align*}
\depth(\Psi_{t}) &\le C \log\left( \frac{M^r \Vert \Psi_t^\dagger \Vert_{C^{r}}}{\epsilon} \right), \\
\size(\Psi_{t}) &\le C \left(\frac{M^r \Vert \Psi_t^\dagger \Vert_{C^{r}}}{\epsilon}\right)^{(2d+1)/r} \log\left( \frac{M^r \Vert \Psi_t^\dagger \Vert_{C^{r}}}{\epsilon} \right),
\end{align*}
where $\Vert \Psi_t^\dagger \Vert_{C^{r}} = \Vert \Psi_t^\dagger \Vert_{C^{r}(\Omega\times [-M,M]^d\times [-M,M])}$ denotes the $C^r$ norm on the relevant domain. To prove the claim, we note that 
for any trajectory $(q_t,p_t)$ satisfying the Hamiltonian ODE system
\begin{align}
\label{eq:H1}
\dot{q}_t = \nabla_{p} H(q_t,p_t), 
\quad
\dot{p}_t = - \nabla_{q} H(q_t,p_t),
\end{align}
with initial data $(q_0,p_0) \in \Omega\times [-M,M]^d$, we have by assumption \ref{ass:blowup}:
\[
\frac{d}{dt} (1+|p_t|^2) = p_t \cdot \dot{p}_t = -2p_t \cdot \nabla_q H(q_t,p_t) \le 2L_H (1+|p_t|^2).
\]
Integration of this inequality implies $|p_t|^2 \le (1+|p_0|^2) \exp(2L_H t) \le (1+dM^2)\exp(2L_Ht)$. Taking also into account that $M\ge 1$, this implies that $p_t \in [-\beta M, \beta M]^d$, where $\beta = (1+\sqrt{d})\exp(L_H t)$ depends only on $d$, $L_H$ and $t$.  Since $p_t$ remains uniformly bounded and since $q\mapsto H(q,p)$ is $2\pi$-periodic, it follows that for any $(q_0,p_0) \in \Omega\times [-M,M]^d$ the Hamiltonian trajectory $(q_t,p_t)$ starting at $(q_0,p_0)$ stays in $\Omega\times [-\beta M, \beta M]$. 

Recall that $\Psi_t^\dagger(q_0,p_0,z_0) = (q_t,p_t,z_t)$ is the flow map of the Hamiltonian ODE \eqref{eq:H1} combined with the action integral 
\begin{align} \label{eq:H2}
z_t = z_0 + \int_0^t \left[p_t\cdot \nabla_p H(q_t,p_t) - H(q_t,p_t) \right] \, d\tau.
\end{align}
Since the Hamiltonian trajectories starting at $(q_0,p_0) \in \Omega\times [-M,M]^d$ are confined to $\Omega\times [-\beta M, \beta M]^d$ and since the right-hand side of \eqref{eq:H1} and \eqref{eq:H2} involve only first-order derivatives of $H$, it follows from basic ODE theory that there exists $C = C(\Vert H \Vert_{C^{r+1}(\Omega\times [-\beta M, \beta M]^d},M,t,r)>0$, such that the $C^r$-norm of the flow can be bounded by
\[
\Vert \Psi_t \Vert_{C^r(\Omega\times [-M,M]^d \times [-M,M])}
\le
C(\Vert H \Vert_{C^{r+1}(\Omega\times [-\beta M, \beta M]^d},M,t,r).
\]

In particular, we finally conclude that there exist constants $\beta = \beta(L_H,d,t)$ and $C = C(\Vert H \Vert_{C^{r+1}(\Omega\times [-\beta M, \beta M]^d},M,t,r)>0$, such that for any $\epsilon > 0$, there exists a neural network $\Psi_t \approx \Psi_t^\dagger$ satisfying the bound \eqref{eq:aprox0}, and such that 
\begin{align*}
\depth(\Psi_{t}) \le C \log\left( \epsilon^{-1} \right), \quad 
\size(\Psi_{t}) \le C \epsilon^{-(2d+1)/r} \log\left( \epsilon^{-1} \right).
\end{align*}
\end{proof}

\subsection{Reconstruction error}

\begin{proof}{(Proposition \ref{prop:recerr})}
Let $f^\dagger \in C^r_\per(\Omega)$ be given with $r\ge 2$, and let $\{{q}^j, {z}^j\}_{j=1}^N$ be approximate interpolation data, with 
\[
|{q}^j - q^{j,\dagger}| \le h^r_{Q,\Omega}, \quad |{z}^j - f^\dagger(q^{j,\dagger}) | \le h^{r}_{Q,\Omega}.
\]
The assertion of this proposition is restricted to $Q^\dagger = \{q^{j,\dagger}\}_{j=1}^N$ satisfying $h_{Q^\dagger,\Omega} \le h_0$ for a constant $h_0$ (to be determined below). We may wlog assume that $h_0 \le 1/16$ in the following. Denote ${Q} := \{{q}^j\}_{j=1}^N$. 
We recall that the first step in the reconstruction Algorithm \ref{alg:reconstruction} consists in an application of the pruning Algorithm \ref{alg:pruning} to determine pruned interpolation points ${Q}' = \{{q}^{j_1},\dots, {q}^{j_m}\} \subset {Q}$, such that (cf. Lemma \ref{lem:stab})
\[
\rho_{{Q}'} \le h_{{Q}',\Omega} \le 3 \rho_{{Q}'},
\quad
h_{{Q}',\Omega} \le 3 h_{{Q},\Omega}, 
\quad
h_{{Q},\Omega} \le \rho_{{Q}'}.
\]
\textbf{Step 1:}
Write $Q'^{,\dagger} = \{q^{j_1,\dagger},\dots, q^{j_m,\dagger}\}$. Our first goal is to show that $Q'^{\dagger}$ is \textbf{quasi-uniform}: To this end, we note that by definition of the separation distance and the upper bound on the distance of $q^{j,\dagger}$ and $q^j$:
\[
\rho_{Q'^{,\dagger}} 
\ge 
\rho_{{Q}'} - 2 \max_{k} |q^{j_k,\dagger} - {q}^{j_k}|
\ge
\rho_{{Q}'} - 2 h_{Q^\dagger,\Omega}^r.
\]
By the definition of the fill distance, and the assumption that $h_{Q^\dagger,\Omega}\le h_0\le 1/2$ and $r\ge 2$, we also have
\[
h_{Q^\dagger,\Omega} 
\le h_{{Q},\Omega} + \sup_{j} |{q}^j - q^{j,\dagger}| 
\le h_{{Q},\Omega} + h_{Q^\dagger,\Omega}^r
\le h_{{Q},\Omega} + \frac12 h_{Q^\dagger,\Omega},
\]
implying the upper bound $h_{Q^\dagger,\Omega} \le 2 h_{{Q},\Omega} \le 2\rho_{{Q}'}$. Similarly we can show that $h_{Q'^{,\dagger},\Omega} \le 2 h_{{Q}',\Omega}$. Substitution in the lower bound on $\rho_{Q'^{,\dagger}}$  above and using that $h_0 \le 1/16$, yields
\[
\rho_{Q'^{,\dagger}} 
\ge 
\rho_{{Q}'} - 2 \rho_{\tilde{Q}'} (2 h_{{Q},\Omega})^{r-1}
\ge 
\rho_{{Q}'} (1-2 (2h_0)^{r-1}) \ge \frac34 \rho_{{Q}'}.
\]
Thus, we conclude that
\[
h_{Q'^{,\dagger},\Omega} \le 2 h_{{Q}',\Omega} \le 6\rho_{{Q}'} \le 12 \rho_{Q'^{,\dagger}}.
\]
Since we always have $\rho_{Q'^{,\dagger}} \le h_{Q'^{,\dagger},\Omega}$, we conclude that $Q'^{,\dagger}$ is quasi-uniform with $\kappa = 12$. 

\textbf{Step 2:} Next, we intend to apply Proposition \ref{p:es} with $Q'^{,\dagger}$ in place of $Q^\dagger$ and with $\rho = \epsilon = h_{Q^\dagger,\Omega}^r$: To this end, it remains to show that $\rho \le \frac12 \rho_{Q'^{,\dagger}}$ is bounded by half the separation distance. To see this, we note that by the above bounds (recall also that $h_0 \le 1/16$),
\[
\rho 
\equiv
h_{Q^\dagger,\Omega}^r 
\le
h_0^{r-1} \, h_{Q^\dagger,\Omega}
\le
\frac{1}{16} h_{Q^\dagger,\Omega} 
\le
\frac{1}{16} \rho_{{Q}'}
\le
\frac{1}{8} \rho_{Q'^{,\dagger}}
<
\frac12 \rho_{Q'^{,\dagger}},
\]
showing that $\rho < \frac12 \rho_{Q'^{,\dagger}}$. We can thus apply Proposition \ref{p:es} to conclude that there exist constants $C,h_0>0$ (with $h_0 \le 1/16$), such that if $h_{Q^\dagger,\Omega} \le h_0/9$, we have 
\[
h_{Q'^{,\dagger},\Omega} \le 2 h_{{Q}',\Omega} \le 6h_{{Q},\Omega} \le 9 h_{Q^\dagger,\Omega} \le h_0
\]
and hence
\begin{align*}
\Vert f^\dagger - f_{{z},{Q}} \Vert_{L^\infty}
&\le
C\left(
\Vert f^\dagger \Vert_{C^r(\Omega)} + \epsilon + \Vert f^\dagger \Vert_{C^1(\Omega)} \rho
\right)
\\
&\le
2C \left(1 + \Vert f^\dagger \Vert_{C^r(\Omega)}\right) h_{Q^\dagger,\Omega}^r.
\end{align*}
Replacing $h_0 \to h_0/9$ and $C\to 2C$ now yields the claimed result of Proposition \ref{prop:recerr}.
\end{proof}

\begin{proof}{(Proposition \ref{prop:re})}
Let $q_0,\tilde{q}_0\in \Omega$ and $u_0\in \cF$ be given, and denote by $(q_\tau,p_\tau)$ the solution of the Hamiltonian ODE, with initial value $(q_0,p_0) = (q_0,\nabla_q u_0(q_0))$. Define $(\tilde{q}_\tau,\tilde{p}_\tau)$ similarly. 

By compactness of $\cF \subset C^r_\per(\Omega) \subset C^2_\per(\Omega)$, there exists a constant $M>0$, such that $|p_0|, |\tilde{p}_0| \le \Vert u_0 \Vert_{C^2(\Omega)} \le M$. By continuity of the flow map, there exists $\bar{M} > 0$, such that 
\[
|p_\tau|, |\tilde{p}_\tau| \le \bar{M}, \quad \forall \, \tau\in [0,t].
\]
Then, we have
\begin{align*}
\frac{d}{d\tau} |(q_\tau,p_\tau) - (\tilde{q}_\tau,\tilde{p}_\tau)|
&\le 
\left|
\nabla H(q_\tau,p_\tau) - \nabla H(\tilde{q}_\tau,\tilde{p}_\tau)
\right|
\\
&\le
\left( \sup_{q\in \Omega, |p|\le \bar{M}} \Vert D^2 H(q,p) \Vert \right)|(q_\tau,p_\tau) - (\tilde{q}_\tau,\tilde{p}_\tau)|,
\end{align*}
where $D^2H$ denotes the Hessian of $H$ and $\Vert D^2 H\Vert$ is the matrix norm induced by the Euclidean vector norm.  Further denoting 
\[
\Vert D^2 H \Vert_{\infty} := \sup_{q\in \Omega, |p|\le \bar{M}} \Vert D^2 H(q,p) \Vert,
\]
then by Gronwall's inequality, it follows that 
\[
|(q_t,p_t) - (\tilde{q}_t,\tilde{p}_t)|
\le
e^{\Vert D^2 H \Vert_{\infty} t} |(q_0,p_0) - (\tilde{q}_0,\tilde{p}_0)|.
\]
Furthermore, since $p_0 = \nabla_q u_0(q_0)$, and $\tilde{p}_0 = \nabla_q u_0(\tilde{q}_0)$, we have $|p_0 - \tilde{p}_0| \le \Vert u_0 \Vert_{C^2} |q_0 - \tilde{q}_0| \le M |q_0 - \tilde{q}_0|$, which implies that
\begin{align*}
|(q_t,p_t) - (\tilde{q}_t,\tilde{p}_t)|
&\le
e^{\Vert D^2 H \Vert_{\infty} t} |(q_0,p_0) - (\tilde{q}_0,\tilde{p}_0)|
\\
&\le (1+M)e^{\Vert D^2 H \Vert_{\infty} t} |q_0 - \tilde{q}_0|.
\end{align*}
Therefore, $\Phi^\dagger_{t,u_0}(q_0) = q_t$, $\Phi^\dagger_{t,u_0}(\tilde{q}_0) = \tilde{q}_t$ satisfy the estimate
\[
\left|\Phi^\dagger_{t,u_0}({q}_0) - \Phi^\dagger_{t,u_0}(\tilde{q}_0)\right|
\le 
C |q_0 - \tilde{q}_0|,
\]
with constant
\[
C = (1+M) \exp\left(t \sup_{q\in \Omega, |p|\le \bar{M}} \Vert D^2 H(q,p) \Vert \right),
\]
which depends only on $t$, $H$ and on $\cF$ (via $M$ and $\bar{M}$), but is \emph{independent} of the particular choice of $u_0$.
 Thus, if 
\[
\Omega \subset \bigcup_{j=1}^N B_h(q^j), 
\]
then 
\[
\Omega \subset \bigcup_{j=1}^N B_{Ch}(\Phi^\dagger_{t,u_0}(q^j)),
\]
for any $u_0 \in \cF$. This implies the claim.
\end{proof}

\end{document}